\documentclass[review]{elsarticle}
\usepackage[letterpaper,top=2.5cm,bottom=2.5cm,left=2.6cm,right=2.6cm,marginparwidth=1.75cm]{geometry}

\usepackage[dvipsnames]{xcolor}
\usepackage{tikz}
\usetikzlibrary{backgrounds}
\usetikzlibrary{arrows,shapes}
\usepackage{amsmath}
\usepackage{amsthm}
\usepackage{amssymb}
\usepackage{algpseudocode}
\usepackage{mathtools, nccmath}
\usepackage{wrapfig}
\usepackage{comment}
\usepackage{diagbox}
\usepackage{arydshln}
\usepackage{algpseudocode}
\usepackage{makecell}
\usepackage{algpseudocode}
\usepackage[colorlinks=true, linkcolor=red, urlcolor=red]{hyperref}
\usepackage{blindtext}

\usepackage{graphicx}
\usepackage{xspace}

\usepackage{array}
\usepackage{ragged2e}
\newcolumntype{P}[1]{>{\RaggedRight\hspace{0pt}}p{#1}}
\newcolumntype{X}[1]{>{\RaggedRight\hspace*{0pt}}p{#1}}
\newcommand{\anglebetween}[2]{\angle\bigl(#1,#2\bigr)}
\usepackage{tcolorbox}

\usepackage{tikz}
\usetikzlibrary{arrows,shapes,positioning,shadows,trees,mindmap}
\usepackage[edges]{forest}
\usetikzlibrary{arrows.meta}
\colorlet{linecol}{black!75}
\usepackage{xkcdcolors} % xkcd colors

\usepackage{tikz}
\usetikzlibrary{backgrounds}
\usetikzlibrary{arrows,shapes}
\usetikzlibrary{tikzmark}
\usetikzlibrary{calc}
\colorlet{mhpurple}{Plum!80}

\usepackage{lineno,amsmath,amsfonts,amssymb,amsthm,mathrsfs}
\usepackage{hyperref}
\hypersetup{colorlinks=true,linkcolor=blue}
\usepackage[linesnumbered,ruled,vlined,onelanguage]{algorithm2e}
\usepackage{booktabs}
\usepackage[skip=0pt]{caption,subcaption}
\usepackage{float}
\usepackage{longtable}
\usepackage{setspace}
\usepackage{tabularx}
\usepackage{booktabs}
\usepackage{tablefootnote}
\DeclareMathOperator*{\argmin}{arg\,min}

\usepackage{todonotes}
\journal{Transportation Research Part C: Emerging Technologies}
\biboptions{authoryear}

\newtheorem{theorem}{Theorem}

\newtheorem{definition}{Definition}
\begin{document}

\begin{frontmatter}

\title{Reconstructing Physics-Informed Machine Learning for Traffic Flow Modeling: a Multi-Gradient Descent and Pareto Learning Approach}

\author[1]{Yuan-Zheng Lei}
\author[1]{Yaobang Gong}
\author[1]{Dianwei Chen}
\author[2]{Yao Cheng}
\author[1]{Xianfeng Terry Yang*} 
\ead{xtyang@umd.edu}

%\cortext[cor1]{Corresponding author. Xianfeng Terry Yang}
\address[1]{University of Maryland, College Park, MD 20742, United States}
\address[2]{Florida Atlantic University, Boca Raton, FL 33431, United States}
\begin{abstract}
\par Physics-informed machine learning (PIML) has been widely adopted for traffic flow modeling in recent studies, due to its potential in combining the benefits of both physics-based and data-driven approaches. In conventional PIML, physics information from classical traffic flow models is typically incorporated by constructing a hybrid loss function that combines data-driven loss and physics loss through linear scalarization. The goal is to find a trade-off between these two objectives to improve the accuracy of model predictions. However, from a mathematical perspective, linear scalarization is limited to identifying only the convex region of the Pareto front, as it treats data-driven and physics losses as separate objectives. Given that most PIML loss functions are non-convex, linear scalarization restricts the achievable trade-off solutions. Moreover, tuning the weighting coefficients for the two loss components can be both time-consuming and computationally challenging. To address these limitations, this paper introduces a paradigm shift in PIML by reformulating the training process as a multi-objective optimization problem, treating data-driven loss and physics loss independently. We apply several multi-gradient descent algorithms (MGDAs), including traditional multi-gradient descent (TMGD) and dual cone gradient descent (DCGD), to explore the Pareto front in this multi-objective setting. These methods are evaluated on both macroscopic and microscopic traffic flow models. In the macroscopic case, MGDAs achieved comparable performance to traditional linear scalarization methods. Notably, in the microscopic case, MGDAs significantly outperformed their scalarization-based counterparts, demonstrating the advantages of a multi-objective optimization approach in complex PIML scenarios.

\end{abstract}
\end{frontmatter}

\section{Introduction} \label{section:1}
%\linenumbers
\par Physics-Informed Machine Learning (PIML) has emerged as a powerful paradigm in traffic flow modeling, enabling the integration of domain knowledge with data-driven learning. In macroscopic traffic modeling, where aggregated quantities such as density, flow, and speed evolve over space and time, PIML enhances the predictive accuracy and generalizability of models by embedding conservation laws and fundamental diagrams into neural network architectures. At the microscopic level, which focuses on individual vehicle behaviors such as car-following and lane-changing dynamics, PIML facilitates the reconstruction of realistic vehicle trajectories by fusing physical models like the Intelligent Driver Model (\cite{treiber2000congested}) with empirical observations. This dual applicability makes PIML an essential tool for bridging theory and data in modern traffic systems, particularly as emerging transportation networks demand interpretable, robust, and generalizable learning frameworks.
\par PIML effectively combines the advantages of data-driven and physics-based methods, making it useful in transportation research. Data-driven approaches have gained popularity due to their low computational cost and flexibility in managing complex scenarios without requiring strong theoretical assumptions. Many of machine learning techniques, especially deep learning architectures \cite{ma2015long, duan2016efficient, polson2017deep, zhou2017recurrent, wu2018hybrid, cui2018deep, cui2019traffic, zhang2025empirical}, have various application in transportation research. However, these models heavily depend on the quality and representativeness of the training data. As shown in \cite{yuan2021macroscopic}, when manually adding errors to the training dataset, the performance of all pure data-driven models has significantly deteriorated. While a data-driven approach can yield positive results with quality data, it often lacks interpretability, which may diminish its practicality in decision-making processes \cite{lei2024unraveling}.
\par Above limitations underscore the need for modeling approaches that strike a balance between physical interpretability and the ability to adapt to imperfect or evolving data environments. To address this challenge, PIML represents a hybrid approach that combines established physical laws and principles with machine learning models, enhancing predictive accuracy and robustness in the presence of data noise. By integrating established physics constraints into the learning process, PIML offers a distinct advantage in scenarios where traditional data-driven methods may struggle with noisy data. This integration enables PIML to capture complex system behaviors, leading to models that are more reliable and interpretable. In general, physics models used to build PIML should perform well even without any data, since most of them employed in computational physics are partial differential equations (PDEs) (\cite{krishnapriyan2021characterizing}). 
\par In nearly all PIML studies, especially in PIML application in transportation, as shown in \textcolor{blue}{Table} \ref{table:1}, researchers commonly use a method called linear scalarization to define the loss function\footnote{In many Physics-regularized Gaussian process models (\cite{yuan2020modeling},\cite{yuan2021macroscopic},\cite{yuan2021traffic}), the training target is to minimize a evidence lower bound, which can be written as:
\begin{equation}
    \log \mathcal{N}([\cdot]) + \gamma\mathbb{E}(\cdot) \label{eq:1}
\end{equation}
where  $\gamma$ is a regularization coefficient. The overall ELBO can also be regarded as linear scalarization.}, as demonstrated in \textcolor{blue}{Eq} \ref{eq:1}. where $\alpha$ and $\beta$ are hyperparameters that control the effect of the data-driven module and physics module, $\mathcal{L}_{data}$ and $\mathcal{L}_{physics}$ represent data-driven loss term and physics loss term, respectively. This approach facilitates the easy integration of physics information into PIML models. When $\alpha\rightarrow 0$, the PIML model degenerates to a pure data-driven model; when $\beta\rightarrow 0$,  the PIML model degenerates to a pure physics-driven model.
\begin{table}[htp!]
\centering
\caption{Comparison of approaches for establishing loss functions} \label{table:1}
\resizebox{\textwidth}{!}{
\begin{tabular}{ccc}
\toprule
Model & Approaches for establishing loss functions & Loss function\textsuperscript{*} \\
\midrule
\cite{shi2021physics} & Linear scalarization  & $\mathcal{L}_{t} = \alpha \mathcal{L}_{d} + \beta\mathcal{L}_{p} $, $\mathcal{L}_{t} = \alpha \mathcal{L}_{d} + \beta\mathcal{L}_{p} + \gamma \mathcal{L}_{r} $   \\
\cite{lu2023physics} & Linear scalarization  & $\mathcal{L}_{t} = \alpha \mathcal{L}_{d} + \beta\mathcal{L}_{p} $   \\
\cite{mo2021physics} & Linear scalarization  & $\mathcal{L}_{t} = \alpha \mathcal{L}_{d} + \beta\mathcal{L}_{p} $   \\
\cite{pereira2022short} & Linear scalarization  & $\mathcal{L}_{t} =  \alpha\mathcal{L}_{d} + \beta\mathcal{L}_{p} + \gamma\mathcal{L}_{r} $   \\
\cite{xue2024network} & Linear scalarization  & $\mathcal{L}_{t} = \alpha \mathcal{L}_{d} + \beta\mathcal{L}_{p} $   \\
\bottomrule
\end{tabular}
}
\footnotesize\textsuperscript{*} Hyperparameters $\alpha$, $\beta$, and $\gamma$ are used to control the weight of specific losses. $\mathcal{L}_{d}$ denotes the data-driven loss, $\mathcal{L}_{p}$ represents the physics loss, and $\mathcal{L}_{r}$ indicates the regularization term.
\end{table}
\begin{equation}
    \label{eq:2}
    \mathcal{L}(\boldsymbol{\theta}) = \alpha\mathcal{L}_{\mathrm{data}}(\mathbf{X};\boldsymbol{\theta}^{data}) + \beta \mathcal{L}_{\mathrm{physics}}(\mathbf{X};\boldsymbol{\theta}^{physics})
\end{equation}
\par Though linear scalarization appears simple and straightforward initially, it can lead to several significant problems. First, by adopting the linear scalarization to construct the loss function, we are essentially making an assumption that the trade-off between the data-driven model and the physics model is linear. Second, it is well known that the process of searching for optimal hyperparameters can be very challenging and time-consuming. Even for the simplest cases, as shown in \textcolor{blue}{Eq} \ref{eq:1}, only two hyperparameters are shown in the model. The search for the optimal coefficient combination is equivalent to searching for the optimal $\frac{\alpha}{\beta}$, if we consider a hundredfold order of magnitude difference between $\alpha$ and $\beta$, and the searching precision is 1, then we have to run $10000$ times experiments only for selecting the best hyperparameters.
\par In this paper, we aim to circumvent the hyperparameter search process by transforming the training process of the PIML model into a multi-objective optimization problem. Mathematically, we transfer the overall training process of PIML models from \textbf{\textcolor{blue}{Eq}} \ref{eq:3} to \ref{eq:4}. \label{sec:4} \label{line:4} To clarify, in the rest of the paper, multi-objective optimization refers to the core training paradigm we adopt, which treats the data loss and physics loss as distinct objectives and optimizes them simultaneously.
\begin{equation} \label{eq:3}
    \begin{aligned}
        & \underset{\boldsymbol{\theta}}{\min}\ \mathcal{L}(\boldsymbol{\theta}) = \alpha\mathcal{L}_{\mathrm{data}}(\mathbf{X};\boldsymbol{\theta}) + \beta \mathcal{L}_{\mathrm{physics}}(\mathbf{X};\boldsymbol{\theta})   && \\
    \end{aligned}  
\end{equation} 
\begin{equation} \label{eq:4}
    \begin{aligned}
        & \underset{\boldsymbol{\theta}}{\min}   \bigg(\mathcal{L}_{\mathrm{data}}(\mathbf{X};\boldsymbol{\theta}), \mathcal{L}_{\mathrm{physics}}(\mathbf{X};\boldsymbol{\theta}) \bigg)   && \\
    \end{aligned} 
\end{equation}
where \( \boldsymbol{\theta}\) represents the total parameter set. \textbf{\textcolor{blue}{Eq}} \ref{eq:4} represents a specific type of multi-objective optimization problem. In this context, the objectives are typically continuous and differentiable. The most common representative of the pure data-driven segment of Physics-Informed Machine Learning (PIML) relies on neural networks. As a result, this multi-objective optimization problem usually involves thousands of parameters that need to be optimized. Some evolutionary algorithms, such as NSGA-II (\cite{deb2002fast}) and NSGA-III (\cite{deb2013evolutionary, jain2013evolutionary}), which are popular for general multi-objective optimization problems, may not be suitable for implementation, as they struggle to find a solution with a large number of decision variables. Therefore, Various multi-gradient descent algorithms (MDGAs) will be used to address this multi-objective optimization problem, as they are considered capable of handling large-scale optimization challenges (\cite{sener2018multi}). Our contribution can be summarized as:
\begin{itemize}
    \item We demonstrate that the potential limitations of the linear scalarization approach can be applied to the PIML in transportation research. That is, the linear scalarization essentially imposes a linear trade-off between different objectives, and may also fail to preserve Pareto stationary points.
    \item We demonstrate that any minimizer derived from scalarization can theoretically be achieved through MDGAs, but the opposite is not true. This indicates that MDGA methods would not theoretically perform worse than linear scalarization.
    \item We test various MGDAs, including traditional MGDA [\cite{sener2018multi} and \cite{desideri2012multiple}] and dual cone gradient descent \cite{hwang2024dual}, on PIML models for both macroscopic traffic flow models and microscopic traffic flow models. MGDAs can achieve comparable or even slightly better prediction accuracy than PIML models that use a linear scalarization loss. This suggests that incorporating MGDAs into PIML models in transportation is beneficial because it eliminates the need for hyperparameter tuning, while also ensuring improved theoretical outcomes.
\end{itemize}
\par The remainder of the paper is organized as follows: Section \ref{section:2} offers a brief review of PLML models, with a focus on their applications in transportation. Section \ref{section:3} introduces the fundamental concept of multi-gradient descent algorithms and discusses some theoretical limitations associated with linear scalarization. Section \ref{section:4} evaluates four multi-gradient descent algorithms (MGDAs) within two distinct PIML frameworks, specifically based on macroscopic and microscopic traffic flow models. Section \ref{section:5} will give the conclusion and discussion.
\section{Literature review} \label{section:2}
\par Recent research on physics-informed machine learning (PIML) in transportation research shows that most successful approaches are built on neural networks that explicitly encode physical knowledge. At their core, these models fuse two information sources: noisy real-world data and well-tested conservation laws or empirical traffic principles. A good example is \cite{ji2022stden}, who present the Spatio-Temporal Differential Equation Network (STDEN). Rather than predicting traffic states outright, STDEN assumes an underlying potential-energy field that steers traffic evolution across the road network. By writing a differential-equation network for the time–space dynamics of this latent field and training it jointly with measured flows and speeds, STDEN achieves strong accuracy on several large-scale data sets and remains stable when observations are sparse or missing.

\par In a similar spirit, \cite{shi2021physics} adds a Fundamental Diagram Learner (FDL)—implemented as a multilayer perceptron—with a conventional Physics-Informed Neural Network (PINN). The FDL learns a smooth, data-driven fundamental diagram that links traffic density and speed, while the surrounding PINN enforces the macroscopic conservation law. This division of labor lets the network capture key flow patterns, reduces over-fitting, and leads the combined model to outperform a broad range of baseline machine-learning methods as well as traditional PINNs on both real and synthetic benchmarks. Their ablation study further shows that injecting the learned diagram improves generalization even when sensor coverage drops sharply.

\par One clear strength of the PIML framework is its flexibility: the learning module can adopt almost any neural-network architecture—or even non-NN techniques—to suit a particular task. For instance, \cite{xue2024network} proposes a Graph Neural Network (GNN) with the Network Macroscopic Fundamental Diagram (NMFD) to impute missing traffic states. The GNN captures complex spatial dependencies among links, while the NMFD term keeps predictions physically plausible. Likewise, \cite{pereira2022short} shows that a long short-term memory (LSTM) backbone plugs in just as naturally, enabling short-range traffic forecasting when strong temporal correlations exist. In both cases, the physics term reduces data needs and stabilizes training.

\par Besides using Neural Networks, Gaussian Process has also been a widely-adopted tool for traffic flow modeling. \cite{xu2023agnp} proposed a method to integrate attentive graph neural process (AGNP) and Gaussian process for network-level short-term traffic speed prediction and imputation. \cite{zhu2023bayesian} develops a Bayesian clustering ensemble Gaussian process (BCEGP) model for network-wide traffic flow clustering and prediction. Integrated with physics knowledge and inspired by \cite{wang2020physics}, \cite{yuan2021macroscopic} introduce the Physics-Regulated Gaussian Process (PRGP). Here, the governing equations appear through a shadow Gaussian process rather than an extra PINN loss term. Training maximizes a compound evidence lower bound (ELBO) that balances the usual GP likelihood with a penalty that measures how well samples obey the flow equation. A follow-up study by \cite{yuan2021traffic} extends this idea to a richer family of macroscopic models, improving fidelity and robustness without sacrificing the analytic uncertainty estimates that make Gaussian processes attractive. 

\par Beyond macroscopic modeling, PIML has also proved valuable at the microscopic level, where individual vehicles interact. \cite{yang2018novel} blends a learned car-following rule with a simple kinematic model using a linear scalarization, overcoming the tendency of purely data-driven rules to violate safety constraints. Building on this direction, \cite{mo2021physics} put forward PIDL-CF, a physics-informed car-following model that can use either a feed-forward neural network or an LSTM backbone. Their extensive experiments—especially those with sparse probe-vehicle data—demonstrate clear gains over standard benchmarks. In parallel, \cite{yuan2020modeling} adapts the PRGP framework to jointly model car-following and lane-changing behaviour, while \cite{liu2023quantile} underlines the rising appeal of PIML for microscopic tasks by introducing a quantile-based extension that describes variability, not just mean behaviour.

\par Finally, researchers have begun to push PIML ideas into other corners of transportation science. \cite{uugurel2024learning} leverages a PRGP variant to create realistic synthetic human-mobility traces that preserve trip-length statistics, and \cite{tang2024physics} designs a physics-informed calibration strategy that consistently beats conventional optimization-based methods when tuning simulation models.

\section{Pareto learning for Physics-informed machine learning} \label{section:3}
\par In this section, we will introduce the Pareto learning approach for PIML and the reason why we introduce Pareto learning into PIML for transportation. The core idea of the Prareto learning approach is that we regard the overall training process as a multi-objective optimization problem, with objectives including data-driven loss, physics loss, and regularization terms. And we will adopt multi-gradient descent algorithms to optimize this multi-objective optimization problem instead of evolutionary algorithms like NSGA-II  and NSGA-III.  Similar logic has been well-studied in the field of multi-task learning (\cite{sener2018multi,ma2020efficient,liu2021profiling} and \cite{lin2022pareto}). \cite{hwang2024dual} proposed a specific MGDA for PIML, which will be one of the foundations of our studies.
\subsection{Multi-gradient descent algorithms} \label{section:3.1}
\par From the previous sections, we know that the Pareto learning approach for PIML involves solving a multi-objective optimization problem, and MGDAs are key to its solution. The core logic of MGDAs is that we aim to find a direction that can simultaneously minimize all objectives (data-driven loss and physics loss). Even in the worst-case scenario, we hope this direction will not increase any of these objectives. In our paper, we primarily focus on two types of MGDAs: the traditional multi-gradient descent algorithm (TMGD) and the dual cone gradient descent algorithm (MGDA).  Before formally introducing each algorithm, we will first provide some basic definitions and theorems to help the reader better understand the algorithm. The definitions of Pareto dominance, Pareto optimality, and Pareto stationarity are crucial for the TMGD and DCGD algorithms. For both algorithms, achieving Pareto stationarity serves as a stopping criterion, indicating that they cannot proceed further.  Pareto stationarity is a weak condition of Pareto optimality. When all objectives are convex and at least one of them is strongly convex, Pareto stationarity is equivalent to Pareto optimality. Since the loss function of a neural network is generally non-convex, all multi-gradient descent algorithms use Pareto stationarity as the final evaluation criterion.
\begin{definition}[Pareto dominance] \label{definition:1}
A solution $\boldsymbol{\theta}$ dominates a solution $\hat{\boldsymbol{\theta}}$ if and only if $\mathcal{F}^{t}(\boldsymbol{\theta}) \leq \mathcal{F}^{t}(\hat{\boldsymbol{\theta}}), \forall t \in   \mathcal{T}$, and $\mathcal{F}^{t'}(\boldsymbol{\theta}) < \mathcal{F}^{t'}(\hat{\boldsymbol{\theta}}), \exists t' \in  \mathcal{T}$.
\end{definition}
where $t$ and $t'$ are two indices of objective functions, $\mathcal{T}$ is the index set of objective functions. $\mathcal{F}^{t}(\cdot)$ refers to the value of the $t$th objective function.
\begin{definition}[Pareto optimal] \label{definition:2} A solution $\boldsymbol{\theta}^{*}$ is called Pareto optimal if and only if there exists no solution $\boldsymbol{\theta}$ that dominates $\boldsymbol{\theta}^{*}$.
\end{definition}
\begin{definition}[Pareto stationary point] \label{definition:3}
Any solution that satisfies either of the following conditions is called a Pareto stationary point.
\begin{enumerate}
    \item There exist $\alpha^{1},...,\alpha^{T}$ such that $\sum_{t = 1}^{T}\alpha^{t} = 1$ and $\sum_{t=1}^{T}\alpha^{t}\nabla_{\boldsymbol{\theta}^{sh}}\mathcal{L}^{t}(\boldsymbol{\theta}^{sh},\boldsymbol{\theta}^{t}) = 0$
    \item For all objectives $t$, $\nabla_{\boldsymbol{\theta}^{t}}\mathcal{L}^{t}(\boldsymbol{\theta}^{sh},\boldsymbol{\theta}^{t}) = 0$
\end{enumerate}
where $\boldsymbol{\theta}^{sh}$ is the shared parameters among all objectives, $\boldsymbol{\theta}^{t}$ is the exclusive parameters of $t$th objective.\footnote{In the context of PIML, $\boldsymbol{\theta}^{sh} \subset \boldsymbol{\theta}^{data} \cap \boldsymbol{\theta}^{physics}$}
\end{definition}
\par In addition, readers can find out that some of the studies about multi-gradient descent (\cite{hwang2024dual}) do not divide the total parameter set $\boldsymbol{\theta}$ into $\boldsymbol{\theta}^{sh}$, $\boldsymbol{\theta}^{data}$ and $\boldsymbol{\theta}^{physics}$ like in \cite{sener2018multi}. In fact, let us consider the data-driven loss $\mathcal{L}_{data}$, if we calculate its derivatives of the entire parameter set $\boldsymbol{\theta}$, we will have:
\begin{equation} \label{eq:5}
    \frac{\partial \mathcal{L}_{data}}{\partial \boldsymbol{\theta}} = 
\begin{bmatrix}
\frac{\partial \mathcal{L}_{data}}{\partial \boldsymbol{\theta}^{data}} \\[0.3em]
\frac{\partial \mathcal{L}_{data}}{\partial \boldsymbol{\theta}^{physics}}
\end{bmatrix} = 
\begin{bmatrix}
\frac{\partial \mathcal{L}_{data}}{\partial \boldsymbol{\theta}^{data}} \\[0.3em]
\mathbf{0}
\end{bmatrix}
\end{equation}
The derivative of the data-driven loss with respect to the physics-exclusive parameters will be zero. Therefore, such a terminology difference will not affect the computation of the common descent direction, as it only depends on the different terminology systems used in various studies.
\subsubsection{Traditional multi-gradient descent algorithm} \label{section:3.1.1}
\par The core idea of the TMGD algorithm is to find the minimum norm point in a convex hull, so we first introduce the definition of convex hull and minimum norm point.
\begin{definition}[Convex hull] \label{definition:4} The convex hull of a set of points $S$ in $n$ dimensions is defined as the intersection of all convex sets containing $S$. For $N$ points $\mathbf{p}_{1}$, $\mathbf{p}_{2}$,..,$\mathbf{p}_{N}$, the convex hull $\textbf{conv}(\mathbf{p}_{1}, \mathbf{p}_{2},..,\mathbf{p}_{N})$ is then given by:
\begin{equation} \label{eq:6}
\textbf{conv}(\mathbf{p}_{1}, \mathbf{p}_{2},..,\mathbf{p}_{N}) = \left\{ \sum_{j=1}^{N} \lambda_j p_j \;:\; \lambda_j \geq 0 \text{ for all } j \text{ and } \sum_{j=1}^{N} \lambda_j = 1 \right\}
\end{equation}
\end{definition}
\begin{definition}[Minimum Norm Point] \label{definition:5}
Let \(C \subset \mathbb{R}^n\) be a nonempty closed convex set. The \emph{minimum norm point} of \(C\) is the unique solution
\begin{equation}
   x^* \;=\; \arg\min_{x \in C} \|x\|_2 \label{eq:7}
\end{equation}
Equivalently, \(x^*\) satisfies the first‐order optimality condition
\begin{equation}
   \langle x^*,\, x - x^* \rangle \;\ge\; 0 \label{eq:8}
   \quad\text{for all } x \in C
\end{equation}
\end{definition}
\par From the \textcolor{blue}{\textbf{Definition}} \ref{definition:4} and \ref{definition:5}, it is clear that the convex hull itself is also a nonempty closed convex set, based on \textcolor{blue}{\textbf{Eq}} \ref{eq:7}, we can natural obtain the following criterion:
\begin{theorem}[Wolfe's criterion (\cite{wolfe1976finding})] Let $P$ $\subset$ $\mathbb{R}^{d}$ be the convex hull of finitely many points, $P$ = $\mathbf{conv}(\mathbf{p}_{1}$, $\mathbf{p}_{2}$,..,$\mathbf{p}_{n}$). A point  $\mathbf{x} \in P$ is the minimum norm point in $P$ if and only if:
\begin{equation}
    \mathbf{x}^{\top}\mathbf{p}_{j} \geq \|\mathbf{x}\|_{2}^{2}, \forall j \in [n] \label{eq:9}
\end{equation}
 \label{theorem:1}
\end{theorem}
\begin{proof} See \textcolor{blue}{\textbf{Theorem 2.1 \& 2.2}} in \cite{wolfe1976finding}.
\end{proof}
\par Based on \textcolor{blue}{\textbf{Theorem}} \ref{theorem:1}, the minimum norm point in a convex hull essentially has very attractive properties. If we consider a special convex hull, which consists of finite points. Those points represent the gradients of different objectives. Then, by using the  \textcolor{blue}{\textbf{Theorem}} \ref{theorem:1}, a common descent direction can be computed. Mathematically, we formulate this in the following theorem:
\begin{theorem}
    Let $P$ $\subset$ $\mathbb{R}^{d}$ be the convex hull of the gradients of the different objective functions, i.e., $P$ = \textbf{conv}($\mathbf{p}_{1}$, $\mathbf{p}_{2}$,..,$\mathbf{p}_{n}$). If $\mathbf{x}$ is the minimum norm point, and $\mathbf{x} \in P$, then $-\mathbf{x}$ is a direction that can improve all objectives. \label{theorem:2}
\end{theorem}
\begin{proof}
If \(\mathbf{x}\) is the minimum norm point in the convex hull \(P\), by Wolfe's criterion, it satisfies:
\begin{equation}
\mathbf{x}^{\top} \mathbf{p}_j \geq \|\mathbf{x}\|_2^2 \quad \text{for all } j = 1, 2, \ldots, n.
\label{eq:10}
\end{equation}
\par Rewriting \eqref{eq:10} yields:
\begin{equation}
-\mathbf{x}^{\top} \mathbf{p}_j \leq -\|\mathbf{x}\|_2^2 \leq 0 \quad \text{for all } j = 1, 2, \ldots, n.
\label{eq:11}
\end{equation}
\par Recall that $\mathbf{p}_j$ is the gradient with respect to $j$th objective function. From \eqref{eq:10}, the directional derivative of each objective in the direction $-\mathbf{x}$ is non-positive. Therefore, $-\mathbf{x}$  constitutes a direction that does not increase (and thus improve) all objectives. In particular, when $\mathbf{x}$ $\neq$ $\mathbf{0}$, $-\mathbf{x}$ strictly decreases (and thus improves) all objectives, completing the proof.
\end{proof}
\par The main concept of TMGD is clearly outlined in \textcolor{blue}{Theorem} \ref{theorem:2}. What TMGD does is to seek a minimal norm point of the convex hull consisting of different objective gradients in every training iteration, and this minimal norm point will be used as the descent direction to update all objectives. We adopted the most widely used variant of TMGD in our studies, as shown in \textcolor{blue}{\textbf{Algorithm}} \ref{algorithm:1}. From \textcolor{blue}{\textbf{Line 1}} to \textcolor{blue}{\textbf{Line 3}}, \textcolor{blue}{\textbf{Algorithm}} \ref{algorithm:1} updates exclusive parameters of data-driven loss or physics loss firstly since those parameters will not affect the other loss. From \textcolor{blue}{\textbf{Line 6}} to \textcolor{blue}{\textbf{Line 14}}, \textcolor{blue}{\textbf{Algorithm}} \ref{algorithm:1} adopted a Frank-Wolfe algorithm (\cite{fliege2000steepest}) to find the minimum norm point in the gradient convex hull. In \textcolor{blue}{\textbf{Line 12}}, we can find out that the stop criterion is either meet the iteration limit or $\gamma \backsim 0$. $\gamma \backsim 0$ indicates that the minimum norm point that we obtain in the current iteration is close to the origin. When $\gamma = 0$, then the current solution is a Pareto stationary point, and we cannot find a non-zero direction that can improve any objectives, which lead to the following theorem:
\begin{algorithm}[htbp]
    \caption{TMGD [\cite{sener2018multi}]}
    \label{algorithm:1}
    \textbf{for} $j = 1$ \textbf{to }$N$ \textbf{do}\\
    \qquad $\boldsymbol{\theta}^{j} = \boldsymbol{\theta}^{j} - \eta \nabla_{\boldsymbol{\theta}^{j}}\hat{\mathcal{L}}^{j}(\boldsymbol{\theta}^{sh},\boldsymbol{\theta}^{j})$ \\
    \textbf{end for} \\
    $\alpha^{1},..,\alpha^{N}$ = \textsc{FrankWolfeSolver}($\boldsymbol{\theta}$), $\mathbf{w} = \sum_{j = 1}^{N}\alpha^{j}\nabla_{\boldsymbol{\theta}^{sh}}\hat{\mathcal{L}}^{j}(\boldsymbol{\theta}^{sh},\boldsymbol{\theta}^{j})$ \\
    $\boldsymbol{\theta}^{sh} = \boldsymbol{\theta}^{sh} - \eta\cdot\mathbf{w}$ \\
    \textbf{Procedure }\textsc{FrankWolfeSolver}($\boldsymbol{\theta}$) \\
    \qquad Initialize $\boldsymbol{\alpha} = (\alpha^{1},..,\alpha^{T}) = (\frac{1}{T},...,\frac{1}{T})$ \\
    \qquad Precompute $\mathbf{M}$ \textbf{st.} $\mathbf{M_{i,j}} = (\nabla_{\boldsymbol{\theta}^{sh}}\hat{\mathcal{L}}^{i}(\boldsymbol{\theta}^{sh},\boldsymbol{\theta}^{i}))^{\top}(\nabla_{\boldsymbol{\theta}^{sh}}\hat{\mathcal{L}}^{j}(\boldsymbol{\theta}^{sh},\boldsymbol{\theta}^{j}))$ \\
    \qquad\textbf{repeat} \\
    \qquad\qquad $\hat{t} = \argmin_{r} \sum_{t}\alpha^{t}\mathbf{M}_{rt}$ \\
    \qquad\qquad $\hat{\gamma} = \argmin_{\gamma} ((1 - \gamma)\boldsymbol{\alpha} + \gamma\mathbf{e}_{\hat{t}})^{\top}((1 - \gamma)\boldsymbol{\alpha} + \gamma\mathbf{e}_{\hat{t}})$ \\
    \qquad\textbf{until} $\gamma \backsim 0$ \textbf{or }Number of Iterations Limit\\
    \qquad\textbf{return} $\alpha^{1},..,\alpha^{T}$ \\
    \textbf{end procedure}
\end{algorithm}
\begin{theorem} \label{theorem:3}
Considering the following optimization problem:
\begin{equation}
\underset{u,\alpha}{\min}\left\{u = \bigg\|\sum_{j=1}^{N}\alpha^{j}\nabla_{\boldsymbol{\theta}^{sh}}{\mathcal{F}}^{\top}(\boldsymbol{\boldsymbol{\theta}^{sh},\boldsymbol{\theta}^{j}})\bigg\|\left|\sum_{j=1}^{N}\alpha^{j} = 1; \alpha^{j} \geq 0, \forall j \right.\right\} \label{eq:12}
\end{equation}
\par Either solution to this optimization problem is 0 and the resulting point satisfies the KKT conditions, or the solution gives a descent direction $-\sum_{j = 1}^{N}\alpha_{j}^{*}\nabla_{\boldsymbol{\theta}^{sh}}{\mathcal{F}}^{\top}(\boldsymbol{\boldsymbol{\theta}^{sh},\boldsymbol{\theta}^{j}})$ that improves all objectives.
\end{theorem}
\begin{proof}
    See \textbf{\textcolor{blue}{Theorem 2.2}} of \cite{desideri2012multiple}.
\end{proof}
\par \textcolor{blue}{\textbf{Theorem}} \ref{theorem:3} essentially ensures that the TMGD can theoretically achieve Pareto stationarity.
\subsubsection{Dual cone gradient descent}
\par In this section, $\nabla \mathcal{L}_{p}(\boldsymbol{\theta}_{t})$, $\nabla \mathcal{L}_{d}(\boldsymbol{\theta}_{t})$, $\nabla \mathcal{L}(\boldsymbol{\theta}_{t})$ represent physics loss, data loss, and total loss, respectively. $\boldsymbol{\theta}_{t}$ is the overall parameter set for all physics and data loss in $t$th iteration. $\phi_{t}$ is the angle between $\nabla \mathcal{L}_{d}(\boldsymbol{\theta}_{t})$ and $\nabla \mathcal{L}_{p}(\boldsymbol{\theta}_{t})$. $\|$ is the parallel component symbol, the parallel component of the vector $\mathbf v$ along $\mathbf a$, defined by:
\begin{equation} \label{eq:13}
\mathbf v_{\parallel \mathbf a}
= \frac{\mathbf v \cdot \mathbf a}{\|\mathbf a\|^2}\,\mathbf a = \mathbf{Proj}_{\mathbf{a}}\mathbf{v}
\end{equation}
,i.e. the projection of $\mathbf{v}$ onto $\mathbf{a}$.
\begin{definition}[Dual cone] \label{definition:6} Let $\mathbf{K}$ be a cone of $\mathbb{R}^d$. Then, the set 
\begin{equation} \label{eq:14}
\mathbf{K}^* = \{ y | \langle x,y\rangle \geq 0 \quad \mbox{for all } x\in \mathbf{K}, y\in \mathbb{R}^d\}
\end{equation}
is called the \textit{dual cone} of $\mathbf{K}$. 
\end{definition}
\par The concept of dual cone shown in \textcolor{blue}{\textbf{Definition}} \ref{definition:6} plays an important role in the DCGD algorithm. However, in practice, we are more interesting of its subspace $\mathbf{G}_{t}$, as shown in \textcolor{blue}{Figure} \ref{fig:1} and mathematically defined in \textcolor{blue}{Eq} \ref{eq:15}. Any direction such that $\mathbf{d} \in \mathbf{G}_{t}$ is a common descent direction since $\langle \mathbf{d},\nabla \mathcal{L}_r(\theta_t)\rangle \geq 0$ and $\langle \mathbf{d},\nabla \mathcal{L}_b(\theta_t)\rangle \geq 0$. Also, from \textcolor{blue}{Eq} \ref{eq:15}, $\mathbf{G}_{t}$ is essentially a conic combination of $\nabla \mathcal{L}(\theta_{t})_{\|(\nabla \mathcal{L}_{p}(\theta_{t}))^\perp}$ and $\nabla_t \mathcal{L}(\theta_{t})_{\|(\nabla \mathcal{L}_{d}(\theta_{t}))^\perp}$, which is very easy to compute compare to compute the $\mathbf{K}^{*}_{t}$.
\begin{figure}
    \centering
    \includegraphics[width=0.6\linewidth]{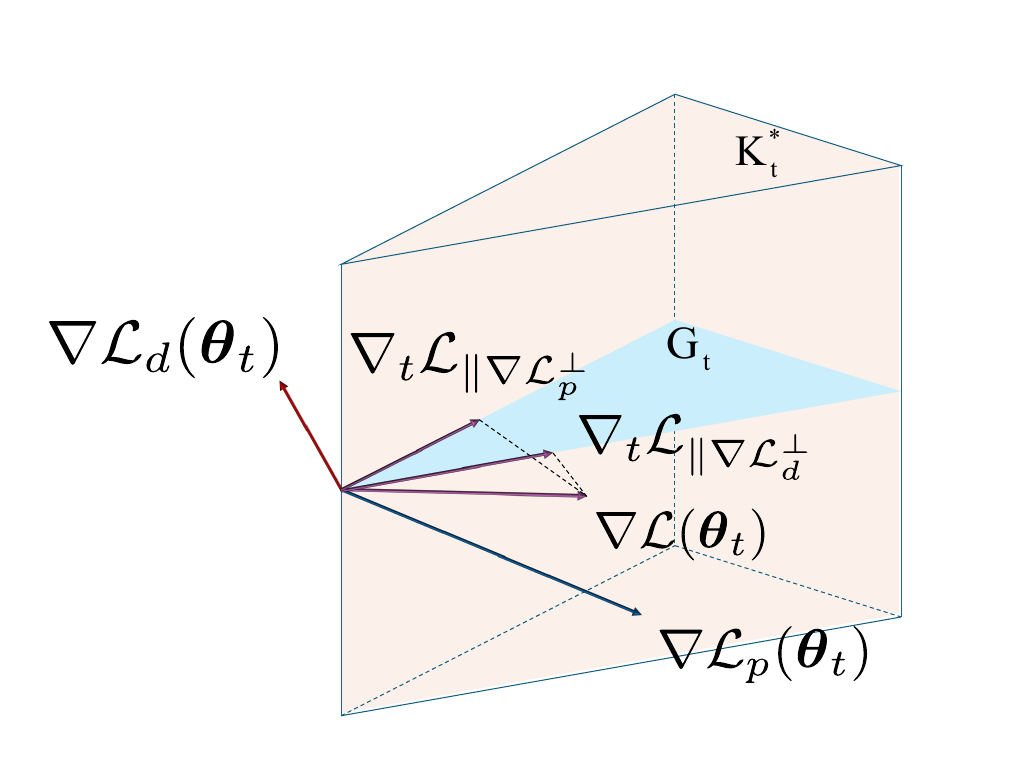}
    \caption{Visualization of dual cone region $\mathbf{K}^{*}_{t}$ and its subspace $\mathbf{G}_{t}$}
    \label{fig:1}
\end{figure}
\begin{equation}
\label{eq:15}
    \mathbf{G}_t := \left\{c_1 \nabla \mathcal{L}(\theta_{t})_{\|(\nabla \mathcal{L}_{p}(\theta_{t}))^\perp} + c_2 \nabla_t \mathcal{L}(\theta_{t})_{\|(\nabla \mathcal{L}_{d}(\theta_{t}))^\perp} \big| c_1, c_2\geq 0, c_1, c_2 \in \mathbb{R} \right\}
\end{equation}
\begin{figure}[htbp]
  \centering
  \begin{subfigure}{0.6\textwidth}
    \includegraphics[width=\linewidth]{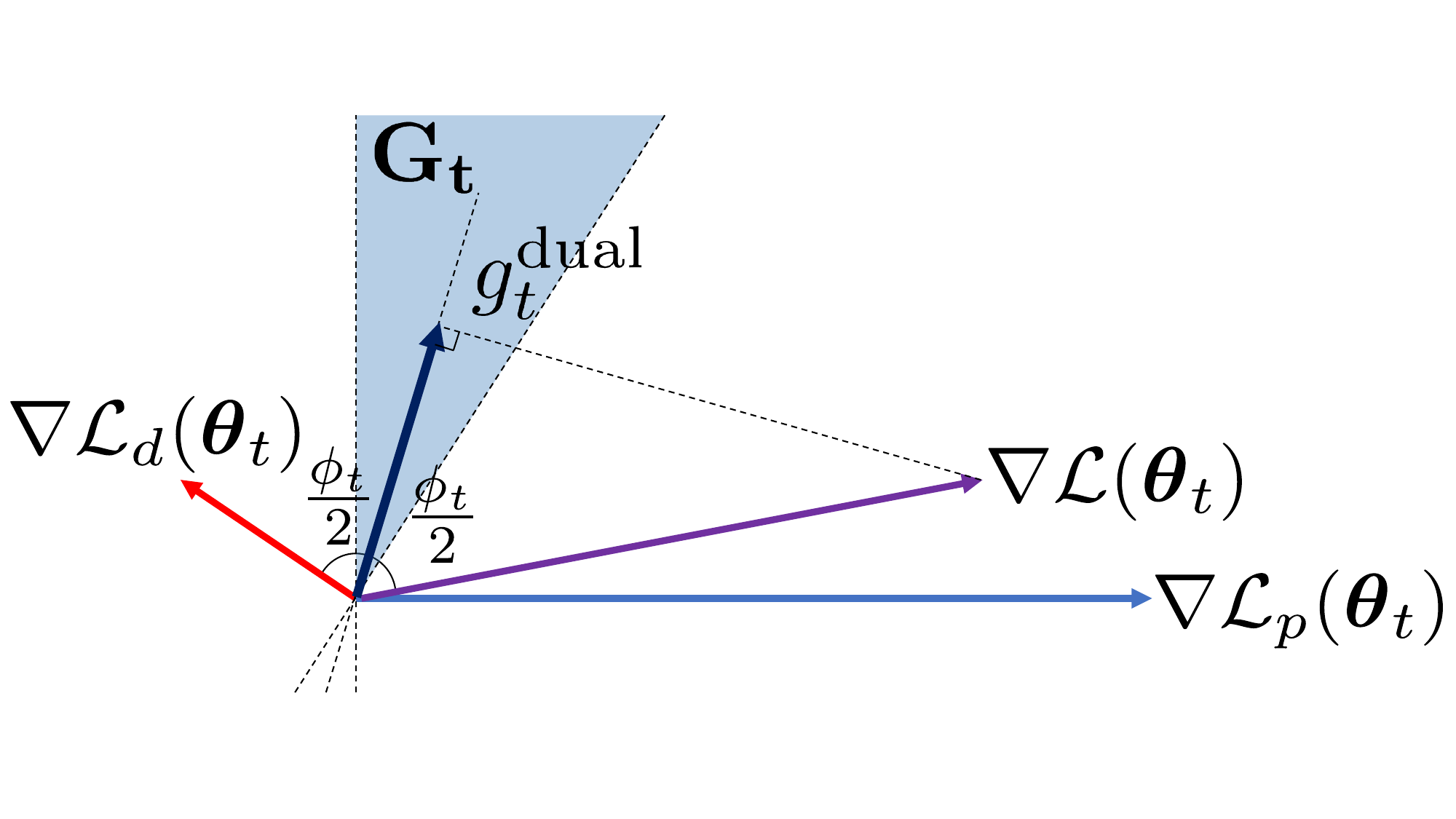}
    \caption{DCGD-Center}
    \label{fig:sub1}
  \end{subfigure}
  \begin{subfigure}{0.6\textwidth}
    \includegraphics[width=\linewidth]{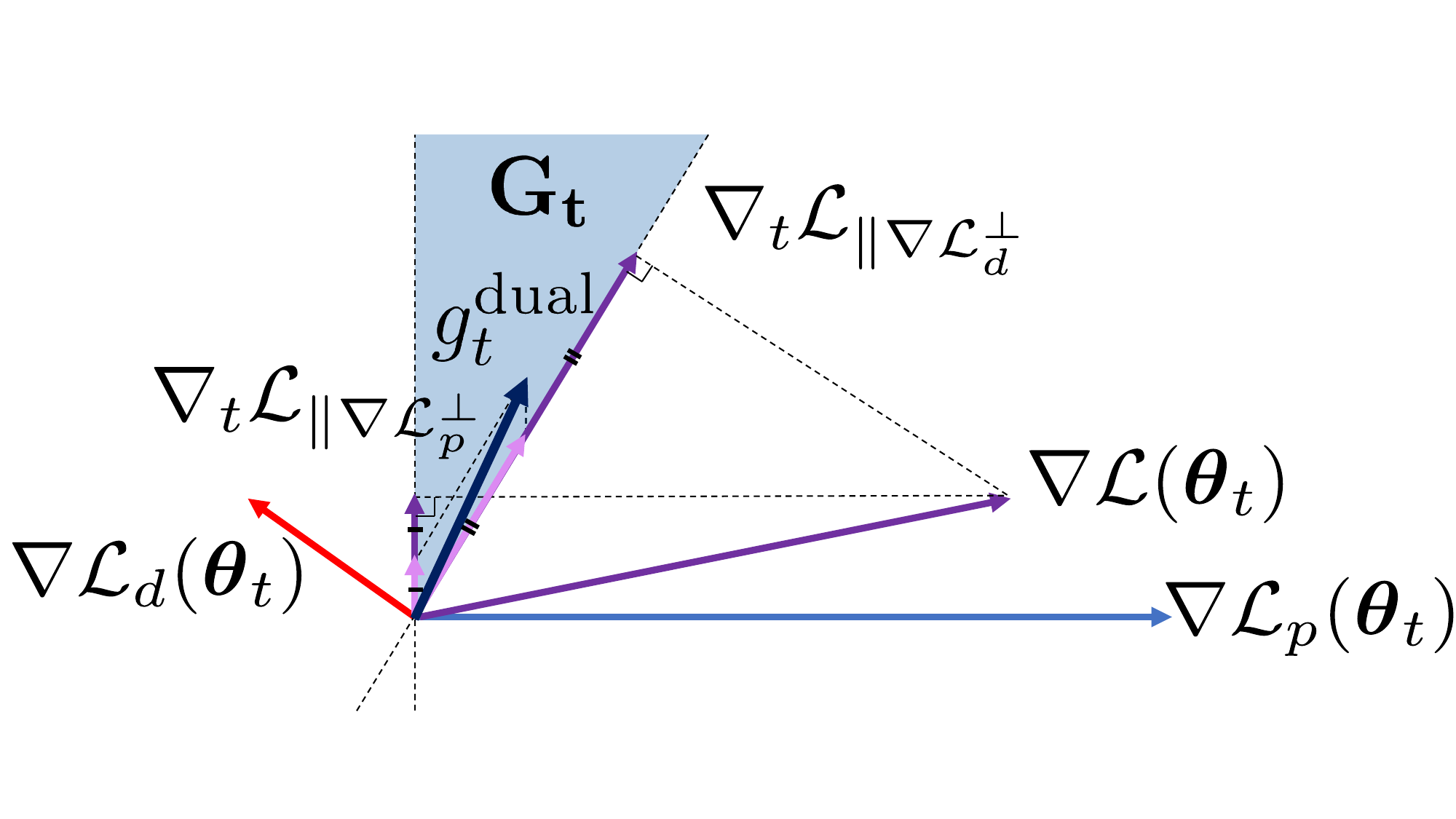}
    \caption{DCGD-AVG}
    \label{fig:sub2}
  \end{subfigure}
  \begin{subfigure}{0.6\textwidth}
    \includegraphics[width=\linewidth]{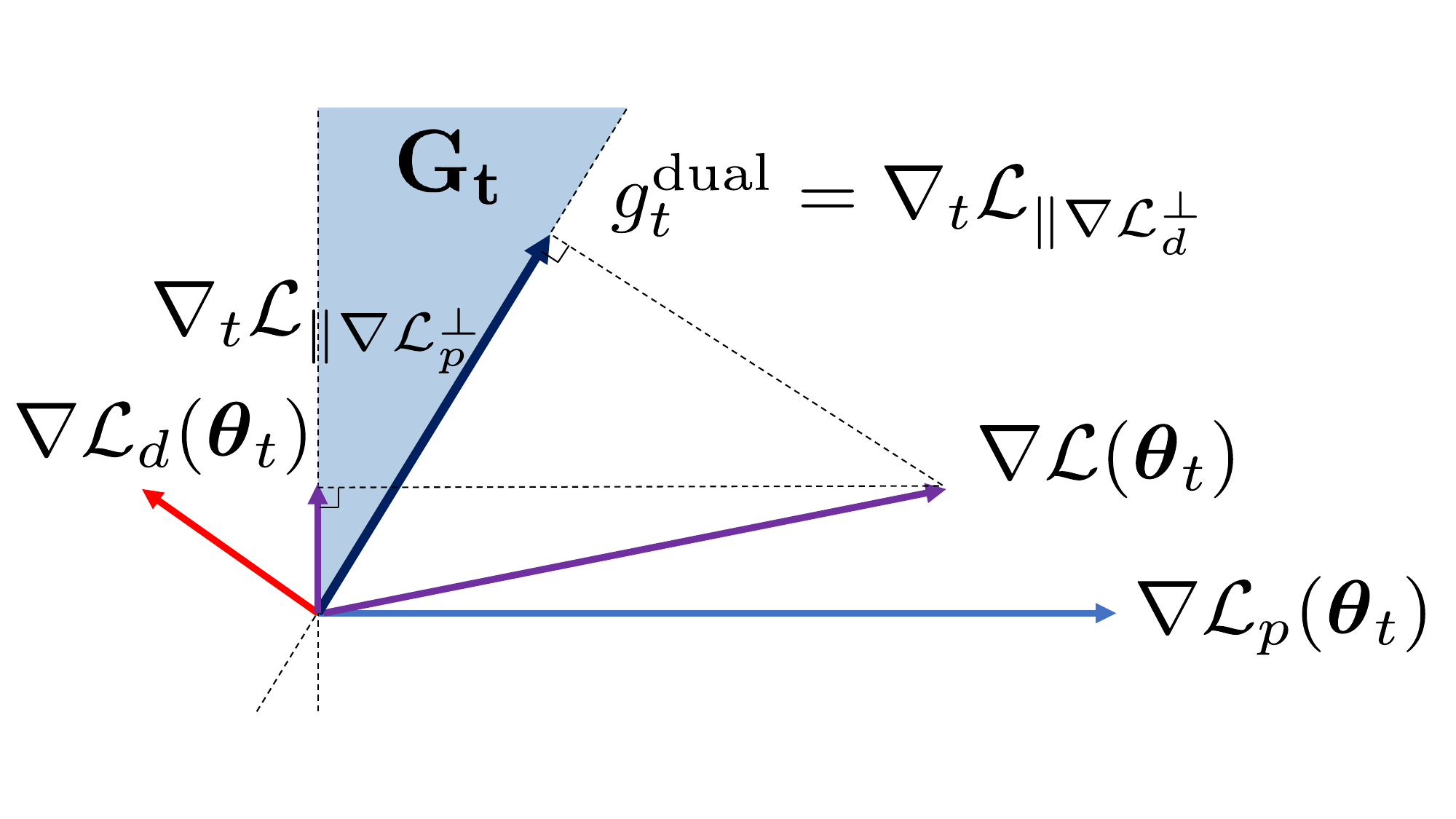}
    \caption{DCGD-Proj}
    \label{fig:sub3}
  \end{subfigure}
  \caption{Various Strategies for Dual Cone Gradient Descent}
  \label{fig:2}
\end{figure}
\par \textcolor{Blue}{\textbf{Algorithm}} \ref{algorithm:2} shows the core idea of DCGD, that is to find a common descent direction from $\mathbf{G}^{*}_{t}$, and since any direction that subset to $\mathbf{G}_{t}$ is a valid common descent direction. Thus, as shown in \textcolor{blue}{\textbf{Figure}} \ref{fig:2}, there are several variants of DVGD algorithms\footnote{For notation simplicity, $\nabla \mathcal{L}_{\|\nabla \mathcal{L}_{p}^\perp}$ and $\nabla_t \mathcal{L}_{\|\nabla \mathcal{L}_{d}^\perp}$ are used to represent $\nabla \mathcal{L}(\theta_{t})_{\|(\nabla \mathcal{L}_{p}(\theta_{t}))^\perp}$ and $\nabla_t \mathcal{L}(\theta_{t})_{\|(\nabla \mathcal{L}_{d}(\theta_{t}))^\perp}$}, the main difference is the way of picking a common descent direction that belongs to $\mathbf{G}_{t}$. The pseudo code of algorithms DCGD-Center, DCGD-Average, and DCGD-Proj are shown as follows:
\begin{algorithm}[htbp]
    \caption{DCGD (Center) [\cite{hwang2024dual}]}
    \label{algorithm:2}
    \textbf{Require:} learning rate $\lambda$, max epoch $T$, initial point $\theta_{0}$, gradient threshold $\varepsilon$, conflict threshold $\alpha$ \\
    \textbf{for} $t = 1$ \textbf{to }$T$ \textbf{do}\\
    \qquad \textbf{if } $\pi - \alpha < \phi_{t} \text{ or } \|\nabla \mathcal{L}(\boldsymbol{\theta}_{t})\| < \varepsilon$ \textbf{ then} \\
    \qquad\qquad \textbf{break} \\
    \qquad \textbf{end if} \\
    \qquad\ $g_{t}^{c} = \frac{\nabla \mathcal{L}_{d}(\boldsymbol{\theta}_{t})}{\|\nabla \mathcal{L}_{d}(\boldsymbol{\theta}_{t})\|} + \frac{\nabla \mathcal{L}_{p}(\boldsymbol{\theta}_{t})}{\|\nabla \mathcal{L}_{p}(\boldsymbol{\theta}_{t})\|}$ \\
    \qquad\ $g_{t}^{dual} = \frac{\langle g_{t}^{c}, \nabla \mathcal{L}(\boldsymbol{\theta}_{t}) \rangle}{\|g_{t}^{c}\|^{2}}g_{t}^{c}$ \\
    \qquad\ $\boldsymbol{\theta}_{t} = \boldsymbol{\theta}_{t - 1} - \lambda g_{t}^{dual}$ \\
    \textbf{end for} \\
\end{algorithm}

\begin{algorithm}[htbp]
    \caption{DCGD (Average) [\cite{hwang2024dual}]}
    \label{algorithm:3}
    \textbf{Require:} learning rate $\lambda$, max epoch $T$, initial point $\theta_{0}$, gradient threshold $\varepsilon$, conflict threshold $\alpha$ \\
    \textbf{for} $t = 1$ \textbf{to }$T$ \textbf{do}\\
    \qquad \textbf{if } $\pi - \alpha < \phi_{t} \text{ or } \|\nabla \mathcal{L}(\theta_{t})\| < \varepsilon$ \textbf{ then} \\
    \qquad\qquad \textbf{break} \\
    \qquad \textbf{end if} \\
    \qquad \textbf{if } $\nabla \mathcal{L}(\boldsymbol{\theta}_{t}) \notin \mathbf{K}^{*}$ \textbf{ then} \\
    \qquad\qquad $g_{t}^{dual} = \frac{1}{2}\nabla \mathcal{L}(\boldsymbol{\theta}_{t})_{\|(\nabla \mathcal{L}_{p}(\boldsymbol{\theta}_{t}))^{\bot}} + \frac{1}{2}\nabla \mathcal{L}_{\|(\nabla \mathcal{L}_{d}(\boldsymbol{\theta}_{t}))^{\bot}}$ \\
    \qquad \textbf{else }  \\
    \qquad\qquad $g_{t}^{dual} = \nabla \mathcal{L}(\boldsymbol{\theta}_{t})$ \\
    \qquad \textbf{end if} \\
    \qquad\ $\boldsymbol{\theta}_{t} = \boldsymbol{\theta}_{t - 1} - \lambda g_{t}^{dual}$ \\
    \textbf{end for} \\
\end{algorithm}

\begin{algorithm}[htbp]
    \caption{DCGD (Projection) [\cite{hwang2024dual}]}
    \label{algorithm:4}
    \textbf{Require:} learning rate $\lambda$, max epoch $T$, initial point $\theta_{0}$, gradient threshold $\varepsilon$, conflict threshold $\alpha$ \\
    \textbf{for} $t = 1$ \textbf{to }$T$ \textbf{do}\\
    \qquad \textbf{if } $\pi - \alpha < \phi_{t} \text{ or } \|\nabla \mathcal{L}(\boldsymbol{\theta}_{t})\| < \varepsilon$ \textbf{ then} \\
    \qquad\qquad \textbf{break} \\
    \qquad \textbf{end if} \\
    \qquad \textbf{if } $\nabla \mathcal{L}(\boldsymbol{\theta}_{t}) \in \mathbf{K}^{*}$ \textbf{ then} \\
    \qquad\qquad $g_{t}^{dual} = \nabla \mathcal{L}(\boldsymbol{\theta}_{t})$ \\
    \qquad \textbf{else if } $\nabla \mathcal{L}(\boldsymbol{\theta}_{t}) \notin \mathbf{K}^{*} \text{ and }\langle\nabla \mathcal{L}(\boldsymbol{\theta}_{t}),\nabla \mathcal{L}_{p}(\boldsymbol{\theta}_{t})\rangle <0$ \textbf{ then} \\
    \qquad\qquad $g_{t}^{dual} = \nabla \mathcal{L}(\boldsymbol{\theta}_{t})_{\|(\nabla \mathcal{L}_{p}(\boldsymbol{\theta}_{t}))^{\bot}}$ \\
    \qquad \textbf{else if } $\nabla \mathcal{L}(\boldsymbol{\theta}_{t}) \notin \mathbf{K}^{*} \text{ and }\langle\nabla \mathcal{L}(\boldsymbol{\theta}_{t}),\nabla \mathcal{L}_{d}(\boldsymbol{\theta}_{t})\rangle <0$ \textbf{ then} \\
    \qquad\qquad $g_{t}^{dual} = \nabla \mathcal(\boldsymbol{\theta}_{t})_{\|(\nabla \mathcal{L}_{d}(\boldsymbol{\theta}_{t})^{\bot}}$ \\
    \qquad \textbf{end if} \\
    \textbf{end for} \\
\end{algorithm}
\par From \textbf{\textcolor{blue}{Line 2}} to \textbf{\textcolor{blue}{Line 5}} in \textbf{\textcolor{blue}{Algorithm}} \ref{algorithm:2}, \ref{algorithm:3}, and \ref{algorithm:4}, they all share the same stop criteria, that is:
\begin{itemize}
    \item $\pi - \alpha < \phi_{t}$, which indicates that the angle between the data-driven loss gradient and the loss gradient physics $\phi_{t}$ is very close to 0. When $\phi_{t} \rightarrow 0$, the total parameters set is a Pareto Stationary point since no non-zero vector exists that can be served as the common descent direction that can decrease all objectives/losses.
    \item Gradients vanish, which indicates that the model has converged.
    \item Meeting the iteration limit.
\end{itemize}
\par In \textbf{\textcolor{blue}{Algorithm}} \ref{algorithm:2}, we first calculate a normalized sum of two directions $g^{c}_{t}$. $\anglebetween{g^{c}_{t}}{\nabla \mathcal{L}_{d}(\boldsymbol{\theta}_{t})} = \anglebetween{g^{c}_{t}}{\nabla \mathcal{L}_{p}(\boldsymbol{\theta}_{t})} = \frac{\phi_{t}}{2}$  \footnote{ we use $\anglebetween{\mathbf{u}}{\mathbf{v}}$ to represent the angle between two vectors, mathematically:
\begin{equation} \label{eq:16}
\anglebetween{\mathbf{u}}{\mathbf{v}}
:= \arccos\!\biggl(
  \frac{\mathbf{u}\cdot\mathbf{v}}
       {\|\mathbf{u}\|\;\|\mathbf{v}\|}
\biggr)
\end{equation}}, and then, we project the gradient of the total loss $\nabla \mathcal{L}(\theta_{t})$ onto $g^{c}_{t}$ to get the final common descent direction $g^{dual}_{t}$. 
\par In \textbf{\textcolor{blue}{Algorithm}} \ref{algorithm:3}, it seeks a simple compromise direction by equally weighting the two objective gradients.  Concretely, let
\begin{equation} \label{eq:17}
\nabla \mathcal{L}(\boldsymbol{\theta}_{t}) = \nabla \mathcal{L}_{p}(\boldsymbol{\theta}_{t}) + \nabla\mathcal{L}_{d}(\boldsymbol{\theta}_{t})
\end{equation}
be the \emph{averaged gradient}.
If this direction simultaneously descends both losses, i.e.
\begin{equation} \label{eq:18}
\nabla \mathcal{L}(\boldsymbol{\theta}_{t}) \in \mathbf{K}^{*}
\end{equation}
we use \(\nabla \mathcal{L}(\boldsymbol{\theta}_{t})\) as the update.  Otherwise, we first project \(\nabla \mathcal{L}(\boldsymbol{\theta}_{t})\) onto two boundary $\nabla \mathcal{L}_{\|(\nabla \mathcal{L}_{p}(\boldsymbol{\theta}_{t}))^{\bot}}$  and $\nabla \mathcal{L}_{\|(\nabla \mathcal{L}_{d}(\theta_{t}))^{\bot}}$ of $\mathbf{G}_{t}$, respectively. Then, calculate the arithmetic average as the update:
\begin{equation} \label{eq:19}
g^{dual}_{t}
 = \frac{1}{2}\nabla \mathcal{L}_{\|(\nabla \mathcal{L}_{p}(\boldsymbol{\theta}_{t}))^{\bot}} + \frac{1}{2}\nabla \mathcal{L}(\boldsymbol{\theta}_{t})_{\|(\nabla \mathcal{L}_{d}(\boldsymbol{\theta}_{t}))^{\bot}}
\end{equation}
This guarantees that the chosen descent direction reduces both \(\mathcal{L}_{p}(\boldsymbol{\theta}_{t})\) and \(\nabla\mathcal{L}_{d}(\boldsymbol{\theta}_{t})\) simultaneously.
\par In \textbf{\textcolor{blue}{Algorithm}} \ref{algorithm:4}, at iteration \(t\), the dual gradient \(g_t^{\mathrm{dual}}\) is constructed based on whether the current gradient of the total loss lies inside the dual cone \(\mathbf{K}^*\) and whether it conflicts with bound gradients $\nabla \mathcal{L}_{d}(\boldsymbol{\theta}_{t})$ and $\nabla \mathcal{L}_{p}(\boldsymbol{\theta}_{t})$:
\begin{equation} \label{eq:20}
g_t^{\mathrm{dual}}
=
\begin{cases}
\nabla \mathcal{L}(\boldsymbol{\theta}_t), 
& \nabla \mathcal{L}(\boldsymbol{\theta}_t)\in \mathbf{K}^*, \\[1ex]
\displaystyle
\nabla \mathcal{L}(\boldsymbol{\theta}_{t})_{\|(\nabla \mathcal{L}_{p}(\boldsymbol{\theta}_{t}))^{\bot}}, 
& \nabla \mathcal{L}(\boldsymbol{\theta}_t)\notin \mathbf{K}^*,\ 
\langle\nabla \mathcal{L}(\boldsymbol{\theta}_t), \nabla \mathcal{L}_p(\boldsymbol{\theta}_t)\rangle < 0, \\[2ex]
\displaystyle
\nabla \mathcal(\boldsymbol{\theta}_{t})_{\|(\nabla \mathcal{L}_{d}(\boldsymbol{\theta}_{t})^{\bot}}, 
& \nabla \mathcal{L}(\boldsymbol{\theta}_t)\notin \mathbf{K}^*,\ 
\langle\nabla \mathcal{L}(\boldsymbol{\theta}_t), \nabla \mathcal{L}_d(\boldsymbol{\theta}_t)\rangle < 0.
\end{cases}
\end{equation}
This strategy preserves descent direction while removing components that conflict with the active bound gradients.
\par Finally, the convergence and stationarity analysis of the DCGD algorithm are stated in the following theorems:
\begin{theorem}\label{theorem:4}
 Assume that both loss functions, $\mathcal{L}_d(\cdot)$ and $\mathcal{L}_p(\cdot)$, are differentiable and the total gradient $\nabla \mathcal{L}(\cdot)$ is $L$-Lipschitz continuous with $L > 0$. If $g_t^{\text{dual}}$ satisfies the following two conditions:
\begin{enumerate}
\item[(i)] $2\langle \nabla \mathcal{L}(\boldsymbol{\theta}_t), g_t^{\text{dual}}\rangle - \|g_t^{\text{dual}}\|^2 \geq 0$, 
\item[(ii)] There exists $M>0$ such that $\|g_t^{\text{dual}}\|\geq M\|\nabla \mathcal{L}(\boldsymbol{\theta}_t)\|$,
\end{enumerate}
then, for $\lambda \leq \frac{1}{2L}$, DCGD in Algo.~\ref{algorithm:1} converges to a Pareto-stationary point, or converges as 
\begin{align} \label{eq:21}
    \frac{1}{T+1}\sum_{t=0}^T\|\nabla \mathcal{L}(\boldsymbol{\theta}_t)\|^2 \leq \frac{2\left(\mathcal{L}(\boldsymbol{\theta}_0)-\mathcal{L}(\boldsymbol{\theta}^*)\right)}{\lambda M(T+1)}.   
\end{align}
\end{theorem}
\begin{proof}
See \textbf{\textcolor{blue}{theorem 4.5}} in \cite{hwang2024dual}.
\end{proof}
\par \textcolor{blue}{\textbf{Theorem}} \ref{theorem:4} also ensures that the DCGD can theoretically achieve Pareto stationarity or good convergence properties.

\subsection{Potential limitation of linear scalarization}
\par In most contemporary PIML studies, the primary and most effective method for calculating the final total loss is through linear scalarization, as shown in the form illustrated in \textcolor{blue}{Eq} \ref{eq:1}. As is aware by many previous studies, the specific values of the coefficients $\alpha$ and $\beta$ cause significant difficulties in fine-tuning. Since the value of the coefficients will not change during the training process, it may lead to more serious problems:
\begin{itemize}
    \item Linear scalarization constructs a combined loss function as a linear weighted sum\footnote{\label{sec:7} \label{line:7} Here, we adopted a more general formulation, that is, we use $\theta$ to represent the entire parameter set and do not divide them into data-exclusive parameter sets and physics-exclusive parameter sets}:
\[
\mathcal{L}_\alpha(\boldsymbol{\theta}) = \alpha \mathcal{L}_1(\boldsymbol{\theta}) + (1 - \alpha)\mathcal{L}_2(\boldsymbol{\theta}), \quad \alpha \in [0,1]
\]
This implicitly assumes a linear trade-off between the two loss functions, enforcing a fixed exchange ratio determined by the scalarization weight $\alpha$. However, in practice, the true relationship between multiple objectives often exhibits nonlinear or complex interactions. Thus, linear scalarization may inadequately represent the intrinsic preference structure, leading to suboptimal or misleading solutions when exploring the Pareto frontier.
    \item When optimizing via linear scalarization, non-convexities inherent in one or both loss functions may cause the optimization process to continue beyond a point that is already Pareto stationary. For instance, consider a point $\theta^*$ where $\mathcal{L}_1$ reaches a local minimum so that $\nabla \mathcal{L}_1(\theta^*) = 0$, while $\nabla \mathcal{L}_2(\theta^*) \neq 0$. In this case, the gradient of the scalarized loss becomes
\[
\nabla \mathcal{L}_\alpha(\boldsymbol{\theta}^*) = (1 - \alpha) \nabla \mathcal{L}_2(\boldsymbol{\theta}^*)
\]
which is non-zero as long as $\alpha < 1$. As a result, gradient-based optimization continues to move $\theta$ in the direction of decreasing $\mathcal{L}_2$, even though such a step may increase $\mathcal{L}_1$, thus breaking Pareto stationarity. In Multi-Gradient Descent Algorithms (MGDAs), this situation will not occur because they follow a common direction of descent.

This behavior illustrates how linear scalarization may fail to preserve Pareto stationary points: despite $\theta^*$ being Pareto stationary since no direction strictly improves both losses or improves one without worsening the other. However, the optimization process ignores this due to the weighted descent direction, leading to movement away from a theoretically valid solution. This flaw becomes particularly evident in multi-objective settings with non-convex landscapes, where linear scalarization may guide the optimization toward undesirable regions of the objective space. 
\end{itemize}

\begin{theorem}\label{theorem:5}
Given two continuously differentiable loss functions $\mathcal{L}_1(\theta)$ and $\mathcal{L}_2(\theta)$ defined over a parameter space $\Theta \subseteq \mathbb{R}^n$, the following statements hold:
\begin{enumerate}
    \item Every minimizer of the total loss obtained via linear scalarization, defined as
    \[
    \mathcal{L}_\beta(\theta) = \beta \mathcal{L}_1(\theta) + (1 - \beta) \mathcal{L}_2(\theta) \quad \beta \in (0,1)
    \]
    can theoretically be obtained by the Multi-Gradient Descent Algorithm (MGDA), which seeks common descent directions.
    \item However, certain Pareto stationary points found by MGDA cannot be theoretically achieved by optimizing the scalarized total loss for any interior linear scalarization weights $\beta\in(0,1)$.

\end{enumerate}
\end{theorem}

\begin{proof}
See \textbf{\textcolor{blue}{Appendix}} \ref{ap:1}.
\end{proof}
\par \textbf{\textcolor{blue}{Theorem}}~\ref{theorem:5} states that training a PIML model with multi‑gradient descent algorithms (e.g., TMGD and DCGD) is guaranteed to achieve performance at least as good as, and potentially better than, training it with a single objective formed by linear scalarization of the losses. However, the scalarized approach offers no such guarantee in the opposite direction. 

\begin{figure}
    \centering
    \includegraphics[width=1.0\linewidth]{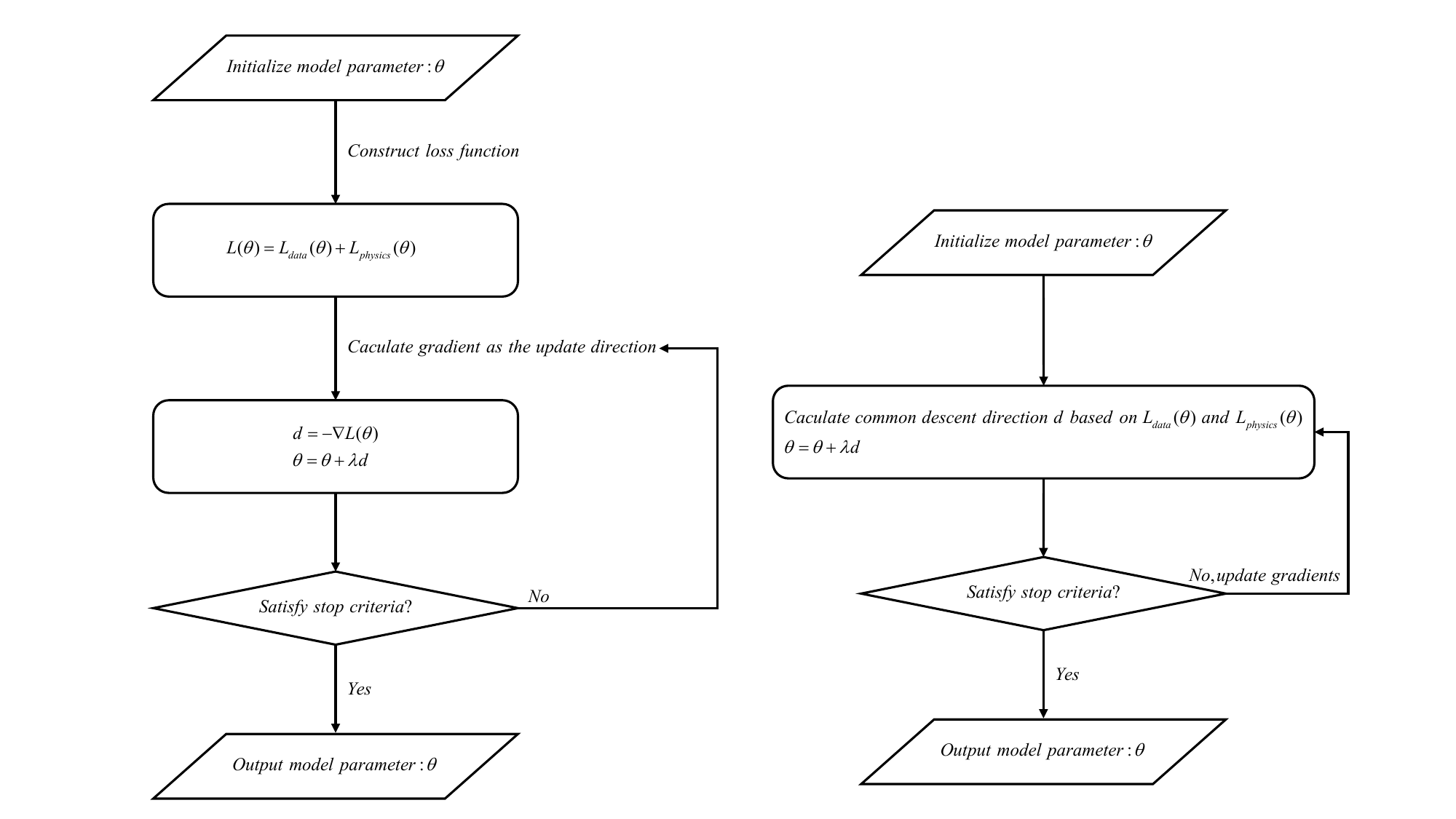}
    \caption{Difference between traditional training methods and training approaches based on multi-gradient descent}
    \label{fig:3}
\end{figure}

\label{sec:8} \label{line:8}

As illustrated in \textcolor{blue}{\textbf{Figure}}~\ref{fig:3}, we provide a summary flow chart that highlights the fundamental difference between traditional training methods and training approaches based on multi-gradient descent.

On the left, the traditional method first constructs a linear scalarization loss function, typically of the form \(\mathcal{L}(\boldsymbol{\theta}) = \alpha \mathcal{L}_{\mathrm{data}}(\boldsymbol{\theta}) + \beta \mathcal{L}_{\mathrm{physics}}(\boldsymbol{\theta})\), and then computes the gradient of this combined loss to update model parameters using standard gradient descent, where $\lambda$ is the step length, it could be either fixed or adaptive.

In contrast, the right side shows the pipeline of training approaches based on multi-gradient descent, which treats the data loss and physics loss as two distinct objectives. Rather than combining them into a single scalarized function, this method directly formulates a multi-objective optimization problem and computes a common descent direction \(\mathbf{d}\) that ensures both losses decrease simultaneously. This approach eliminates the need for manual tuning of scalarization weights and enables more principled updates for all parameter types.

% \begin{algorithm}[htbp]
%     \caption{DCGD with ADAM [\cite{hwang2024dual}]}
%     \label{algorithm:6}
%     \textbf{Require:} learning rate $\lambda$, max epoch $T$, betas $\beta_{1},\beta_{2}$, DCGD operator DCGD($\cdot$) \\
%     \textbf{for} $t = 1$ \textbf{to }$T$ \textbf{do}\\
%     \qquad $g_{t}^{dual} = \text{DCGD}(\mathcal{L}_{r}(\theta)),\mathcal{L}_{b}(\theta))$ \\
%     \qquad $g_{t}  \gets g_{t}^{dual}$ \\
%     \qquad $m_{t}  \gets \beta_{1}m_{t-1} + (1-\beta_{1})g_{t}$\\
%     \qquad $v_{t}  \gets \beta_{2}v_{t-1} + (1-\beta_{2})g_{t}^{2}$\\
%     \qquad $\widehat{m}_{t}  \gets \frac{m_{t}}{1-\beta_{1}^{t}}$\\
%     \qquad $\theta_{t}  \gets \theta_{t-1} - \gamma_{t}\frac{\widehat{m}_{t}}{\sqrt{v_{t}}+\epsilon}$\\
%     \textbf{end for} \\
% \end{algorithm}

% \begin{algorithm}[htbp]
%     \caption{DCGD with a loss balancing method [\cite{hwang2024dual}]}
%     \label{algorithm:7}
%     \textbf{Require:} learning rate $\lambda$, max epoch $T$, loss balancing operator LB($\cdot$) \\
%     \textbf{for} $t = 1$ \textbf{to }$T$ \textbf{do}\\
%     \qquad $(\beta_{r},\beta_{b})_{t}= \text{LB}(\mathcal{L}_{r}(\theta)),\mathcal{L}_{b}(\theta))$ \\
%     \qquad $\mathcal{L}_{b}(\theta))  \gets \beta_{b}\mathcal{L}_{b}(\theta))$ \\
%     \qquad $\mathcal{L}_{r}(\theta))  \gets \beta_{r}\mathcal{L}_{r}(\theta))$ \\
%     \qquad $\text{Choose} g_{t}^{dual} \in \textbf{K}^{*}_{t}$\\
%     \qquad $\theta_{t} = \theta_{t - 1} - \lambda g_{t}^{dual}$ \\
%     \textbf{end for} \\
% \end{algorithm}
\section{Experiments} \label{section:4}
\par In this section, we apply TMGD and three different DCDA methods to two PIML studies about traffic flow modeling: (1) PIML model based on the macroscopic traffic flow model \cite{shi2021physics}. (2) PIML model based on the microscopic traffic flow model \cite{mo2021physics}.
\subsection{PIML model based on the macroscopic traffic flow model} \label{section:4.1}
\par The first example is a PIML model based on the macroscopic traffic flow model. The physics model used in this PIML study is the famous LWR model (\cite{lighthill1955kinematic} and \cite{richards1956shock}). The LWR model for one-dimensional traffic flow is given by:
\begin{equation} \label{eq:22}
\frac{\partial \rho}{\partial t} + \frac{\partial q(\rho)}{\partial x} = 0
\end{equation}
where $\rho,q$ represent traffic density and flow, respectively.
\par Following \cite{shi2021physics}, we adopted the same evaluation criterion, the $\mathbb{L}^{2}$ relative error, to assess prediction performance for traffic density and speed on the testing dataset, as shown in \textcolor{blue}{Eqs} \ref{eq:23} and \ref{eq:24}. We utilize real field data collected in Utah, which can be accessed at \url{https://github.com/UMD-Mtrail/Field-data-for-macroscopic-traffic-flow-model}. The dataset consists of data pairs $\{\mathbf{X}=(\mathbf{x,t}),\mathbf{Y} = (\mathbf{q,u})\}$, where $\mathbf{x}$ represents position, $\mathbf{t}$ represents time, and $\mathbf{q,u}$ represent average traffic flow and traffic density over five minutes.
\textbf{\textcolor{blue}{Figure}} \ref{fig:4} shows the architecture of the studied PIML model, and we referred to it as the LWR-PINN model. In addition to the Physics-uninformed neural network (PUNN) component present in all PIML studies, there is an FD learner module. The FD learner, which stands for fundamental diagram learner, is a multi-layer perceptron used to learn the density-flow relationship. As shown in the blue dashed line frame of \textbf{\textcolor{blue}{Figure}} \ref{fig:4}, the PUNN part of the LWR-PINN model is a multi-layer perceptron, the input dimension is two and the output dimension is one, implemented with eight hidden layers, each containing 20 nodes.  The FD learner, which approximates the fundamental diagram relationship, is implemented as a multilayer perceptron (MLP) with two hidden layers and 20 nodes per layer. In the blue dashed line frame of \textbf{\textcolor{blue}{Figure}} \ref{fig:3}, all arrows represent weights that need to be trained. In contrast, the arrows in the red dashed line frame of \textbf{\textcolor{blue}{Figure}} \ref{fig:4} indicate different operators such as the differentiation operator, the addition operator, and the division operator. For additional architectural details, we refer the reader to \cite{shi2021physics}.
\begin{equation}
    Err(\hat{\rho},\rho) = \frac{\sum_{i = 1}^{n} |\hat{\rho}(t^{i},x^{i};\theta) - \rho(t^{i},x^{i})|^{2}}{\sum_{i = 1}^{n}|\rho(t^{i},x^{i})|^{2}} \label{eq:23}
\end{equation}
\begin{equation}
    Err(\hat{u},u) = \frac{\sum_{i = 1}^{n} |\hat{u}(t^{i},x^{i};\theta) - u(t^{i},x^{i})|^{2}}{\sum_{i = 1}^{n}|u(t^{i},x^{i})|^{2}} \label{eq:24}
\end{equation}
\begin{figure}
    \centering
    \includegraphics[width=0.8\linewidth]{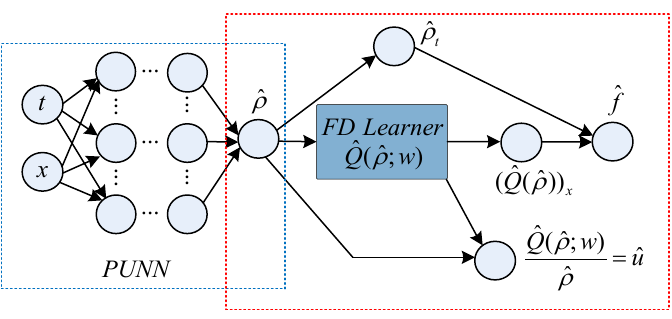}
    \caption{LWR-PINN architecture}
    \label{fig:4}
\end{figure}
\par The loss-function design for LWR-PINN is shown in \textcolor{blue}{\textbf{Eq}} \ref{eq:25}, where $N_{o}$ and $N_{a}$ denote the numbers of observation and auxiliary points, respectively. 
\begin{equation}
\begin{split}
    Loss_{\theta,\omega}^{LWR} &= \alpha \cdot MSE_{o} + \beta \cdot MSE_{a} \\
    &= \frac{\alpha_{1}}{N_{o}}\sum_{i = 1}^{N_{o}}|\hat{\rho}(t_{o}^{(i)},x_{o}^{(i)};\theta) - \rho^{(i)}|^{2} + \frac{\alpha_{2}}{N_{o}}\sum_{i = 1}^{N_{o}}|\hat{u}(t_{o}^{(i)},x_{o}^{(i)};\theta) - u^{(i)}|^{2} \\
    &+ \frac{\beta}{N_{a}}\sum_{j = 1}^{N_{a}}|\hat{f}(t_{a}^{(j)},x_{a}^{(j)};\theta,\omega)|^{2}
\end{split} \label{eq:25}
\end{equation}
\par For each observation point  
\begin{equation}  \label{eq:26}
  (\mathbf{X}_{i},\mathbf{Y}_{i})
  =\bigl((\mathbf{x}_{i},\mathbf{t}_{i}),(\boldsymbol{\rho}_{i},\mathbf{u}_{i})\bigr)
  \qquad i=1,\dots ,N_{o}
\end{equation}
both the feature vector $\mathbf{X}_{i}=(\mathbf{x}_{i},\mathbf{t}_{i})$ and the corresponding label vector $\mathbf{Y}_{i}=(\boldsymbol{\rho}_{i},\mathbf{u}_{i})$ are known and comes from the training data.  
For auxiliary points, we only know the feature vector  
\begin{equation}  \label{eq:27}
  \mathbf{Z}_{j}=(\mathbf{x}'_{j},\mathbf{t}'_{j}) 
  \qquad j=1,\dots ,N_{a},
\end{equation} 
These points do not need to be part of the training set, and they are typically generated by uniform sampling within the problem domain.  They serve to evaluate the physics-based residual. Concretely, each auxiliary point is supplied to the PINN as an input, which produces a predicted physics term $\hat{f}$. The physics residual at that point is generated through the operators shown in the red dashed line frame in \textcolor{blue}{\textbf{Figure}} \ref{fig:4}, represented as $|\hat{f} - 0|$.
\par We adopted a hyperparameter fine-tuning approach similar to that used in \cite{shi2021physics}. Specifically, we fixed the value of \(\alpha\) at 100 and adjusted only the value of \(\beta\). We designed several rounds of hyperparameter fine-tuning. In the first round, the values of \(\beta\) will be set to \([1, 10, 100, 1000, 10000]\) to determine the optimal magnitude interval, as we consider a hundredfold order of magnitude difference between \(\alpha\) and \(\beta\).
\begin{table}[htp!]
\centering
\caption{First round fine-tuning of the hyperparameter} \label{table:2}
\begin{tabular}{ccc}
\toprule
Coefficient value & $Err(\hat{\rho},\rho)$ & $ Err(\hat{u},u)$ \\
\midrule
$\alpha = 100$, $\beta = 1$ & 0.6298 $\pm$ 0.0024  & 0.2051 $\pm$ 0.0177   \\
$\alpha = 100$, $\beta = 10$ & 0.6295 $\pm$ 0.0033  & 0.2999 $\pm$ 0.2706   \\
$\checkmark\alpha = 100$, $\beta = 100$ & 0.6298 $\pm$ 0.0024  & 0.2079 $\pm$ 0.0174   \\
$\checkmark\alpha = 100$, $\beta = 1000$ & 0.6298 $\pm$ 0.0022  & 0.1966 $\pm$ 0.0097   \\
$\alpha = 100$, $\beta = 10000$ & 0.6308 $\pm$ 0.0020  & 0.2259 $\pm$ 0.0369   \\
\bottomrule
\end{tabular}
\end{table}
\par As indicated in \textcolor{blue}{Table} \ref{table:2}, the current optimal magnitude interval is \(\beta \in [100, 1000]\). Therefore, in the second round, the values of \(\beta\) will be adjusted to \([200, 300, 400, 500, 600, 700, 800, 900]\) to further refine the determination of the optimal magnitude interval, as shown in \textcolor{blue}{Table} \ref{table:3}. After two rounds of hyperparameter searching and some additional fine-tuning, the optimal hyperparameters are $\alpha = 100$ and $\beta = 500$.
\begin{table}[htp!]
\centering
\caption{second round fine-tuning of the hyperparameter} \label{table:3}
\begin{tabular}{ccc}
\toprule
Coefficient value & $Err(\hat{\rho},\rho)$ & $ Err(\hat{u},u)$ \\
\midrule
$\alpha = 100$, $\beta = 200$ & 0.6297 $\pm$ 0.0025  & 0.2317 $\pm$ 0.0743   \\
$\alpha = 100$, $\beta = 300$ & 0.6305 $\pm$ 0.0028  & 0.2130 $\pm$ 0.0203   \\
$\alpha = 100$, $\beta = 400$ & 0.6307 $\pm$ 0.0021  & 0.2031 $\pm$ 0.0081   \\
$\checkmark\alpha = 100$, $\beta = 500$ & 0.6291 $\pm$ 0.0023  & 0.1958 $\pm$ 0.0071   \\
$\alpha = 100$, $\beta = 600$ & 0.6296 $\pm$ 0.0025  & 0.2329 $\pm$ 0.0923   \\
$\alpha = 100$, $\beta = 700$ & 0.6296 $\pm$ 0.0016  & 0.2976 $\pm$ 0.2768   \\
$\alpha = 100$, $\beta = 800$ & 0.6302 $\pm$ 0.0021  & 0.4738 $\pm$ 0.7114   \\
$\alpha = 100$, $\beta = 900$ & 0.6304 $\pm$ 0.0034  & 0.2116 $\pm$ 0.0315   \\
\bottomrule
\end{tabular}
\end{table}
\par After tuning the hyperparameters, the experimental results are displayed in \textcolor{blue}{\textbf{Table}} \ref{table:5}, where the LWR-PINN model refers to the model shown in \textcolor{blue}{\textbf{Figure}} \ref{fig:4}, which adopted a loss function as: $Loss_{\theta,\omega}^{LWR} = 100 \cdot MSE_{o} + 500 \cdot MSE_{a}$, for all other models, they treat the training process as a bi-objective optimization problem, the two objectives are $MSE_{o}$ and $ MSE_{a}$. According to \textcolor{blue}{\textbf{Table}} \ref{table:5}, the TMGD-LWR-PINN model and the models based on DCGDs exhibit performance similar to that of the baseline model, the LWR-PINN model, with the exception of the DCGD-AVER-LWR-PINN model, which shows a significantly larger speed $\mathbb{L}^{2}$ error compared to the baseline. \label{sec:3} \label{line:3} In traffic flow theory~\cite{ni2015traffic}, both traffic density \(\rho(x,t)\) and speed \(v(x,t)\) are functions of space \(x\) and time \(t\). In the context of macroscopic traffic flow models, field data typically refer to sensor measurements collected from detectors distributed along the highway. Mathematically, the sensor data consist of input-output pairs \(\{(x_i, t_i), (\rho_i, v_i)\}_{i=1}^N\), where each spatial-temporal coordinate \((x_i, t_i)\) corresponds to a unique measurement of density \(\rho_i\) and speed \(v_i\). In other words, the datasets \(\mathbf{x} = \{(x_i, t_i)\}_{i=1}^N\) and \(\mathbf{y} = \{(\rho_i, v_i)\}_{i=1}^N\) are aligned element-wise.
\par The PIML model first uses its data-driven component to approximate the density and speed functions. Based on these approximations, the physics residual is computed via automatic differentiation to enforce the macroscopic traffic flow model.
\par From the perspective of function approximation, high-resolution data are essential to accurately capture the underlying physical patterns. As illustrated in \textcolor{blue}{\textbf{Figure}}~\ref{fig:5}, low-resolution data can lead to catastrophic errors in function approximation. However, most real-world traffic sensor data are inherently low-resolution, since the spatial resolution \(\Delta x\) (i.e., the distance between two sensors) is often large, and the temporal resolution \(\Delta t\) (i.e., the time interval between consecutive measurements) is usually no less than five minutes.
\par As a result, PIML may fail to approximate the true density function \(\rho(x,t)\) due to insufficient data resolution~\cite{lei2025potential}. In this case, the physics-informed term \(\mathcal{L}_{\mathrm{physics}}\) becomes an invalid or misleading component in the overall loss function:
\begin{equation} \label{eq:28}
\mathcal{L} = \alpha \mathcal{L}_{\mathrm{data}} + \beta \mathcal{L}_{\mathrm{physics}}
\end{equation}
whose gradient information is uninformative and can even misguide the optimization process.
\par Therefore, when the PIML model suffers from such an endogenous flaw induced by low-resolution training data, even an improved training strategy cannot yield better performance.
 \begin{figure}
    \centering
    \includegraphics[width=0.8\linewidth]{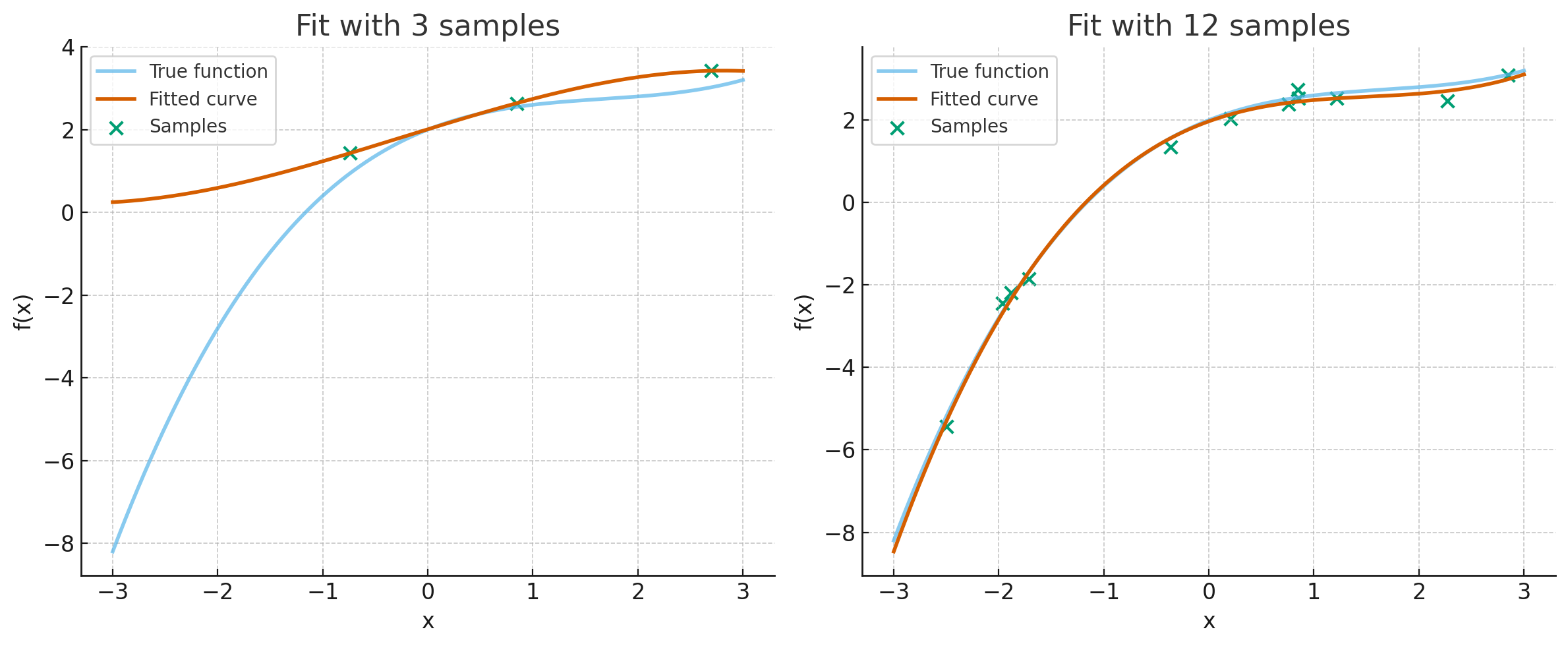}
    \caption{When low-resolution data meet the function approximation}
    \label{fig:5}
\end{figure}

% The highest accuracy difference in density prediction among all models occurs between the DCGD-Center-LWR-PINN model and the TMGD-LWR-PINN model, with a difference of 0.395\%. In terms of speed prediction accuracy, the most notable difference is between the DCGD-CENTER-LWR-PINN model and the TMGD-LWR-PINN model, showing a 20.93\% difference. Theoretically, according to  \textcolor{blue}{\textbf{Theorem}} \ref{theorem:5}, training based on multi-gradient descent algorithms could achieve the same performance as training the linear scalarization loss at least. In practice, not all multi-gradient descent algorithms can achieve a consistently guaranteed minimum performance. To ensure that we obtain either the same trained parameters or improved parameters, we would need to set a very large number of initial guess parameters for these multi-gradient descent algorithms. However, this becomes particularly challenging when considering the efficiency of solving. Additionally, high-resolution data is needed to capture the true pattern of the physics dynamics described by the LWR model and ARZ model. However, the general real field data might not be high-resolution data, which means that the $\Delta x$ and $\Delta t$ in those data should be as small as possible, $\Delta x$  represents the minimum distance between two traffic sensors, and $\Delta t$ represents the minim time interval between consecutive measurements. As a result, the physics part of the PIML model may always provide a gradient that could degrade the model's performance.
\begin{table}[htp!]
\centering
\caption{Comparison of MGDAs and linear scalarization for LWR-based PINN model} \label{table:4}
\begin{tabular}{cccc}
\toprule
Model & $Err(\hat{\rho},\rho)$ & $ Err(\hat{u},u)$ & Coefficient value \\
\midrule
LWR-PINN &  0.6291 $\pm$ 0.0023  & \textbf{0.1958 $\pm$ 0.0071}  & $\alpha = 100$, $\beta = 500$ \\
TMGD-LWR-PINN & \textbf{0.6270 $\pm$ 0.0024}  & 0.1984 $\pm$ 0.0009  & -\\
DCGD-CENTER-LWR-PINN & 0.6314 $\pm$ 0.0002  & 0.1998 $\pm$ 0.0001  & -\\
DCGD-PROJ-LWR-PINN & 0.6306 $\pm$ 0.0025  & 0.2020 $\pm$ 0.0066  & -\\
DCGD-AVER-LWR-PINN & 0.6315 $\pm$ 0.0016  & 0.2267 $\pm$ 0.0585  & -\\
\bottomrule
\end{tabular}
\end{table}

% \begin{table}[htp!]
% \centering
% \caption{Comparison of MGDAs and linear scalarization for ARZ-based PINN model} \label{table:2}
% \begin{tabular}{cccc}
% \toprule
% Model & $Err(\hat{\rho},\rho)$ & $ Err(\hat{u},u)$ & Coefficient value \\
% \midrule
% ARZ-PINN & 0.6284  & \textbf{0.1975}  & $\alpha = 100$, $\beta = 120$ \\
% MGDA-ARZ-PINN & 0.6312  & 0.1998  & -\\
% DCGD-CENTER-ARZ-PINN & 0.6216  & 0.2066  & -\\
% DCGD-PROJ-ARZ-PINN & 0.6323  & 0.2451  & -\\
% DCGD-AVER-ARZ-PINN & \textbf{0.6148}  & 0.3348  & -\\
% \bottomrule
% \end{tabular}
% \end{table}
\subsection{PIML model based on the microscopic traffic flow model}
\par The second example is a PIML model based on the microscopic traffic flow model. The physics model used in this PIML study is the IDM car-following model (\cite{treiber2000congested}). The IDM model is formulated as follows:
\begin{equation} \label{eq:29}
    a(t+\Delta t)   = a_{max}[ 1-(\frac{v(t)}{v_0}) ^\delta - (\frac{s^*(v(t), \Delta v(t))}{h(t)})^2] 
\end{equation}
\begin{equation} \label{eq:30}
    s^*(v(t),\Delta v(t)) =s_0 + v(t) T_0 + \frac{v(t) \cdot \Delta v(t)}{2\sqrt{a_{max}b}},
\end{equation}
In this context, the parameters are defined as follows:  \(v_0\) represents the desired velocity, \(T_0\) denotes the desired time headway, \(s_0\) indicates the minimum spacing in congested traffic, \(a_{\text{max}}\) is the maximum acceleration, \(b\) refers to the comfortable deceleration, and \(\delta\) is a constant typically set to 4 for the car-following model.

\par The baseline physics-informed car-following model is carried out based on \cite{mo2021physics}. The official reproduction of the baseline physics-informed car-following model and the corresponding dataset can be found at \url{https://github.com/CU-DitecT/TRC21-PINN-CFM}. The reproduction model differs slightly from the original paper \cite{mo2021physics}. In \cite{mo2021physics}, there are two frameworks for the physics-informed car-following model. The first framework, referred to as the prediction-only framework, is illustrated in \textbf{\textcolor{blue}{Figure}} \ref{fig:6}. In this framework, we first calibrate the car-following model, and subsequently, the model parameters are adjusted. The second framework is known as the joint framework, where we learn the parameters of the physics model and the weights of the neural network during the training process. The official reproduction only construes the first framework, so we will take the prediction-only framework as our baseline. 
\par As shown in \textbf{\textcolor{blue}{Figure}} \ref{fig:6}, a genetic algorithm (GA) will be used to calibrate the physics-based car-following model. The Physics-Uninformed Neural Network (PUNN) component is a multi-layer perceptron with an input dimension of three and an output dimension of one. It consists of three hidden layers, each containing 60 nodes. The loss function of the PINN-PICF model is still a classical form of linear scalarization, formulated as:
\begin{equation}  
\begin{split}
\label{eq:31}
    Loss_{\theta} &= \alpha MSE_{O} + (1-\alpha)MSE_{C} \\
                &=\frac{\alpha}{N_{o}}\sum_{i = 1}^{N_{o}}|f_{\theta}(\hat{\mathbf{s}}^{(j)}|\theta) - \hat{a}^{(i)}|^{2} + \frac{(1 - \alpha)}{N_{c}}\sum_{j = 1}^{N_{c}}|f_{\theta}(\mathbf{s}^{(j)}|\theta) - f_{\hat{\lambda}}(\mathbf{s}^{(j)}|\lambda)|^{2}
\end{split}
\end{equation}
where $\alpha$: the weight of the loss function that balances the contributions made by the data discrepancy and the physics discrepancy; $\hat{\lambda}$: parameters of the physics that are calibrated beforehand; $\theta$: PUNN's parameters;
$f_{\theta}(\cdot)$: PUNN parameterized by $\theta$; $f_{\hat{\lambda}}(\cdot)$: physics parameterized by $\hat{\lambda}$, which is calibrated prior to training; $N_O$: the number of observed data; $N_C$: the number of collocation states; $\hat{\mathbf{s}}^{(i)}$: the $i$th observed feature; $\mathbf{s}^{(j)}$: the $j$th collocation feature; $\hat{a}^{(i)}$: the $i$th observed acceleration; $a^{(i)}$: the $i$th acceleration predicted by PUNN; $a^{(j)}$: the $j$th acceleration predicted by PUNN; $a_{phy}^{(j)}$: the $j$th acceleration predicted by the physics. The collocation states are functionally equivalent to auxiliary points, which is used to construct the physics residuals. Readers could refer to \textbf{\textcolor{blue}{Section}} \ref{section:4.1} or \textbf{\textcolor{blue}{Section 3.2.2}} in \cite{mo2021physics}.
\par We utilized the Root Mean Square Error (RMSE) to measure the errors in both position and velocity as follows:
\begin{figure}
    \centering
    \includegraphics[width=0.9\linewidth]{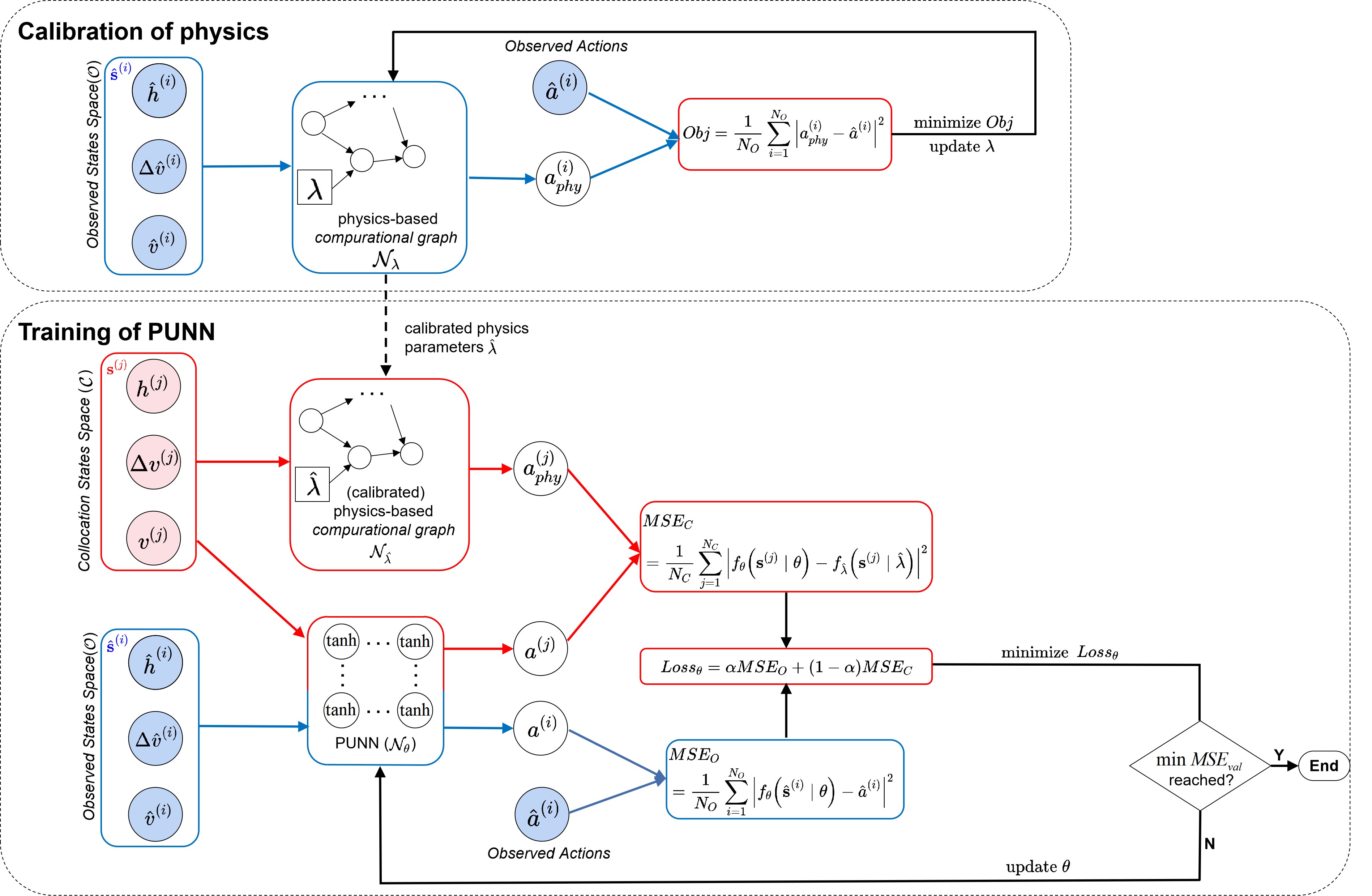}
    \caption{PINN-PICF architecture (\cite{mo2021physics})}
    \label{fig:6}
\end{figure}
\begin{equation} \label{eq:32}
     RMSE_x = \sqrt{
            \frac{1}{N_T T_i}\sum_{i=1}^{N_T} \sum_{t=0}^{T_i}
            \left| 
                x(t)-\hat{x}(t)
            \right|^2 
     }
 \end{equation}   
 \begin{equation} \label{eq:33}
     RMSE_v = \sqrt{
            \frac{1}{N_T T_i}\sum_{i=1}^{N_T} \sum_{t=0}^{T_i}
            \left| 
                v(t)-\hat{v}(t)
            \right|^2 
     }
 \end{equation}
 where $x$ and $v$ denote the positions and velocities of the follower, $T_i$ is the time horizon of the $i$th trajectory, and $N_T$ is the number of trajectories.  Since only the IDM-based physics-informed car-following model is present in the official reproduction code, we denote the IDM-based physics-informed car-following model as the original physics-informed car-following model (OPICF) in our paper. \label{sec:2} \label{line:2} It is noted that we adopted the same early stop mechanism as \cite{mo2021physics} to prevent overfitting, and have the same maximum training iterations.  \label{sec:9} \label{line:9} As shown in \textbf{\textcolor{blue}{Table}}~\ref{table:6}, we conduct a series of comparative experiments involving the benchmark training setup, traditional multi-gradient descent~\cite{sener2018multi}, dual cone gradient descent~\cite{hwang2024dual}, dynamic weight averaging (DWA)~\cite{liu2019end}, and uncertainty weighting (UW)~\cite{kendall2018multi}. The key distinction among these experiments lies in the choice of loss formulation and training strategy. Specifically, the benchmark, DWA, and uncertainty weighting approaches fall under single-objective optimization frameworks, as they scalarize multiple loss terms into a single objective. In contrast, multi-gradient descent and dual cone gradient descent are genuine multi-objective optimization methods that aim to balance task-specific losses via principled directional updates grounded in multi-task optimization theory. \textbf{\textcolor{blue}{Table}}~\ref{table:5} summarizes the loss functions and update directions used in a subset of the experimental settings. In this table, \(\mathcal{L}_1\) denotes the data loss and \(\mathcal{L}_2\) denotes the physics loss, \(k\) is the iteration index, \(\sigma_t\) denotes the noise associated with the \(t\)th loss term, and \(\alpha^{*}\) is the a weight vector obtained from \textbf{\textcolor{blue}{Problem}}~\ref{eq:12}. While the table does not explicitly list the formulation for dual cone gradient descent (DCGD), it similarly treats \(\mathcal{L}_1\) and \(\mathcal{L}_2\) as two distinct objectives within a multi-objective framework. Its update direction is computed via a different strategy, as detailed in \textcolor{blue}{\textbf{Section}}~\ref{section:3.1}.

 \begin{table}[htp!]
\centering
\caption{Comparison of different training set up} \label{table:5}
\begin{tabular}{ccc}
\toprule
Model & Loss function & Update direction \\
\midrule
OPICF &  $\mathcal{L} = 0.9\cdot\mathcal{L}_{1} + 0.1 \cdot \mathcal{L}_{2}$  &  $\mathbf{d} = - 0.9\cdot\nabla\mathcal{L}_{1}-0.1\nabla\mathcal{L}_{2}$  \\
DWA-PICF & $\mathcal{L} = \sum_{t = 1}^{2}w_{t}^{(k)}\mathcal{L}_{t},w_{t}^{(k)} = \frac{2\exp{\bigg(\frac{\mathcal{L}_{t}^{(k-1)}}{\mathcal{L}_{t}^{(k-2)}}/2\bigg)}}{\sum_{j = 1}^{2}\exp{\bigg(\frac{\mathcal{L}_{j}^{(k-1)}}{\mathcal{L}_{j}^{(k-2)}}/2\bigg)}}$  &  $\mathbf{d} = - \sum_{t = 1}^{2}w_{t}^{k}\nabla\mathcal{L}_{t}$  \\
UW-PICF &  $\mathcal{L} = \sum_{t = 1}^{2}\bigg(\frac{1}{2\sigma_{t}^{2}}\mathcal{L}_{t}\bigg) + \log \sigma_{t}$  &  $\mathbf{d} = - \sum_{t = 1}^{2}\frac{1}{2\sigma_{t}^{2}}\nabla\mathcal{L}_{t}$  \\
TMGD-PICF & $\mathcal{L}_{1},\mathcal{L}_{2}$ &  $\mathbf{d} = - \sum_{t = 1}^{2}\alpha^{*}_{t}\nabla\mathcal{L}_{t}$ \\
\bottomrule
\end{tabular}
\end{table}
\label{sec:5} \label{line:5}

All training methods based on multi-gradient descent outperform the baseline model as well as the DWA and UW strategies. Among them, the DCGD variants consistently yield the best results. Specifically, models trained using DCGD-CENTER, DCGD-PROJ, and DCGD-AVER achieve significantly lower RMSEs in both position and speed compared to scalarization-based methods and earlier MGDA baselines.

As shown in \textbf{\textcolor{blue}{Table}}~\ref{table:6}, the DCGD-CENTER-PICF model achieves the lowest position error, with an RMSE of 5.7275~$\pm$~0.2051, while DCGD-PROJ-PICF achieves the best velocity accuracy, with an RMSE of 0.6478~$\pm$~0.0997. The performance differences among the three DCGD variants are relatively minor and not statistically significant based on standard deviation overlaps. However, the improvement over scalarization-based models such as OPICF (with RMSE of 8.0587~$\pm$~4.0024 in position and 0.8598~$\pm$~0.1357 in speed) is substantial and consistent, indicating the effectiveness of direction correction in multi-objective optimization.

\label{sec:6} \label{line:6}
We also provide visual comparisons between predicted and actual trajectories for each training method in \textbf{\textcolor{blue}{Figures}}~\ref{fig:7} to~\ref{fig:11}. Each method was evaluated on 10 independent test trajectories. In these figures, the red dashed line represents the ground-truth trajectory, while the blue line denotes the predicted trajectory generated by the physics-informed car-following model trained with each method. As illustrated, DCGD-based models produce predictions that more closely track the true trajectory, both in position and dynamic pattern, further supporting the quantitative findings.

\begin{table}[htp!]
\centering
\caption{Performance comparison between linear scalarization-based and MGDA-based training for physics-informed car-following models} \label{table:6}
\begin{tabular}{cccc}
\toprule
Model & $RMSE(\hat{x},x)$ & $ RMSE(\hat{u},u)$ & Coefficient value \\
\midrule
OPICF & 8.0587 $\pm$ 4.0024  & 0.8598 $\pm$ 0.1357  & $\alpha = 0.9$, $\beta = 0.1$ \\
UW-PICF & 7.8995 $\pm$ 4.0174  & 0.7908 $\pm$ 0.1717  & -\\
DWA-PICF & 7.7455 $\pm$ 2.8498  & 0.7839 $\pm$ 0.1053  & -\\
TMGD-PICF & 7.6101 $\pm$ 2.5439  & 0.7902 $\pm$ 0.1173  & -\\
DCGD-CENTER-PICF & \textbf{5.7275 $\pm$ 0.2051}  & 0.6490 $\pm$ 0.0092  & -\\
DCGD-PROJ-PICF & 5.7853 $\pm$ 0.2009   & \textbf{0.6478 $\pm$ 0.0997}  & -\\
DCGD-AVER-PICF & 5.8061 $\pm$ 0.1885  &  0.6479 $\pm$ 0.0098  & -\\
\bottomrule
\end{tabular}
\end{table}
% \begin{table}[htp!]
% \centering
% \caption{Relative improvement of training based on MGDAs compared to the baseline} \label{table:6}
% \begin{tabular}{ccc}
% \toprule
% Model & Relative improvement of $RMSE(\hat{x},x)$ & Relative improvement of $ RMSE(\hat{u},u)$ \\
% \midrule
% TMGD-PICF & 5.5667\%  & 8.0949\%  \\
% DCGD-CENTER-PICF & 28.9277\%  & 24.5173\%  \\
% DCGD-PROJ-PICF & 28.2105\%  & 24.6569\%  \\
% DCGD-AVER-PICF & 27.9524\%  & 24.6453\%  \\
% \bottomrule
% \end{tabular}
% \end{table}
\begin{figure}[htbp]
  \centering

  \begin{subfigure}[b]{0.48\textwidth}
    \centering
    \includegraphics[width=\textwidth]{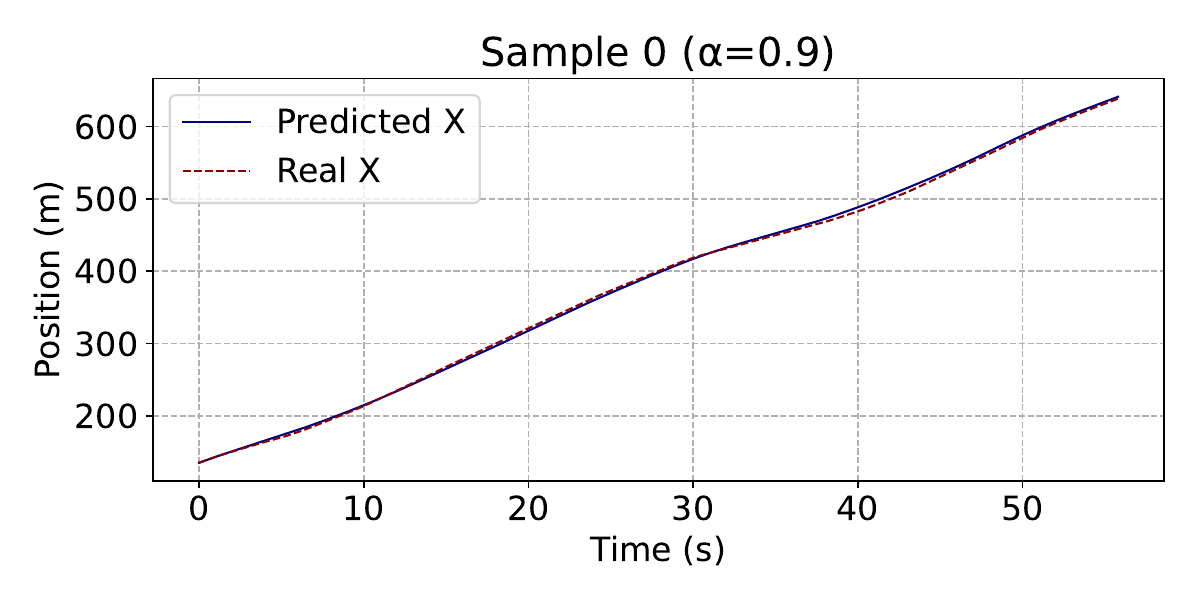}
  \end{subfigure}
  \begin{subfigure}[b]{0.48\textwidth}
    \centering
    \includegraphics[width=\textwidth]{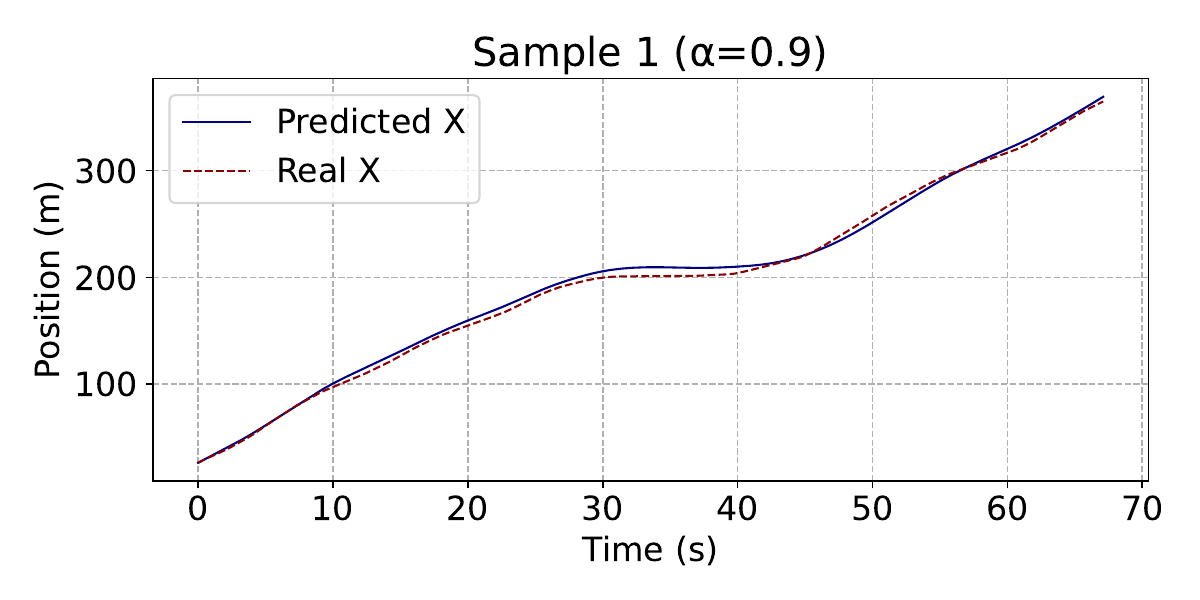}
  \end{subfigure}
  
  \begin{subfigure}[b]{0.48\textwidth}
    \centering
    \includegraphics[width=\textwidth]{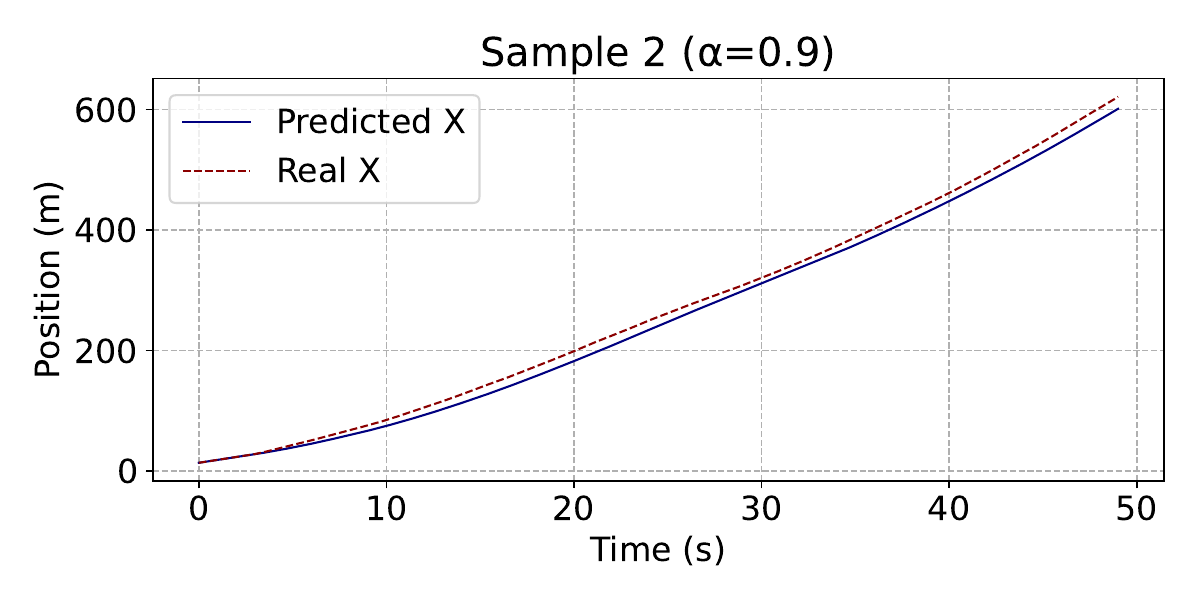}
  \end{subfigure}
  \begin{subfigure}[b]{0.48\textwidth}
    \centering
    \includegraphics[width=\textwidth]{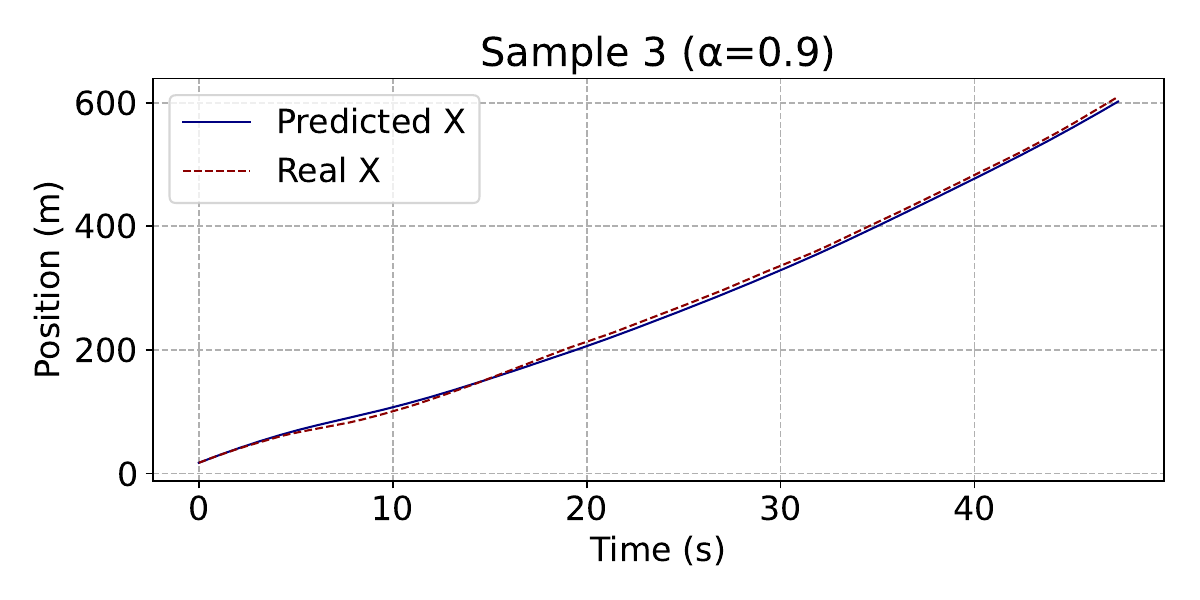}
  \end{subfigure}
  
  \begin{subfigure}[b]{0.48\textwidth}
    \centering
    \includegraphics[width=\textwidth]{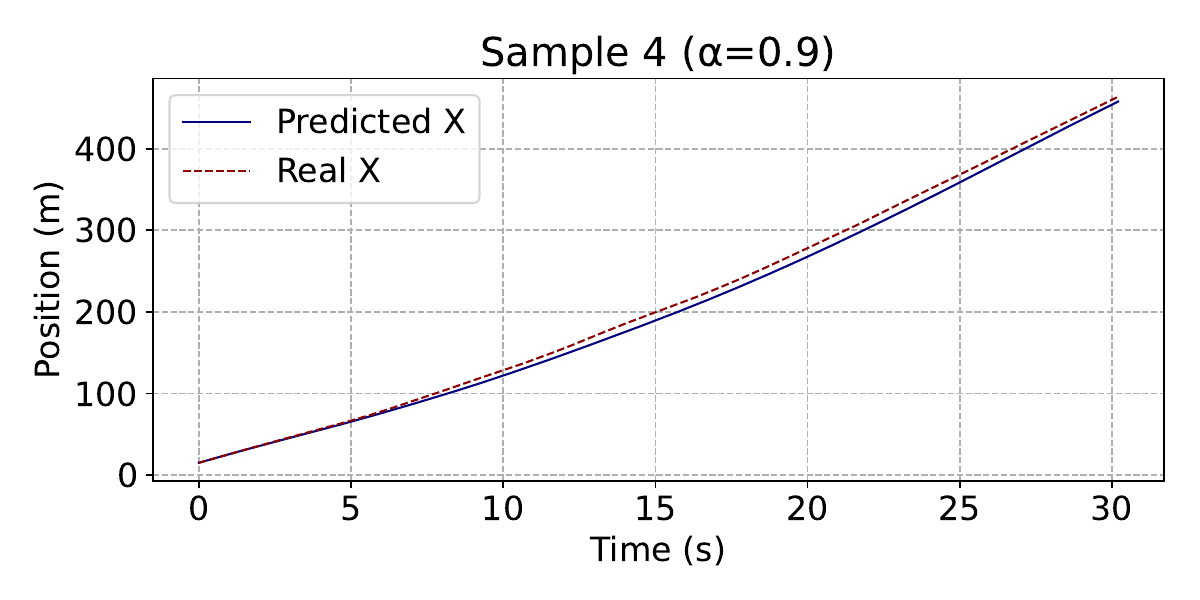}
  \end{subfigure}
  \begin{subfigure}[b]{0.48\textwidth}
    \centering
    \includegraphics[width=\textwidth]{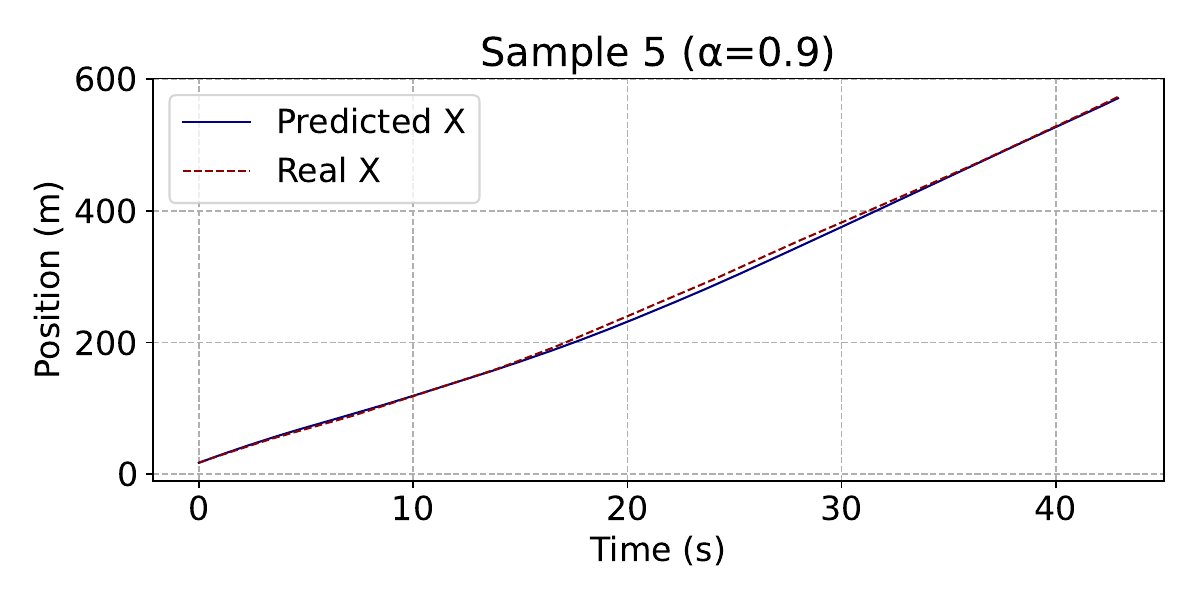}
  \end{subfigure}

  \begin{subfigure}[b]{0.48\textwidth}
    \centering
    \includegraphics[width=\textwidth]{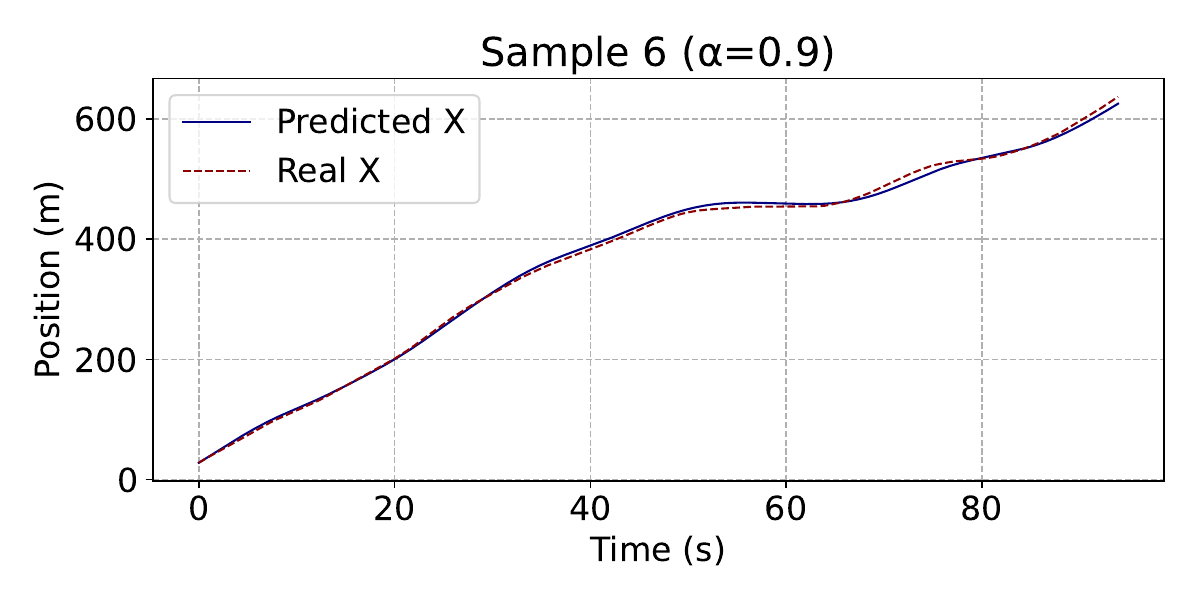}
  \end{subfigure}
  \begin{subfigure}[b]{0.48\textwidth}
    \centering
    \includegraphics[width=\textwidth]{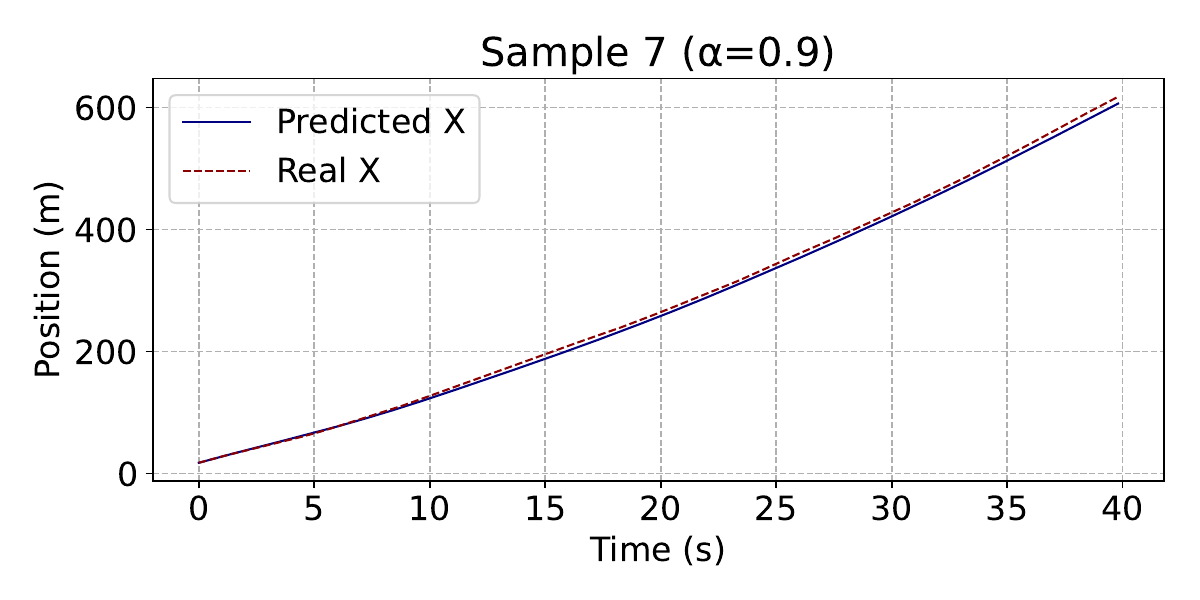}
  \end{subfigure}

  \begin{subfigure}[b]{0.48\textwidth}
    \centering
    \includegraphics[width=\textwidth]{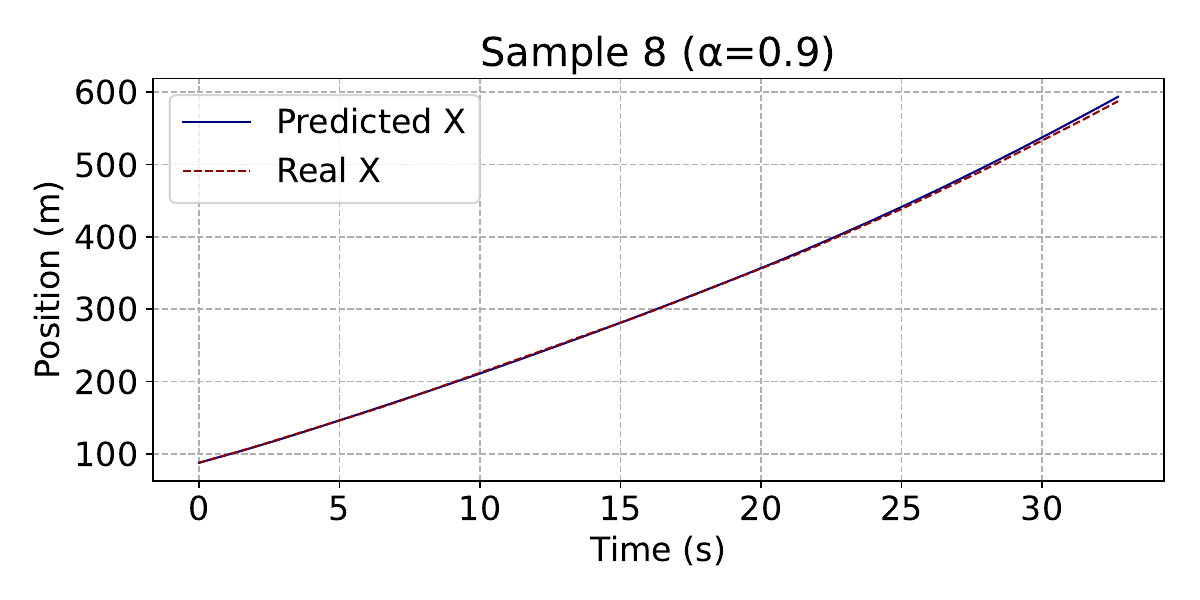}
  \end{subfigure}
  \begin{subfigure}[b]{0.48\textwidth}
    \centering
    \includegraphics[width=\textwidth]{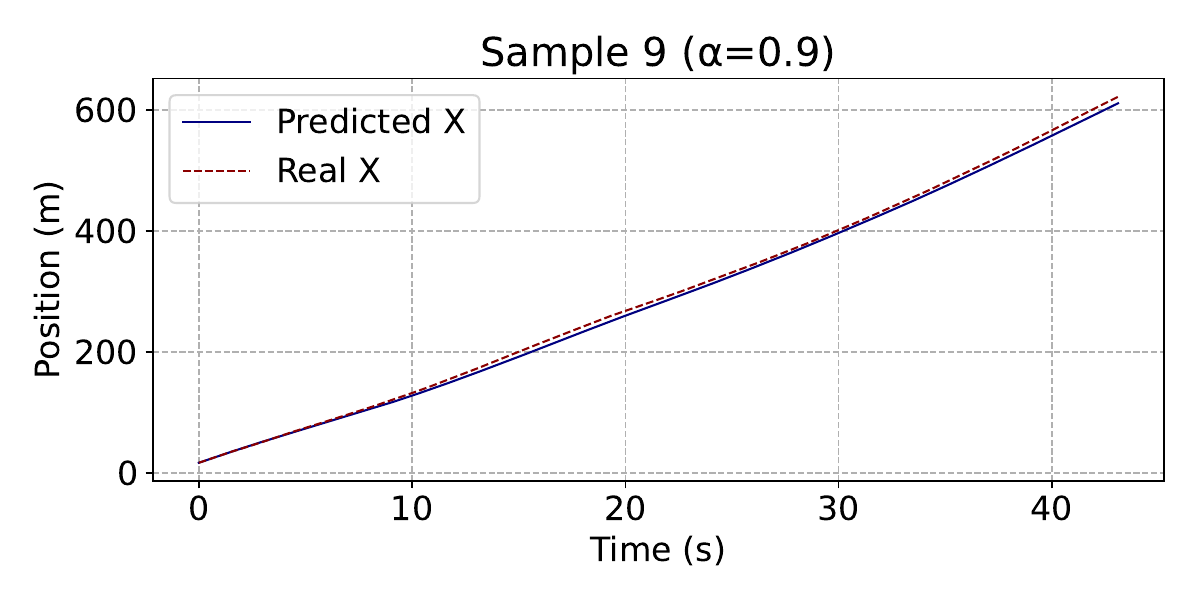}
  \end{subfigure}
  \caption{Comparison of predicted and real positions for ten Sample trajectories based on OPICF}
  \label{fig:7}
\end{figure}

\begin{figure}[htbp]
  \centering

  \begin{subfigure}[b]{0.48\textwidth}
    \centering
    \includegraphics[width=\textwidth]{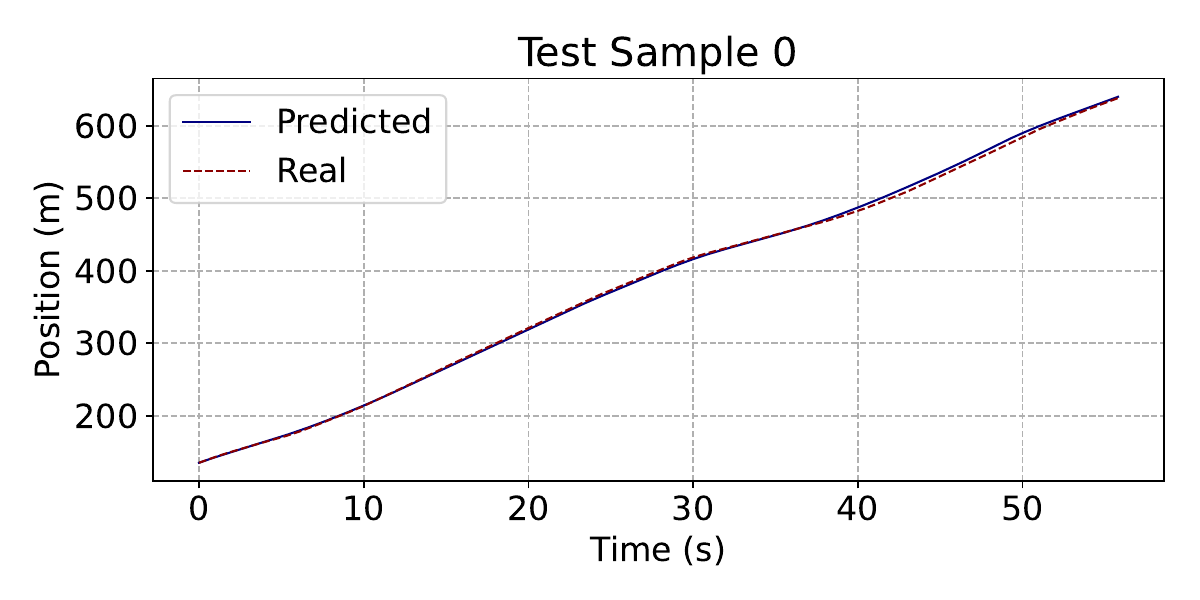}
  \end{subfigure}
  \begin{subfigure}[b]{0.48\textwidth}
    \centering
    \includegraphics[width=\textwidth]{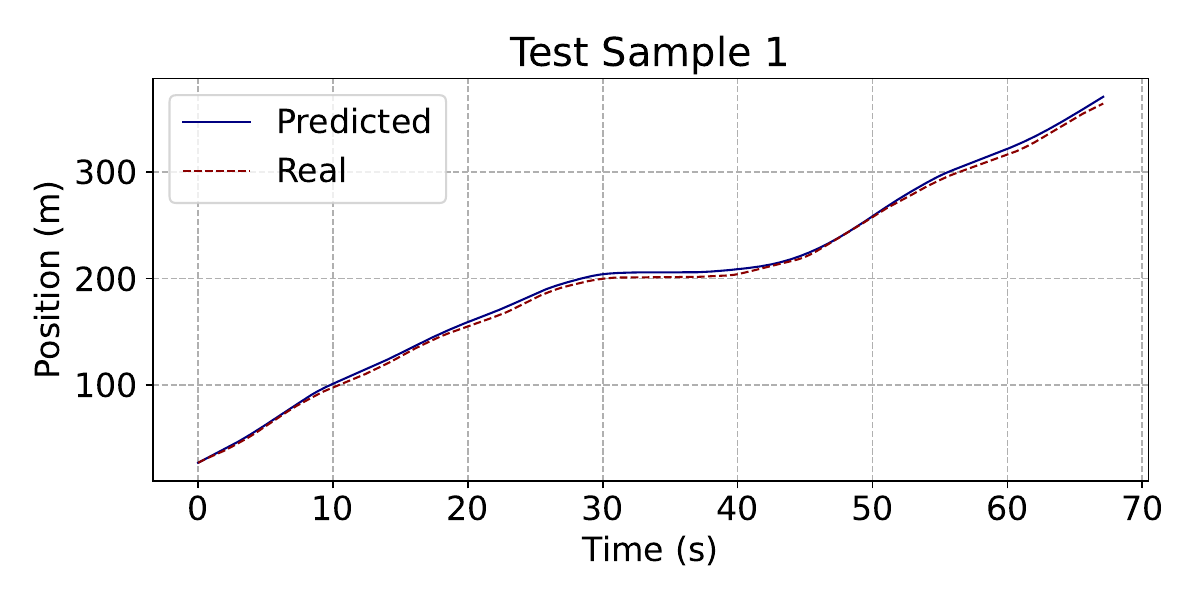}
  \end{subfigure}
  
  \begin{subfigure}[b]{0.48\textwidth}
    \centering
    \includegraphics[width=\textwidth]{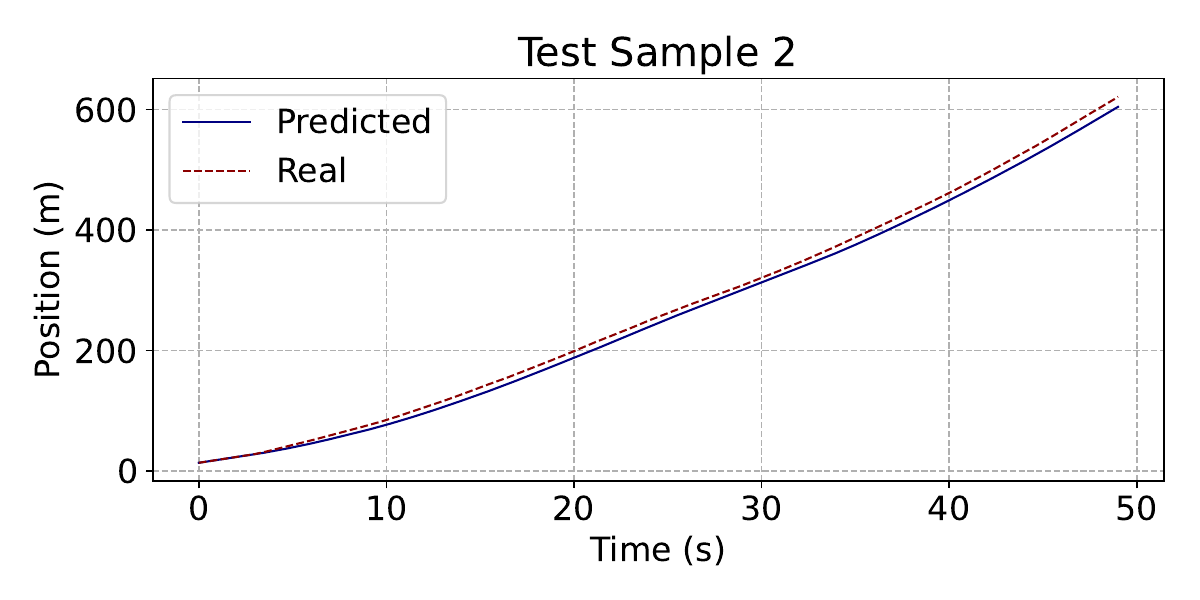}
  \end{subfigure}
  \begin{subfigure}[b]{0.48\textwidth}
    \centering
    \includegraphics[width=\textwidth]{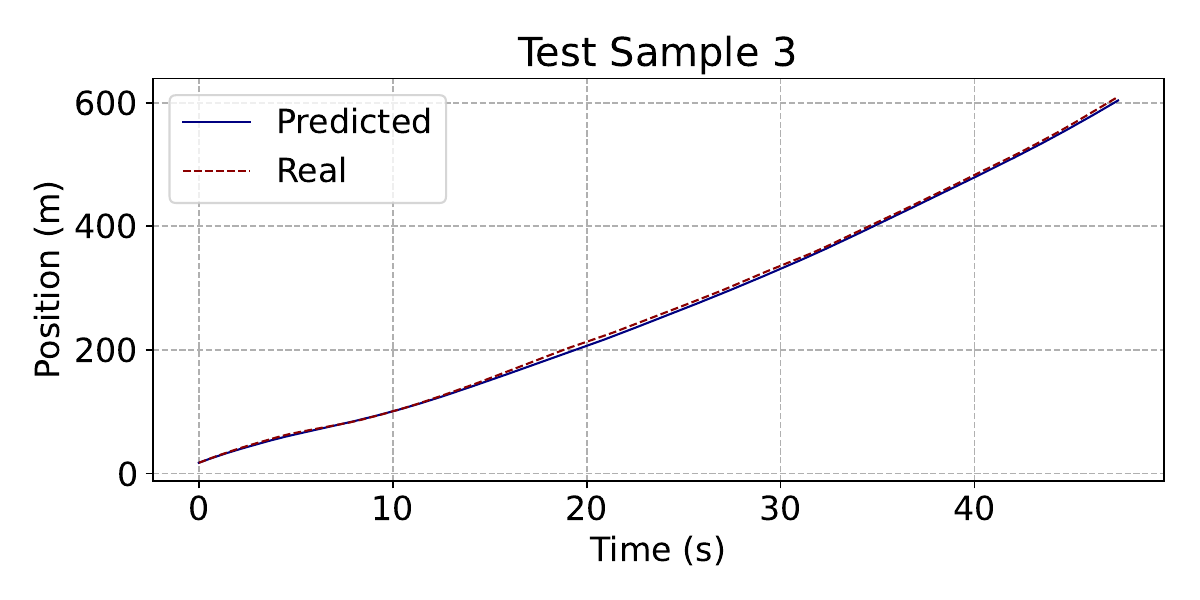}
  \end{subfigure}
  
  \begin{subfigure}[b]{0.48\textwidth}
    \centering
    \includegraphics[width=\textwidth]{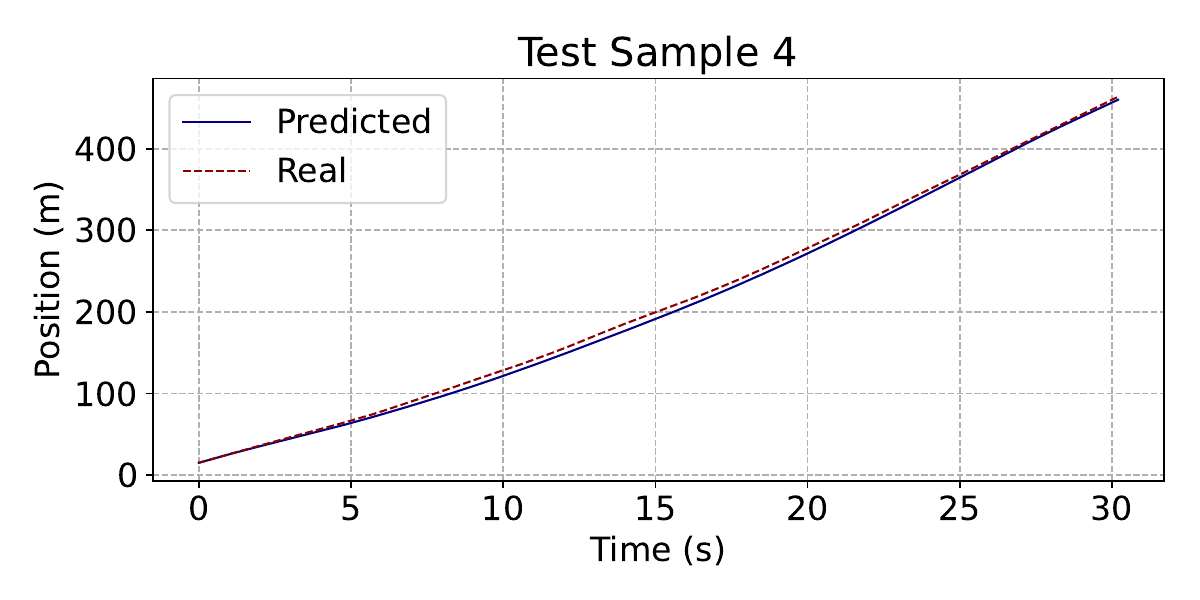}
  \end{subfigure}
  \begin{subfigure}[b]{0.48\textwidth}
    \centering
    \includegraphics[width=\textwidth]{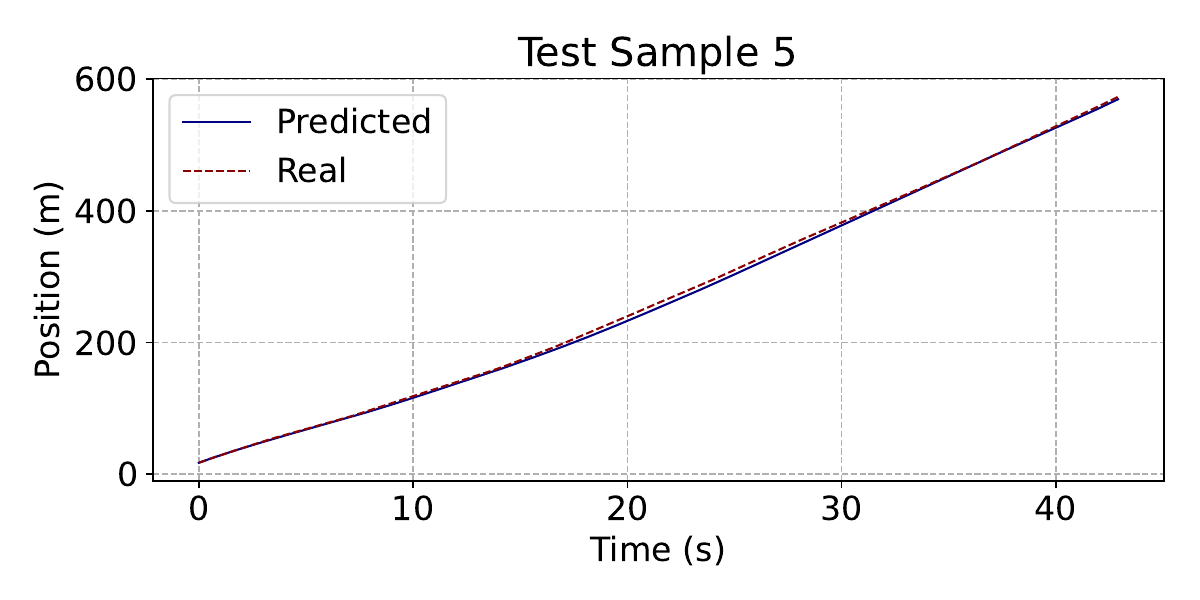}
  \end{subfigure}

  \begin{subfigure}[b]{0.48\textwidth}
    \centering
    \includegraphics[width=\textwidth]{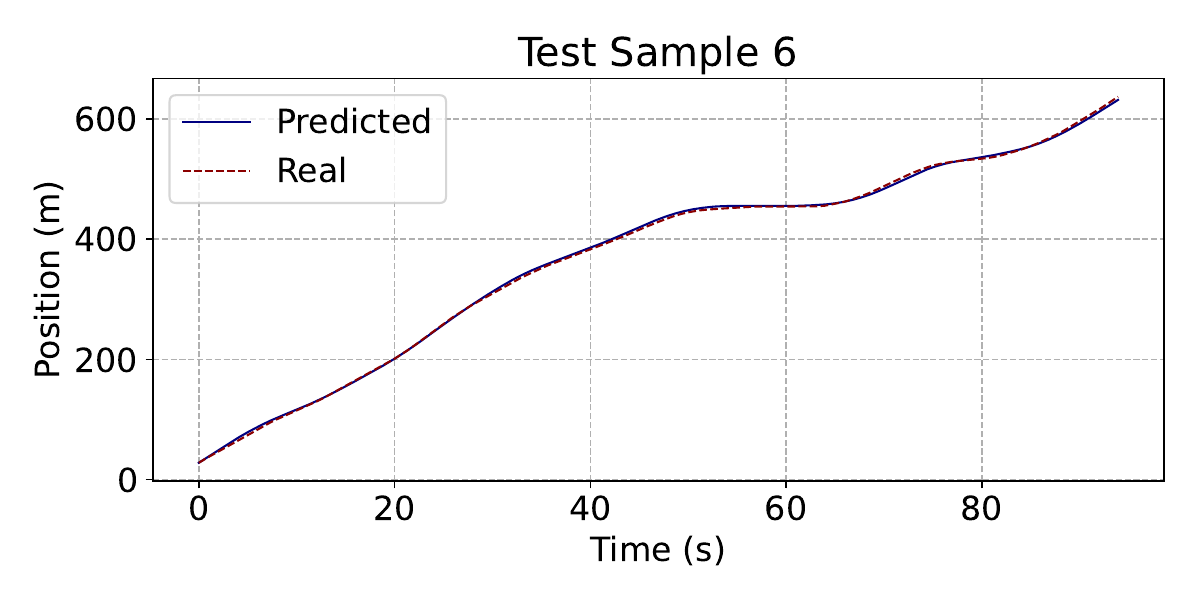}
  \end{subfigure}
  \begin{subfigure}[b]{0.48\textwidth}
    \centering
    \includegraphics[width=\textwidth]{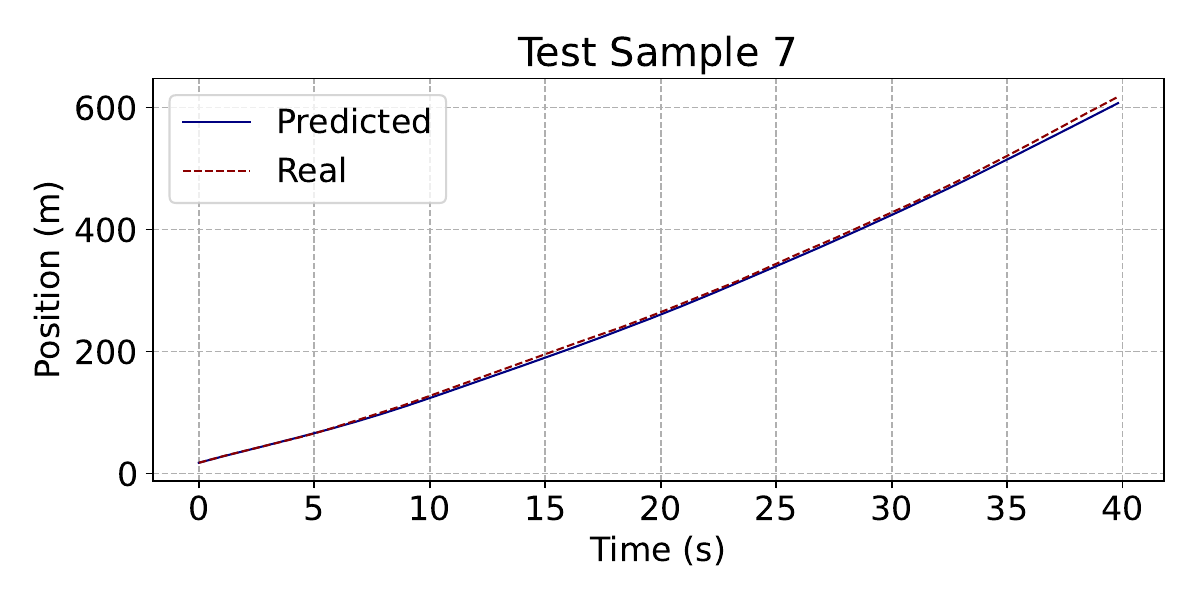}
  \end{subfigure}

  \begin{subfigure}[b]{0.48\textwidth}
    \centering
    \includegraphics[width=\textwidth]{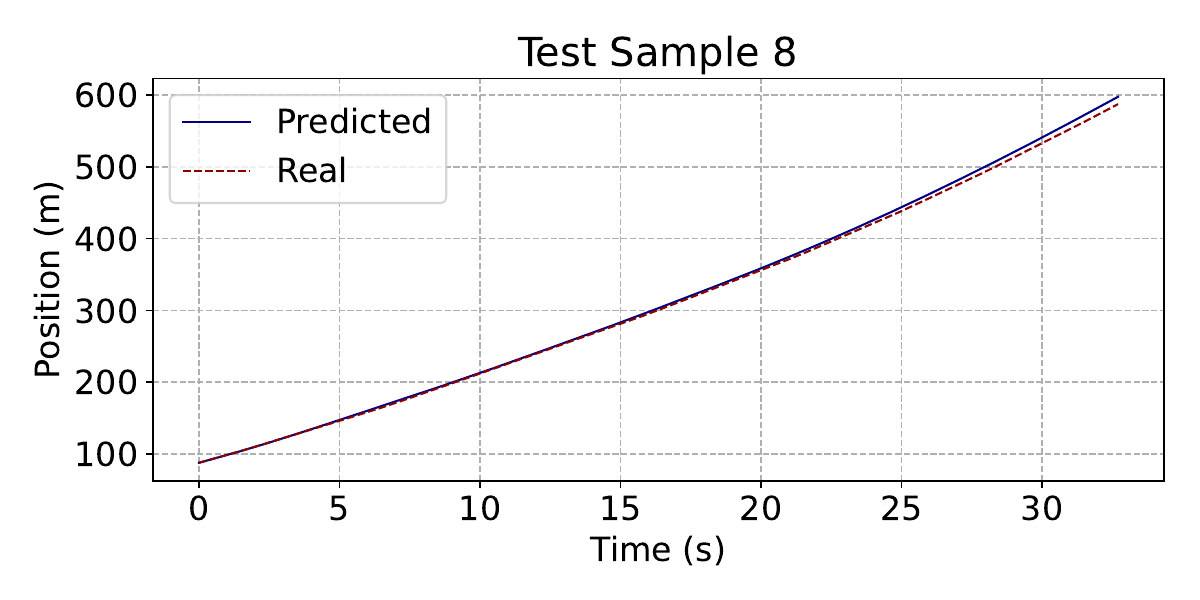}
  \end{subfigure}
  \begin{subfigure}[b]{0.48\textwidth}
    \centering
    \includegraphics[width=\textwidth]{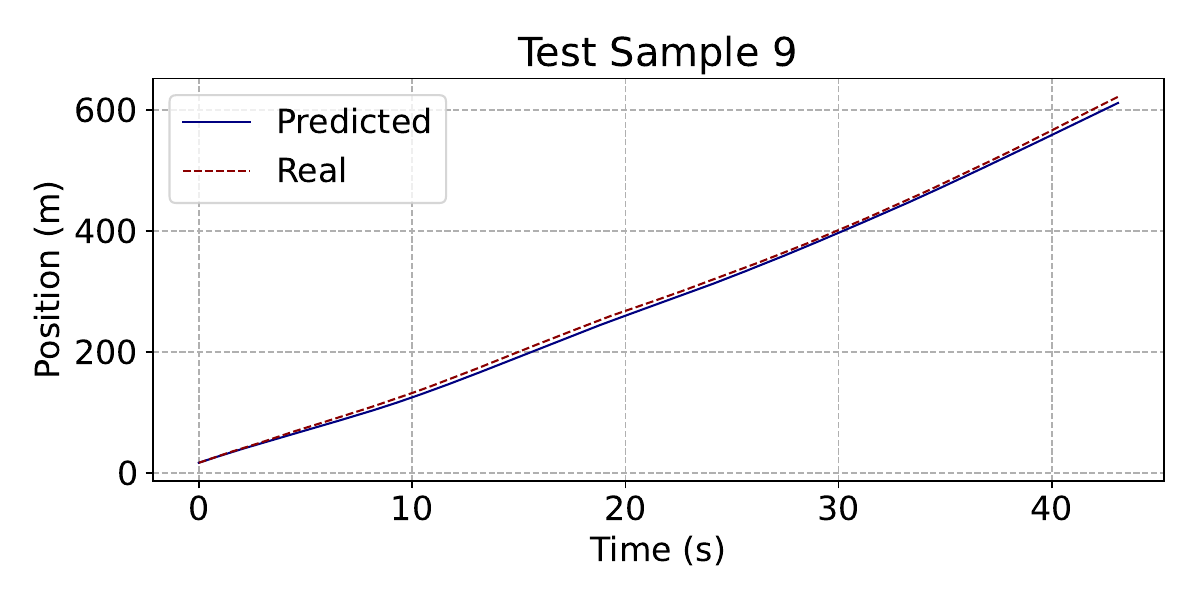}
  \end{subfigure}
  \caption{Comparison of predicted and real positions for ten Sample trajectories based on UW-PICF}
  \label{fig:8}
\end{figure}

\begin{figure}[htbp]
  \centering

  \begin{subfigure}[b]{0.48\textwidth}
    \centering
    \includegraphics[width=\textwidth]{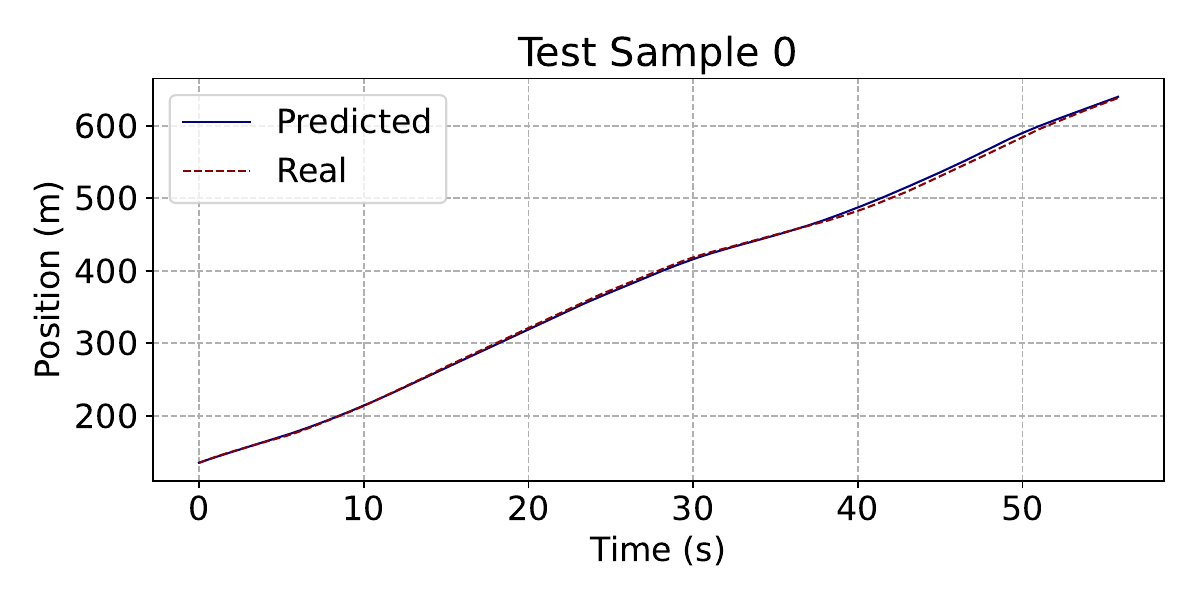}
  \end{subfigure}
  \begin{subfigure}[b]{0.48\textwidth}
    \centering
    \includegraphics[width=\textwidth]{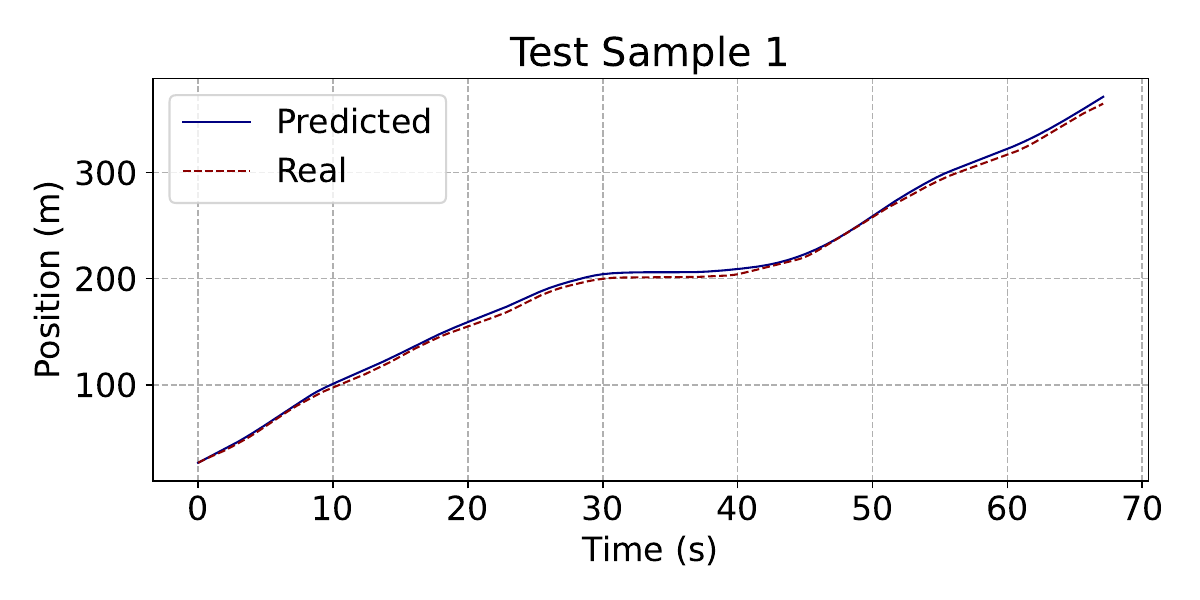}
  \end{subfigure}
  
  \begin{subfigure}[b]{0.48\textwidth}
    \centering
    \includegraphics[width=\textwidth]{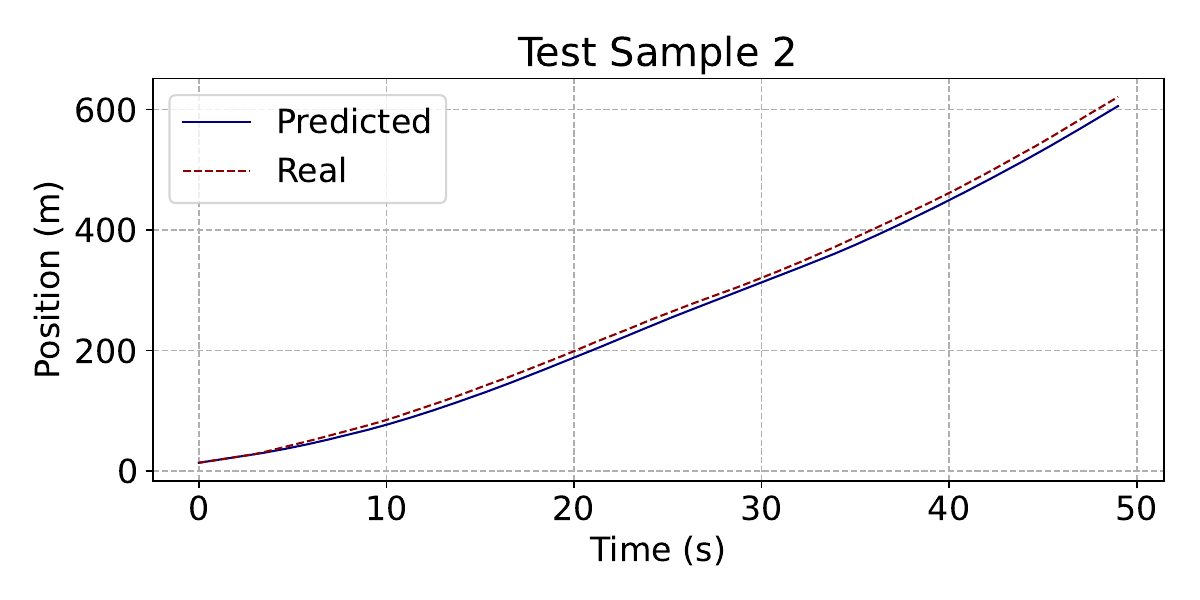}
  \end{subfigure}
  \begin{subfigure}[b]{0.48\textwidth}
    \centering
    \includegraphics[width=\textwidth]{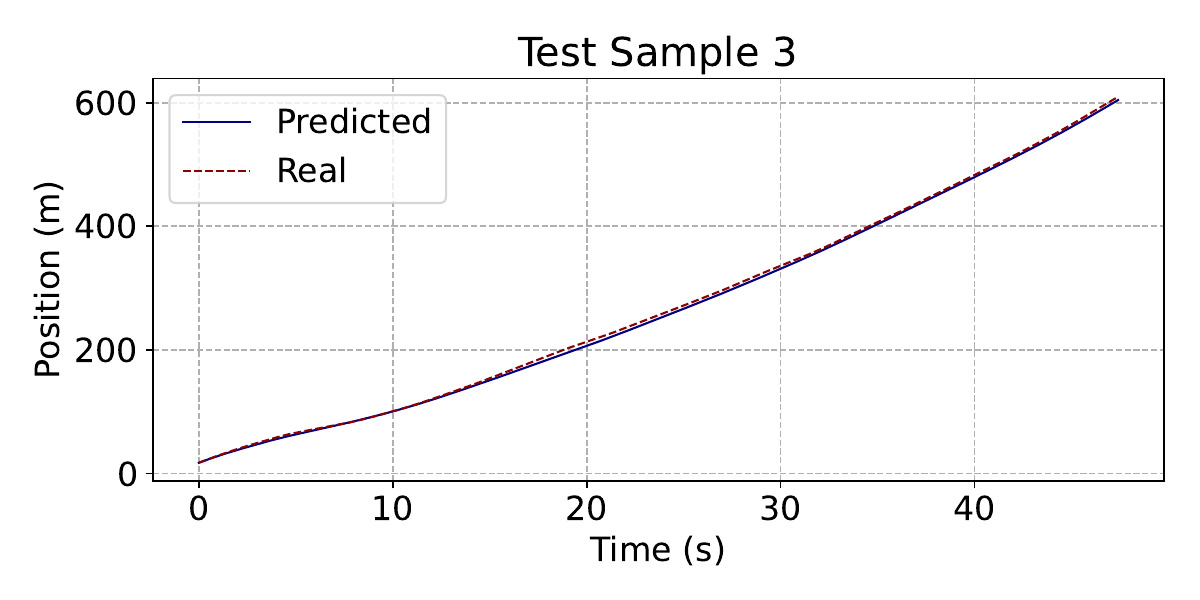}
  \end{subfigure}
  
  \begin{subfigure}[b]{0.48\textwidth}
    \centering
    \includegraphics[width=\textwidth]{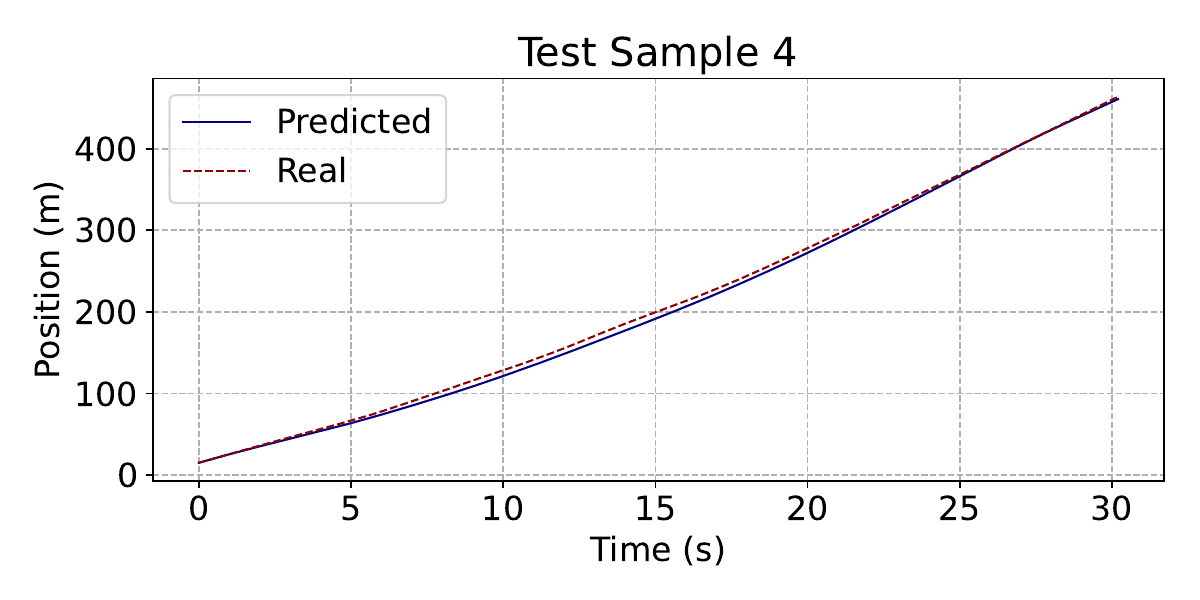}
  \end{subfigure}
  \begin{subfigure}[b]{0.48\textwidth}
    \centering
    \includegraphics[width=\textwidth]{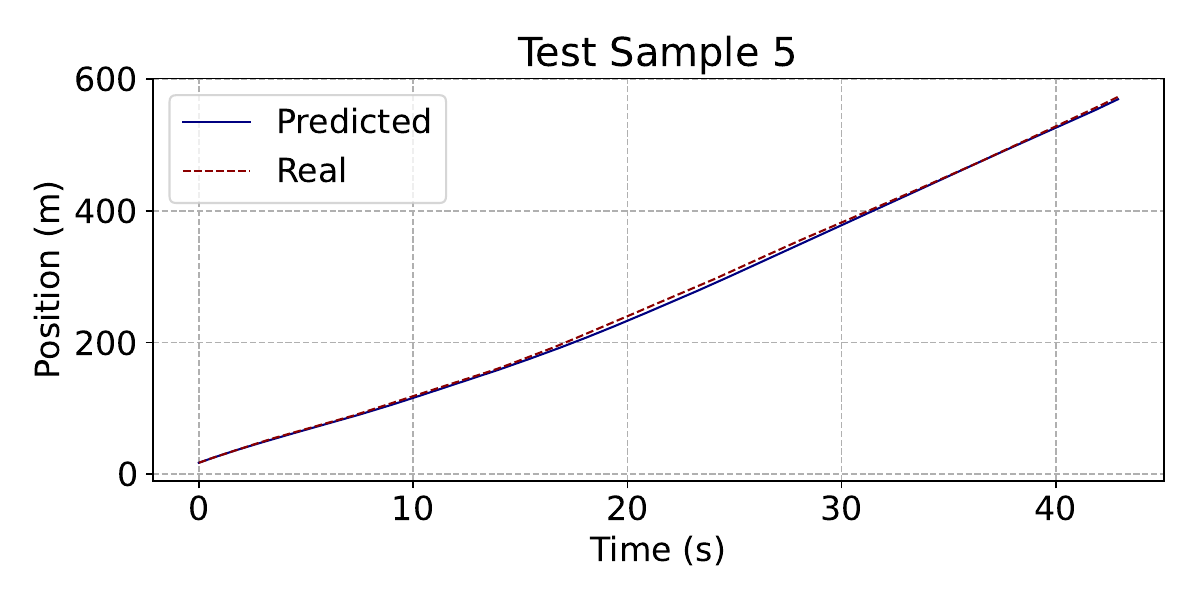}
  \end{subfigure}

  \begin{subfigure}[b]{0.48\textwidth}
    \centering
    \includegraphics[width=\textwidth]{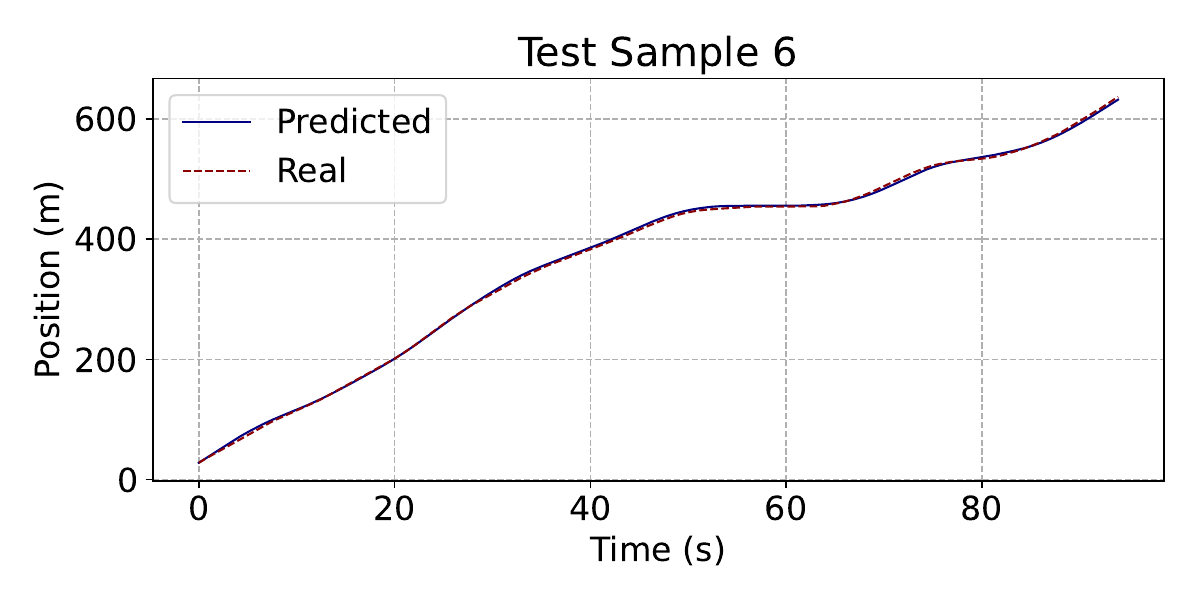}
  \end{subfigure}
  \begin{subfigure}[b]{0.48\textwidth}
    \centering
    \includegraphics[width=\textwidth]{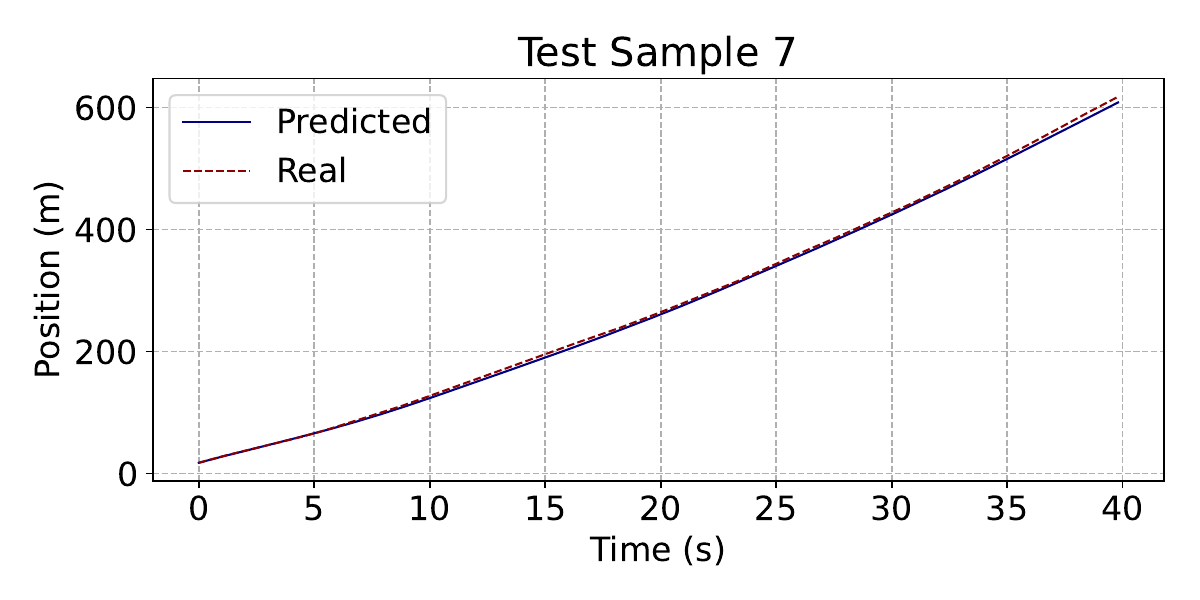}
  \end{subfigure}

  \begin{subfigure}[b]{0.48\textwidth}
    \centering
    \includegraphics[width=\textwidth]{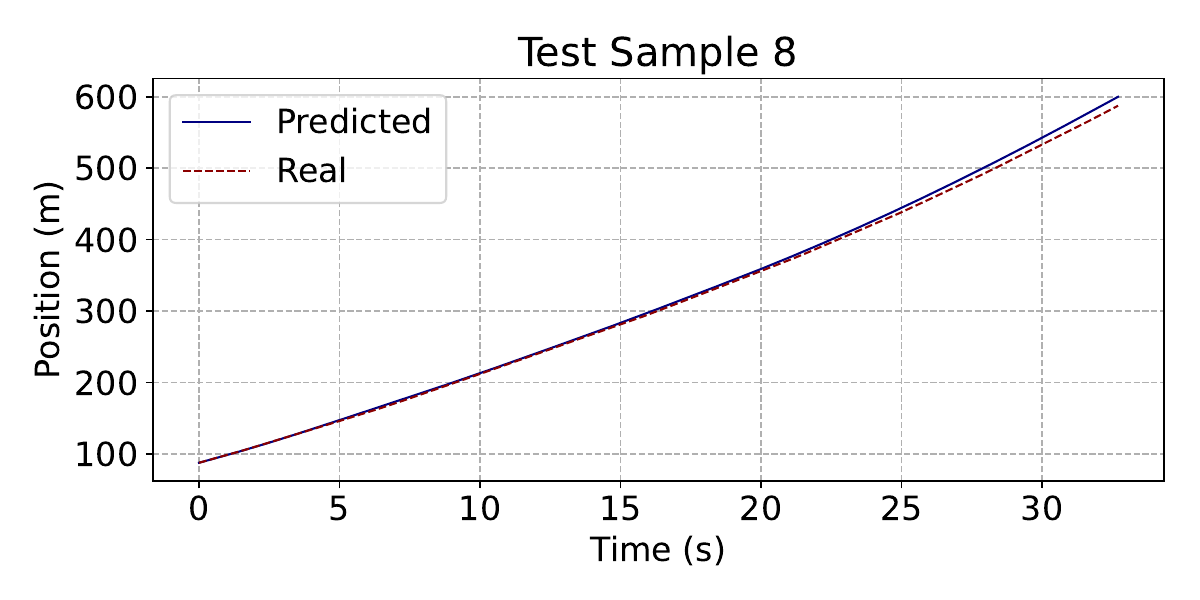}
  \end{subfigure}
  \begin{subfigure}[b]{0.48\textwidth}
    \centering
    \includegraphics[width=\textwidth]{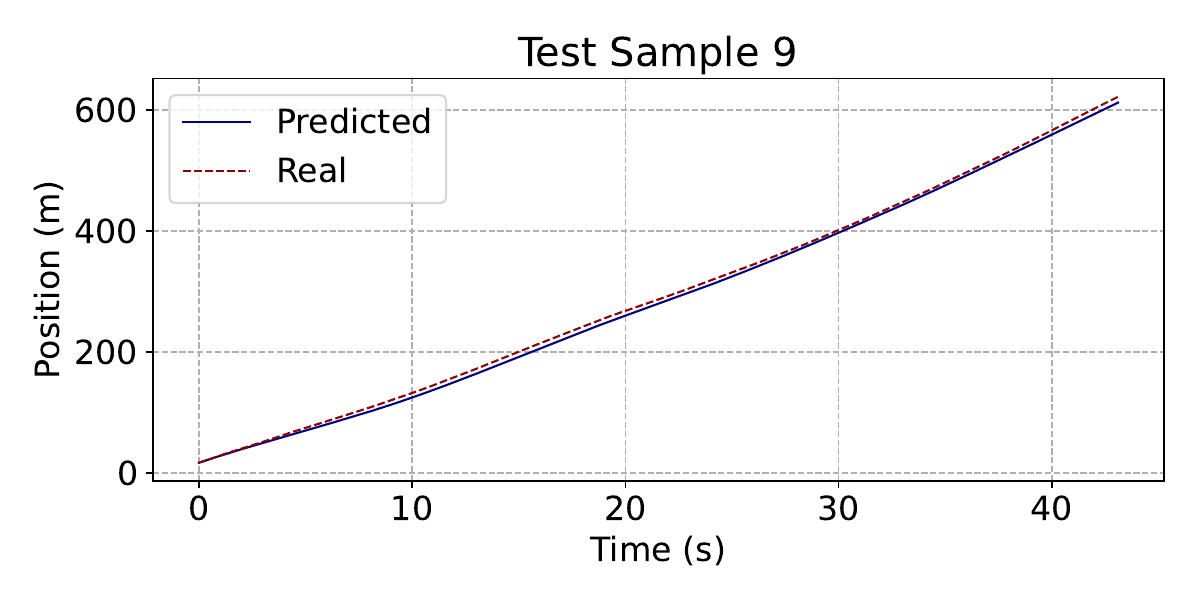}
  \end{subfigure}
  \caption{Comparison of predicted and real positions for ten Sample trajectories based on DWA-PICF}
  \label{fig:9}
\end{figure}

\begin{figure}[htbp]
  \centering

  \begin{subfigure}[b]{0.48\textwidth}
    \centering
    \includegraphics[width=\textwidth]{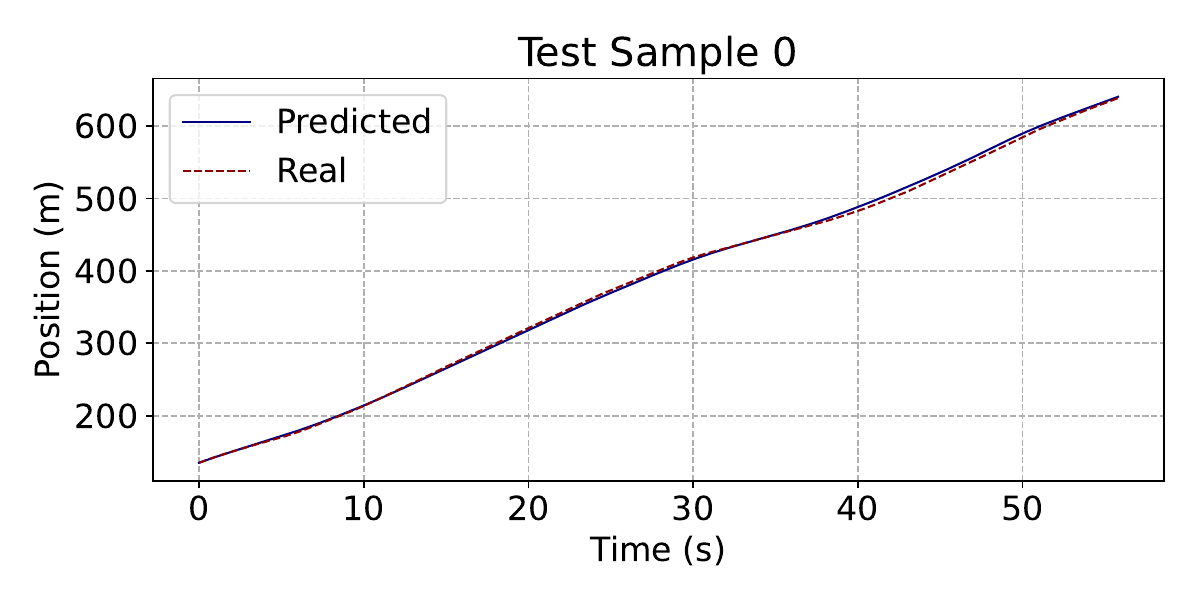}
  \end{subfigure}
  \begin{subfigure}[b]{0.48\textwidth}
    \centering
    \includegraphics[width=\textwidth]{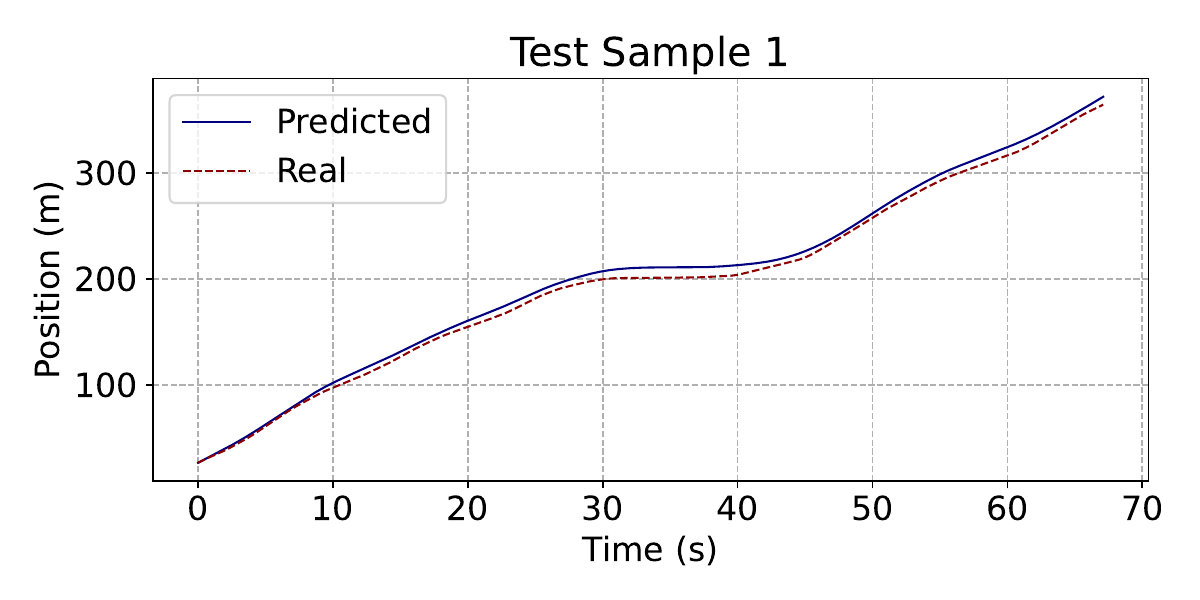}
  \end{subfigure}
  
  \begin{subfigure}[b]{0.48\textwidth}
    \centering
    \includegraphics[width=\textwidth]{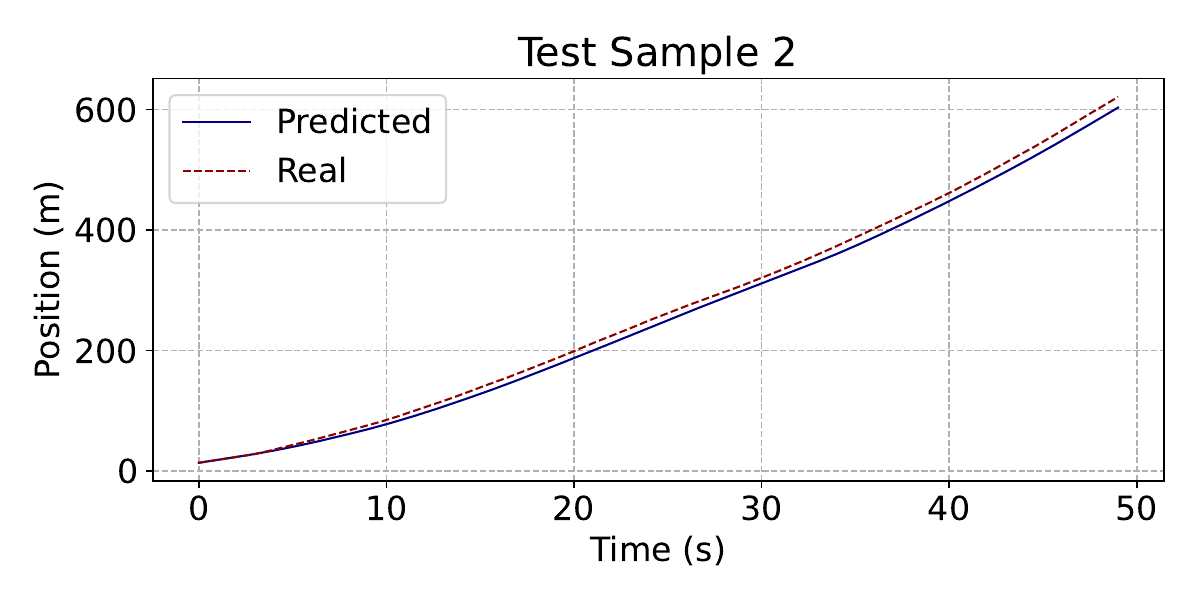}
  \end{subfigure}
  \begin{subfigure}[b]{0.48\textwidth}
    \centering
    \includegraphics[width=\textwidth]{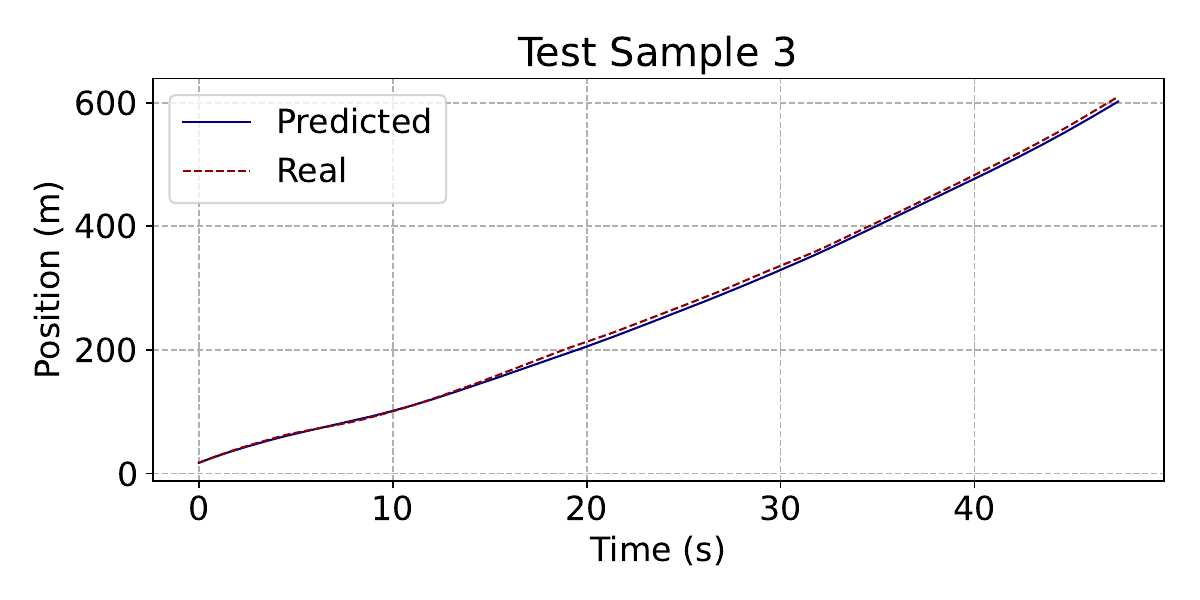}
  \end{subfigure}
  
  \begin{subfigure}[b]{0.48\textwidth}
    \centering
    \includegraphics[width=\textwidth]{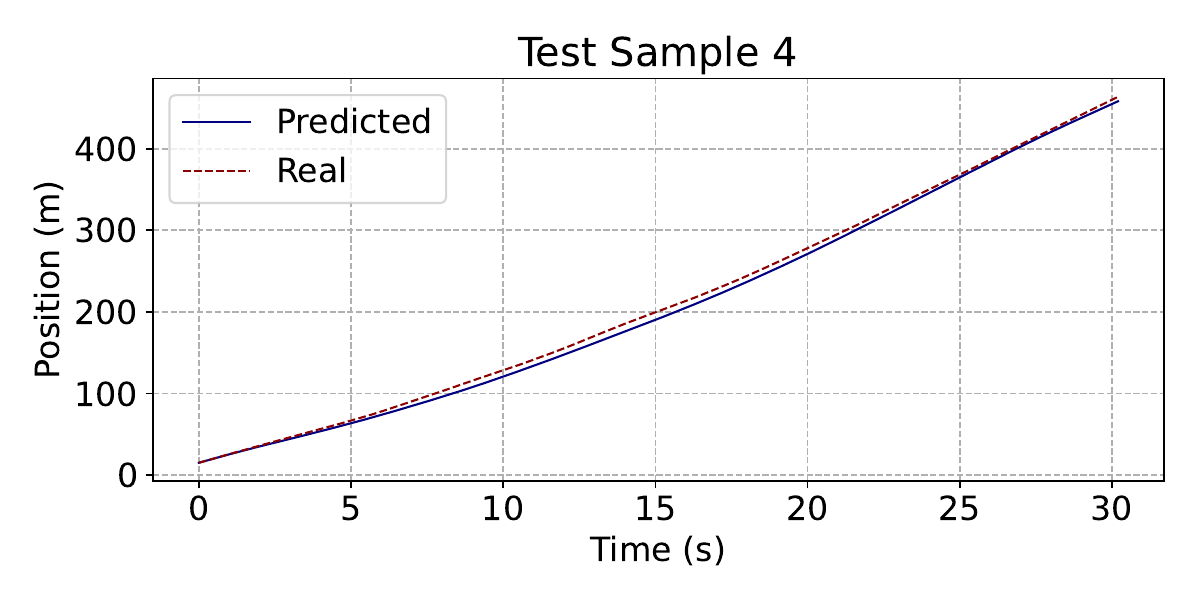}
  \end{subfigure}
  \begin{subfigure}[b]{0.48\textwidth}
    \centering
    \includegraphics[width=\textwidth]{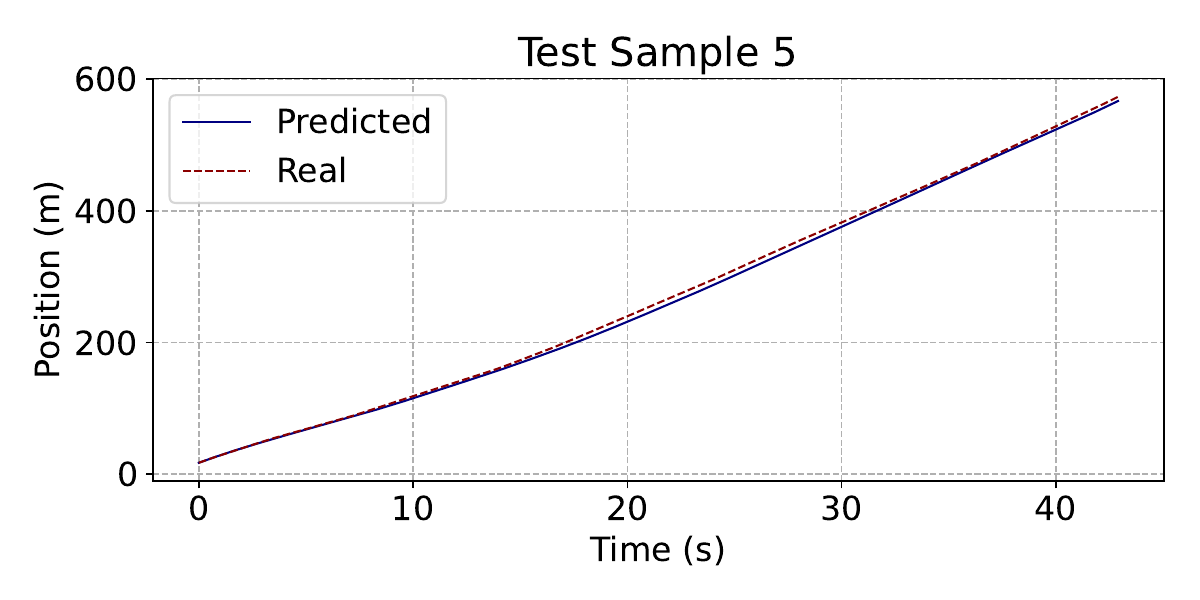}
  \end{subfigure}

  \begin{subfigure}[b]{0.48\textwidth}
    \centering
    \includegraphics[width=\textwidth]{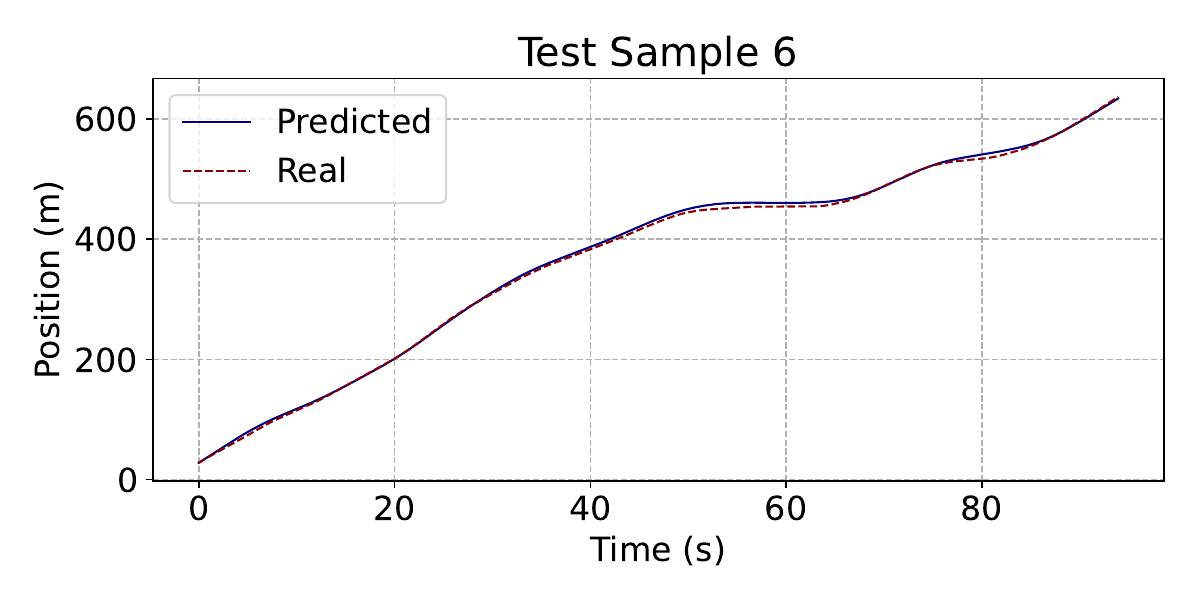}
  \end{subfigure}
  \begin{subfigure}[b]{0.48\textwidth}
    \centering
    \includegraphics[width=\textwidth]{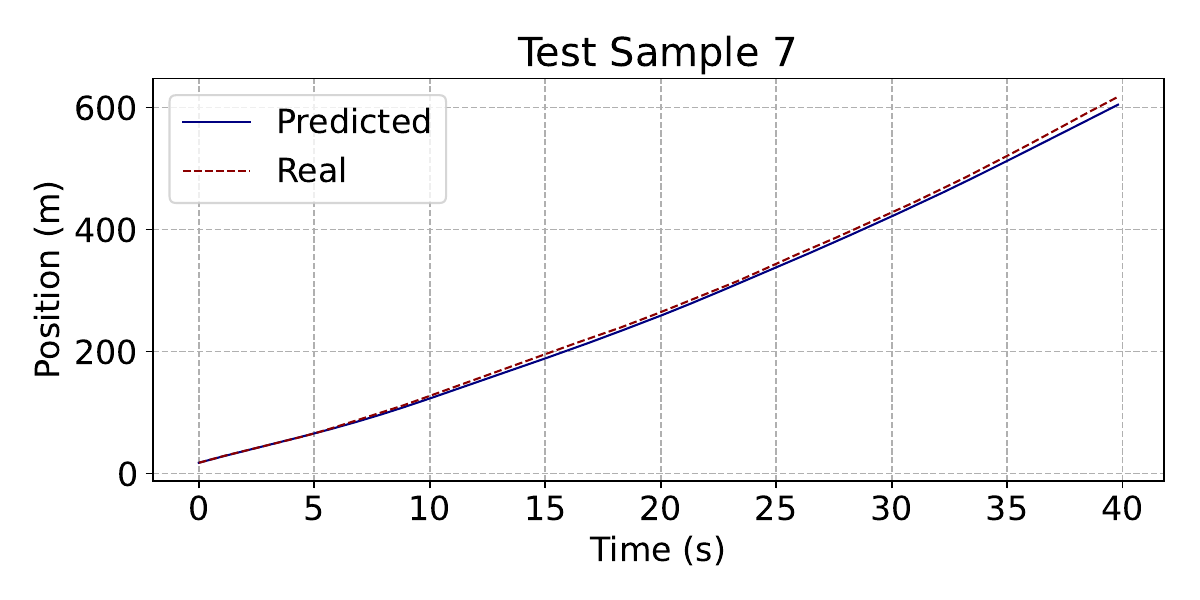}
  \end{subfigure}

  \begin{subfigure}[b]{0.48\textwidth}
    \centering
    \includegraphics[width=\textwidth]{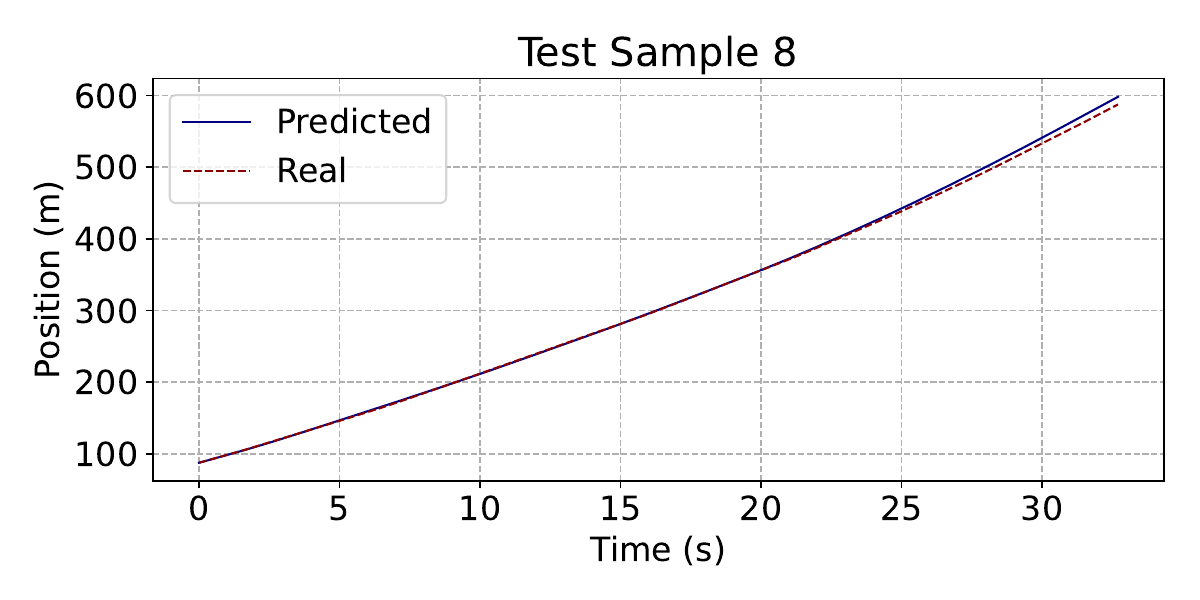}
  \end{subfigure}
  \begin{subfigure}[b]{0.48\textwidth}
    \centering
    \includegraphics[width=\textwidth]{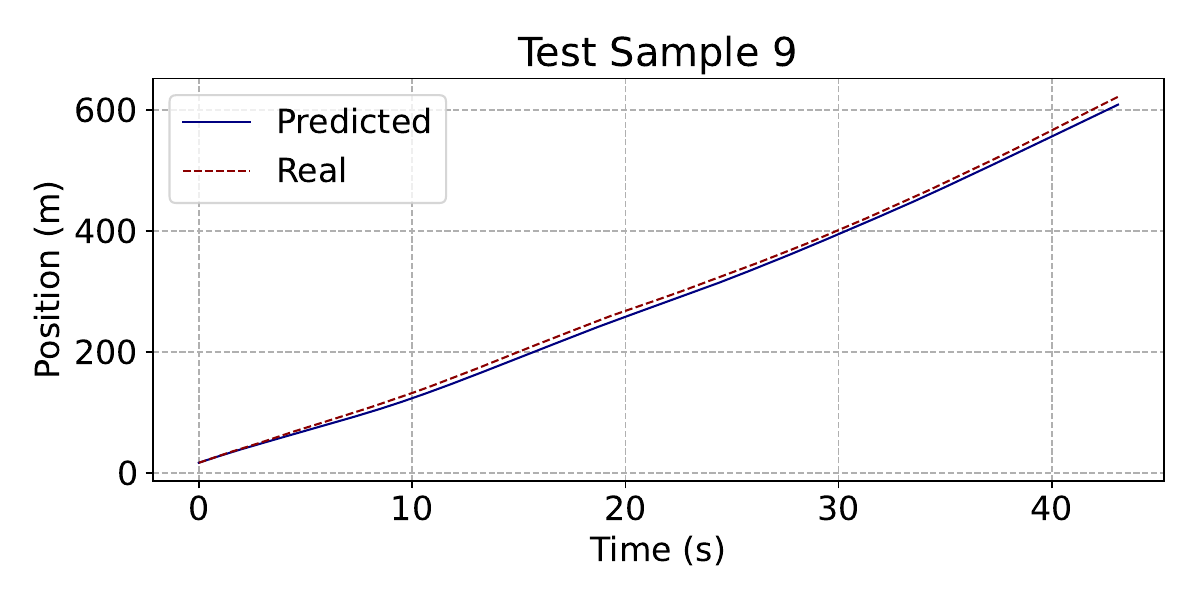}
  \end{subfigure}
  \caption{Comparison of predicted and real positions for ten Sample trajectories based on TMGD-PICF}
  \label{fig:10}
\end{figure}
\begin{figure}[htbp]
  \centering

  \begin{subfigure}[b]{0.48\textwidth}
    \centering
    \includegraphics[width=\textwidth]{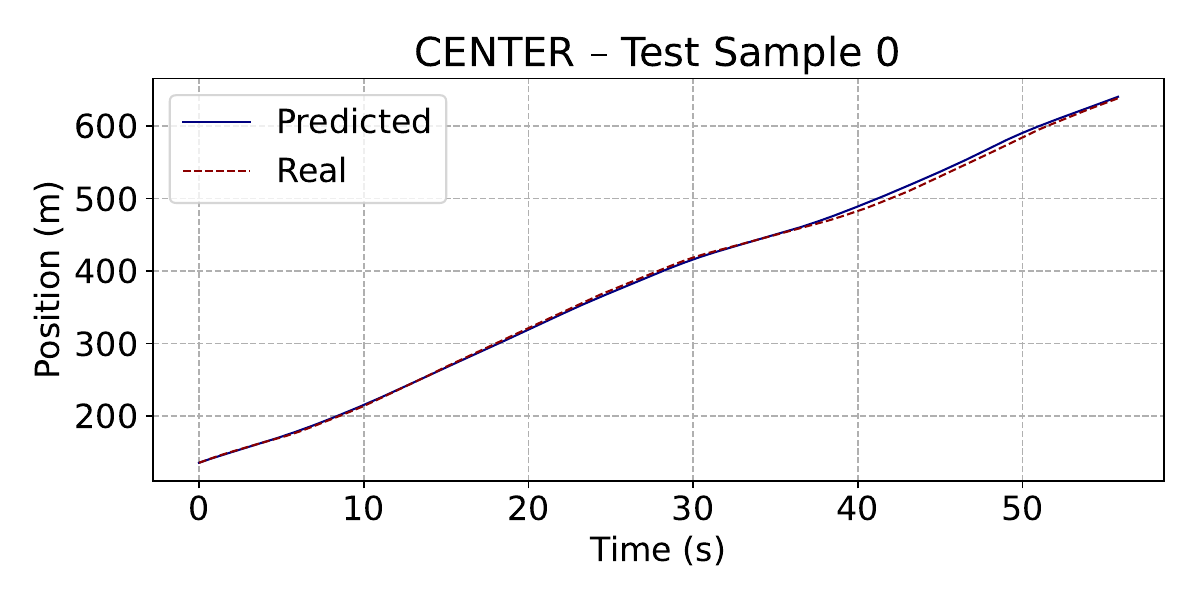}
  \end{subfigure}
  \begin{subfigure}[b]{0.48\textwidth}
    \centering
    \includegraphics[width=\textwidth]{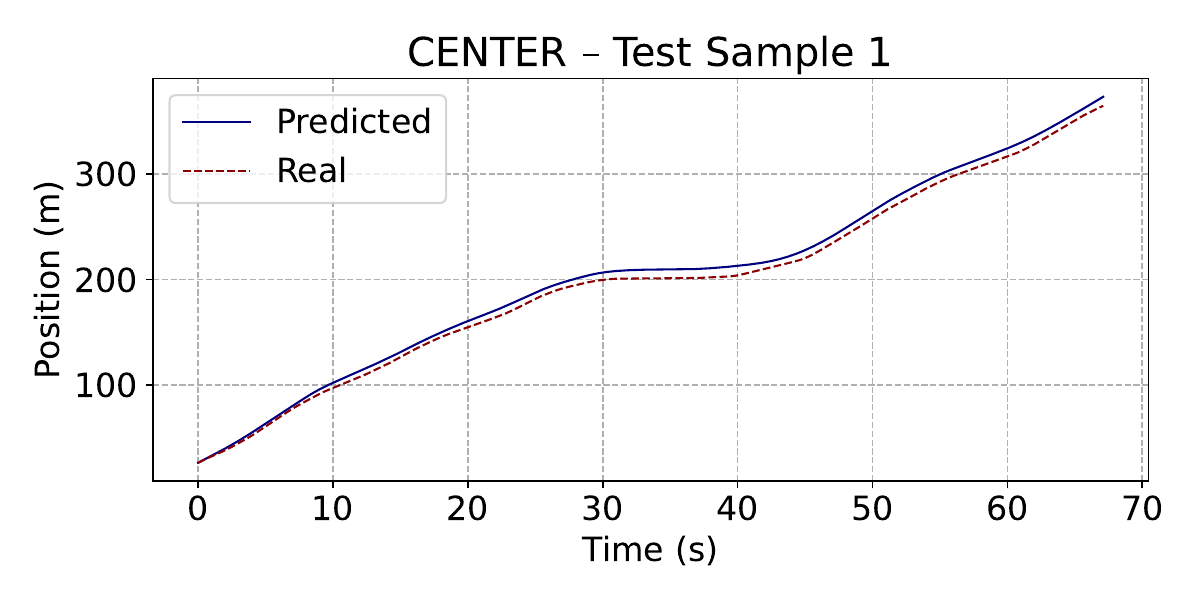}
  \end{subfigure}
  
  \begin{subfigure}[b]{0.48\textwidth}
    \centering
    \includegraphics[width=\textwidth]{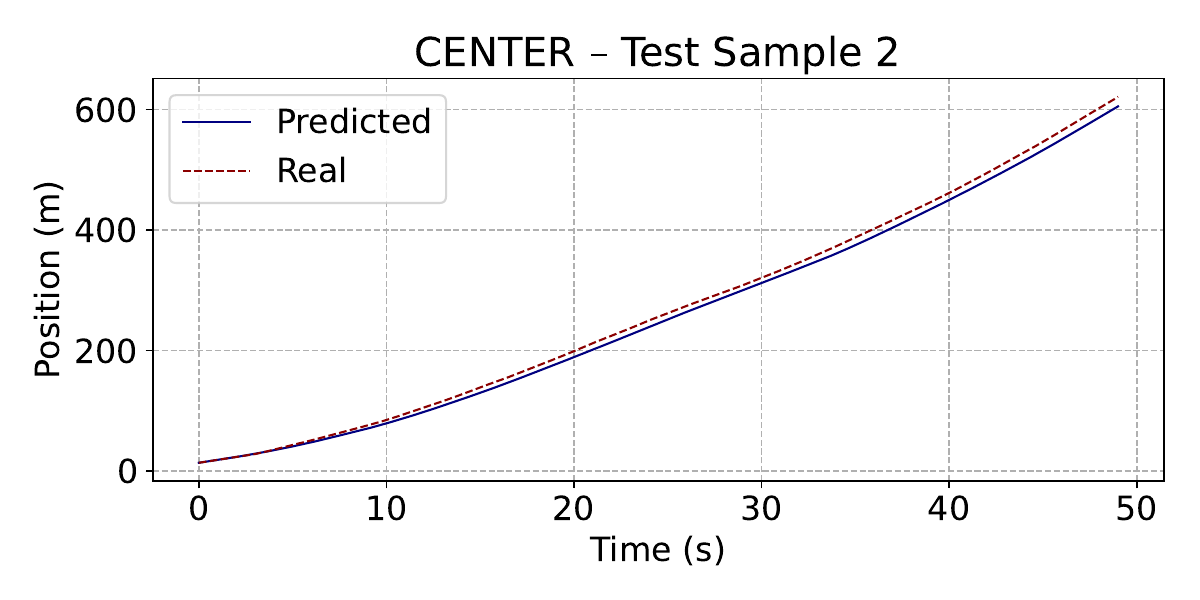}
  \end{subfigure}
  \begin{subfigure}[b]{0.48\textwidth}
    \centering
    \includegraphics[width=\textwidth]{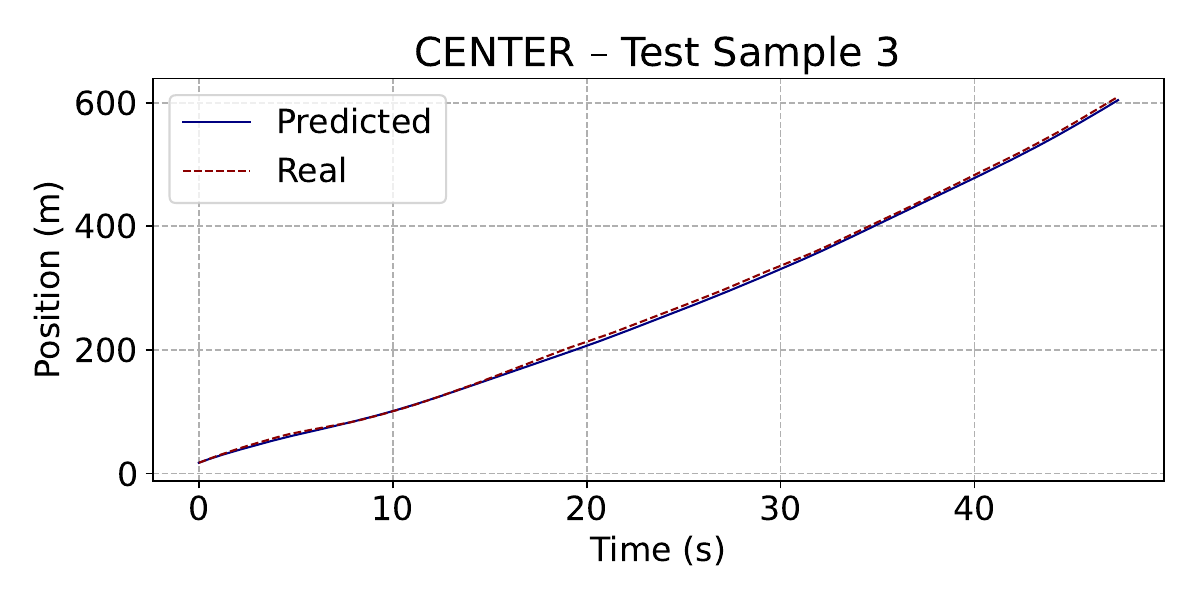}
  \end{subfigure}
  
  \begin{subfigure}[b]{0.48\textwidth}
    \centering
    \includegraphics[width=\textwidth]{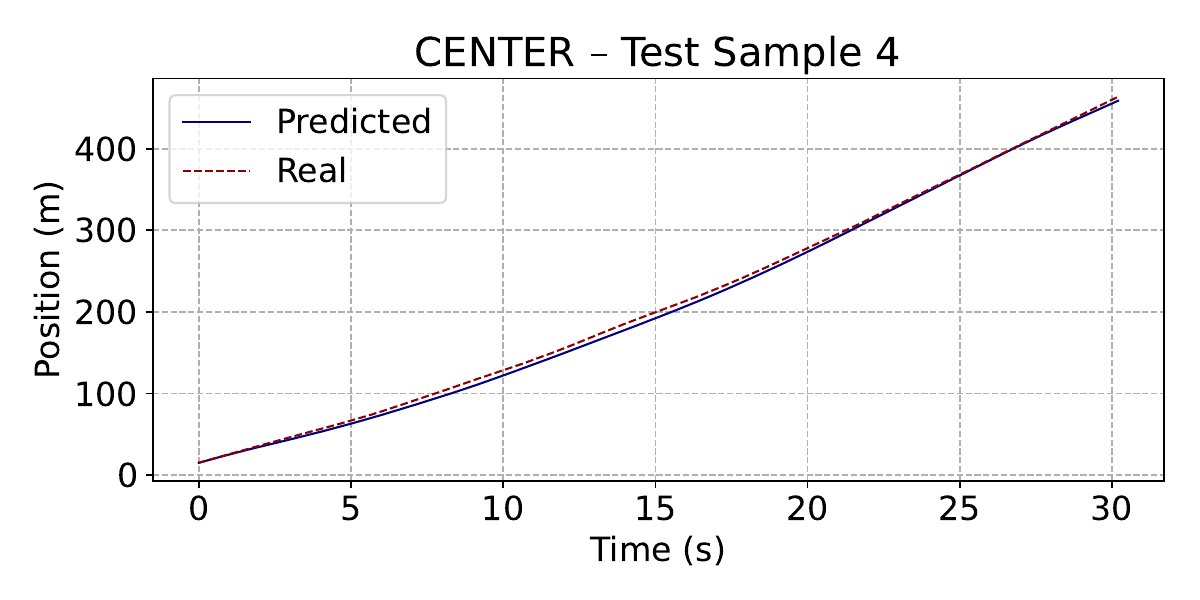}
  \end{subfigure}
  \begin{subfigure}[b]{0.48\textwidth}
    \centering
    \includegraphics[width=\textwidth]{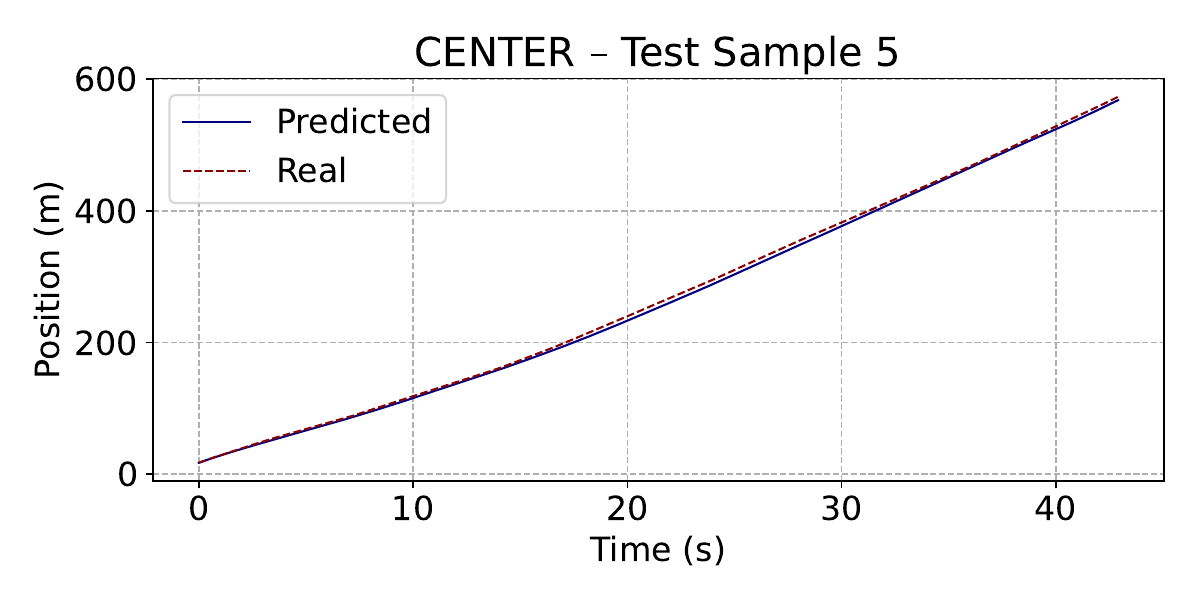}
  \end{subfigure}

  \begin{subfigure}[b]{0.48\textwidth}
    \centering
    \includegraphics[width=\textwidth]{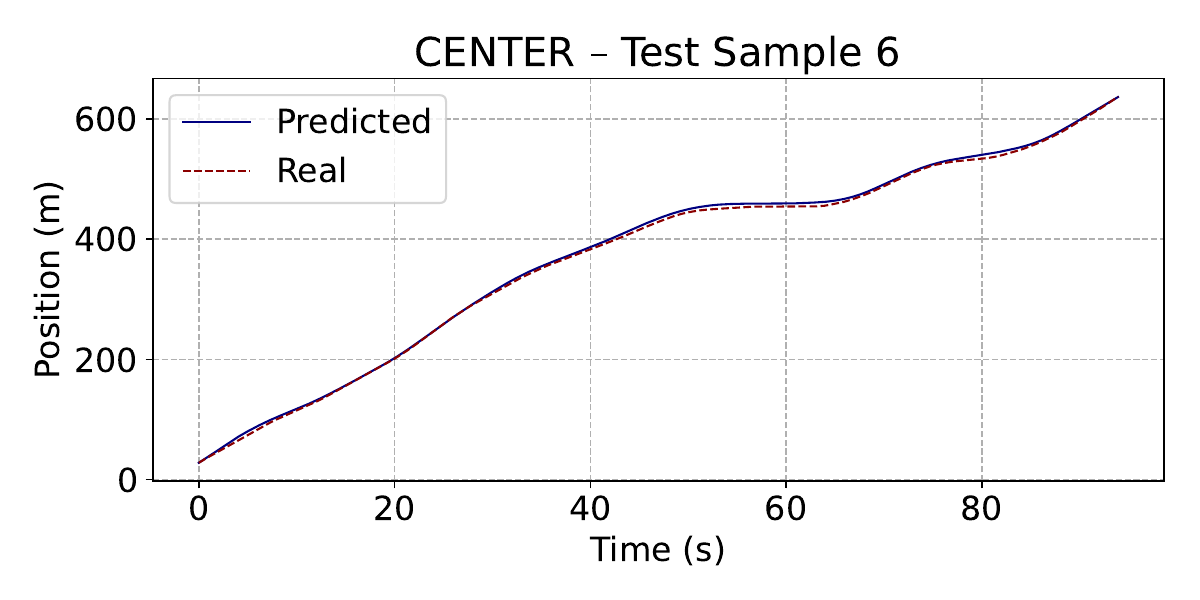}
  \end{subfigure}
  \begin{subfigure}[b]{0.48\textwidth}
    \centering
    \includegraphics[width=\textwidth]{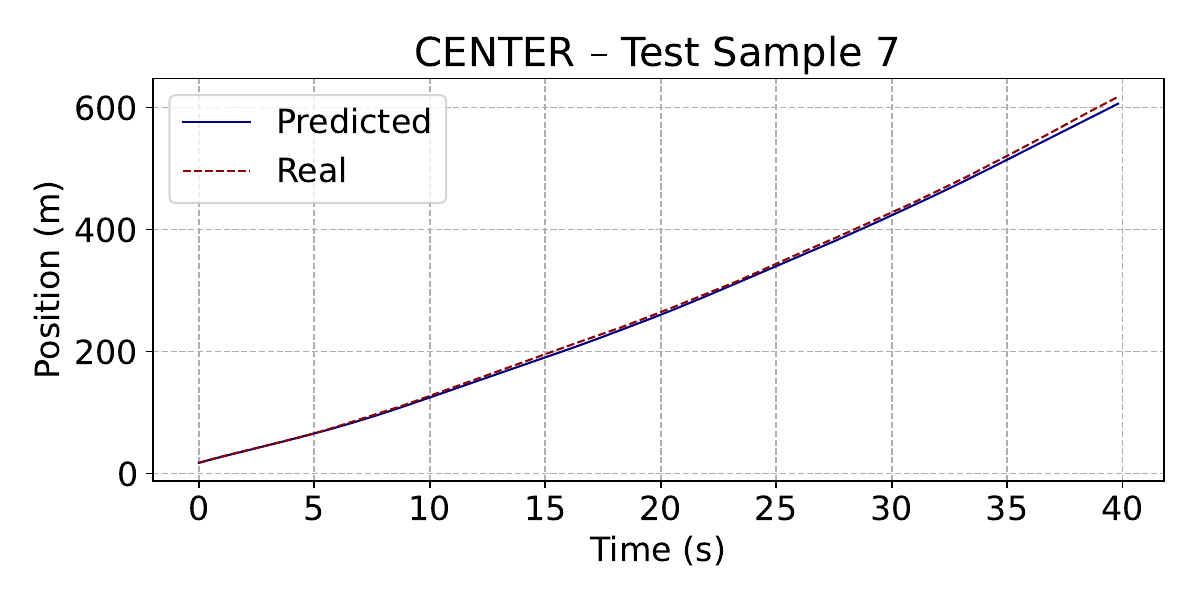}
  \end{subfigure}

  \begin{subfigure}[b]{0.48\textwidth}
    \centering
    \includegraphics[width=\textwidth]{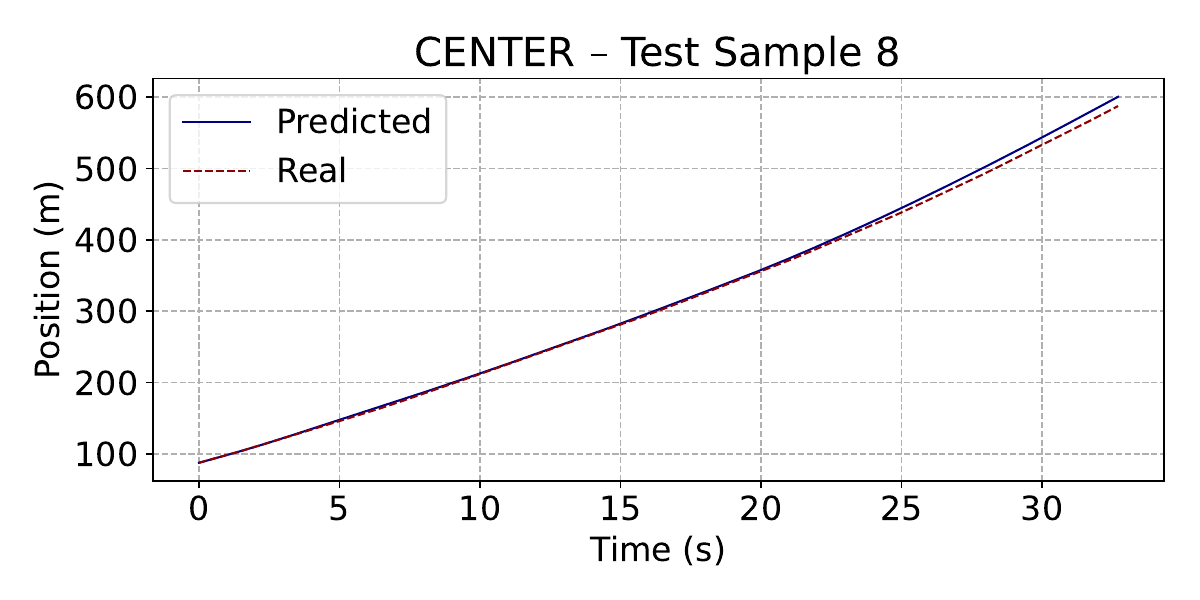}
  \end{subfigure}
  \begin{subfigure}[b]{0.48\textwidth}
    \centering
    \includegraphics[width=\textwidth]{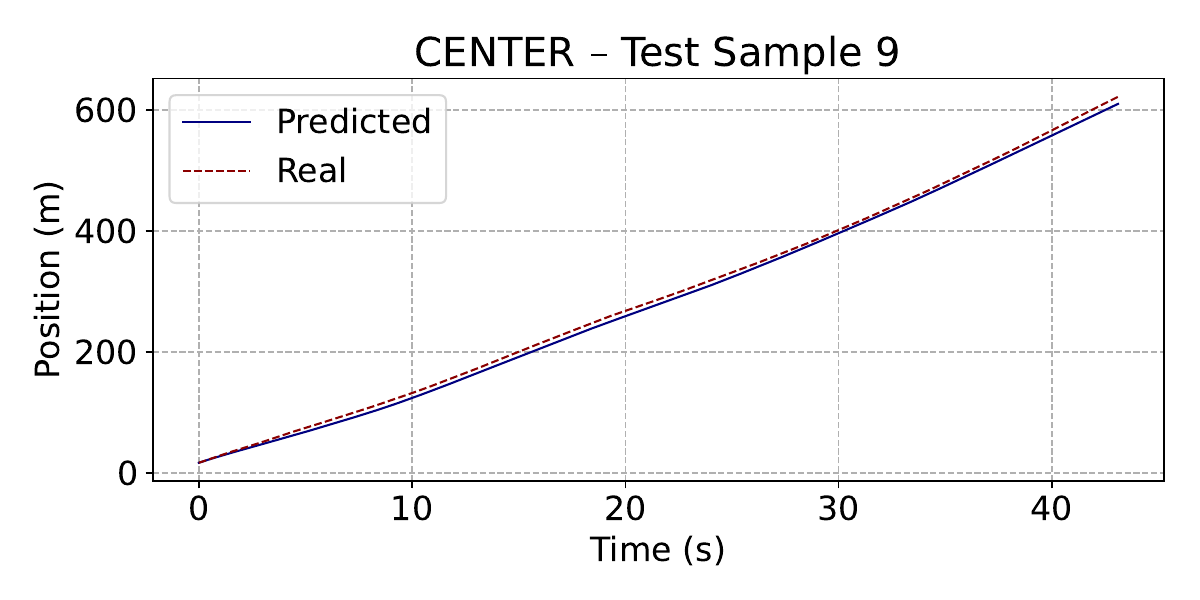}
  \end{subfigure}
  \caption{Comparison of predicted and real positions for ten Sample trajectories based on DCGD-CENTER-PICF}
  \label{fig:11}
\end{figure}
\begin{figure}[htbp]
  \centering

  \begin{subfigure}[b]{0.48\textwidth}
    \centering
    \includegraphics[width=\textwidth]{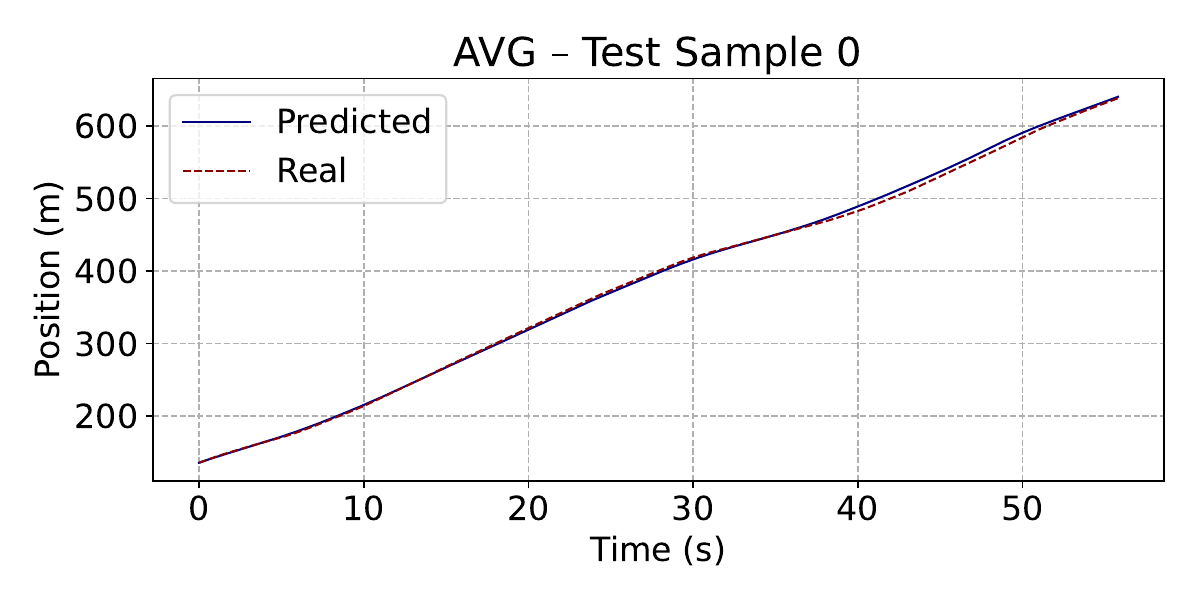}
  \end{subfigure}
  \begin{subfigure}[b]{0.48\textwidth}
    \centering
    \includegraphics[width=\textwidth]{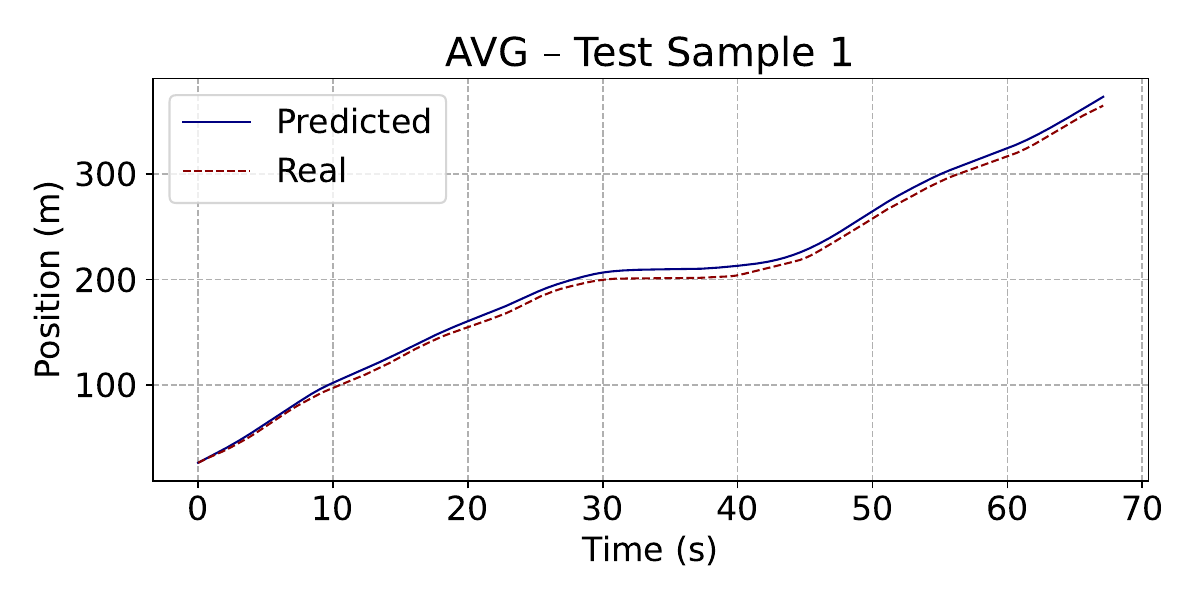}
  \end{subfigure}
  
  \begin{subfigure}[b]{0.48\textwidth}
    \centering
    \includegraphics[width=\textwidth]{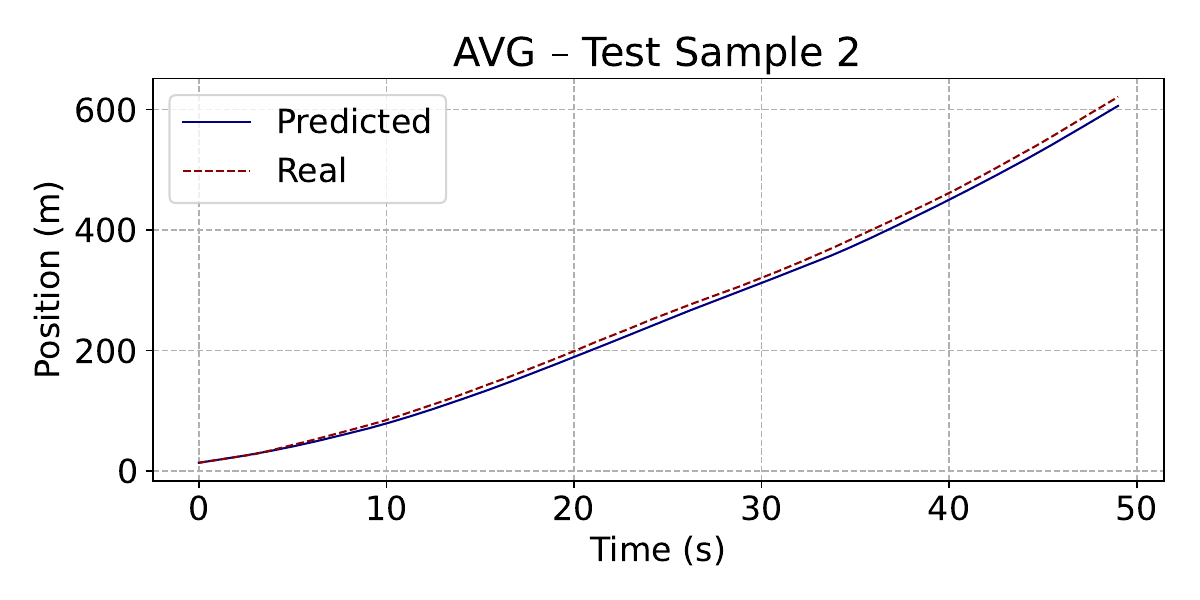}
  \end{subfigure}
  \begin{subfigure}[b]{0.48\textwidth}
    \centering
    \includegraphics[width=\textwidth]{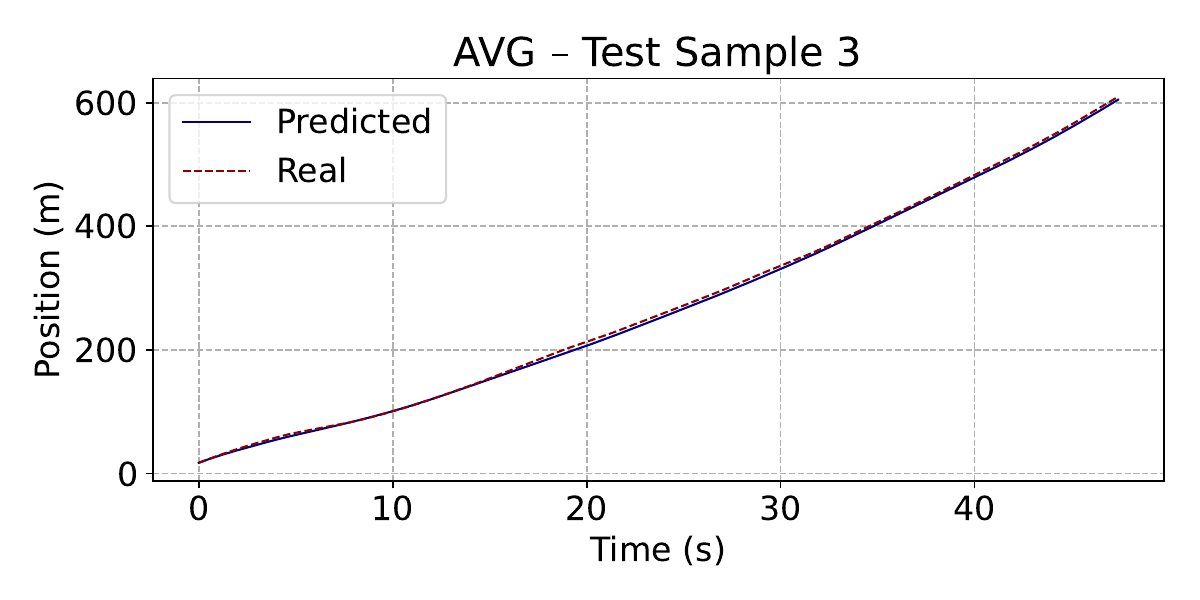}
  \end{subfigure}
  
  \begin{subfigure}[b]{0.48\textwidth}
    \centering
    \includegraphics[width=\textwidth]{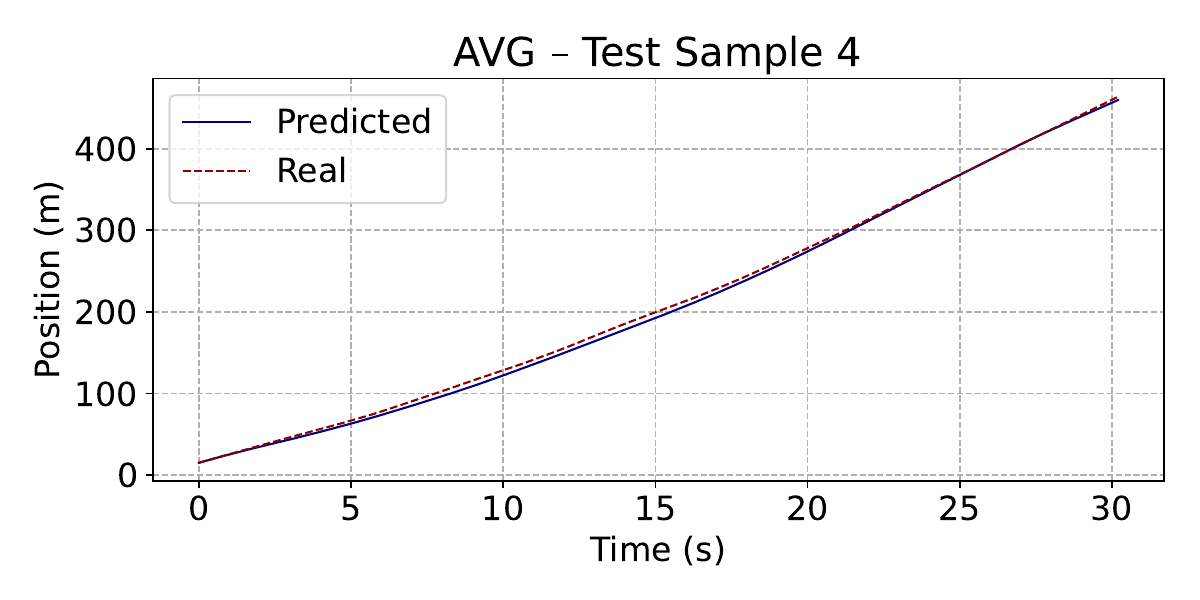}
  \end{subfigure}
  \begin{subfigure}[b]{0.48\textwidth}
    \centering
    \includegraphics[width=\textwidth]{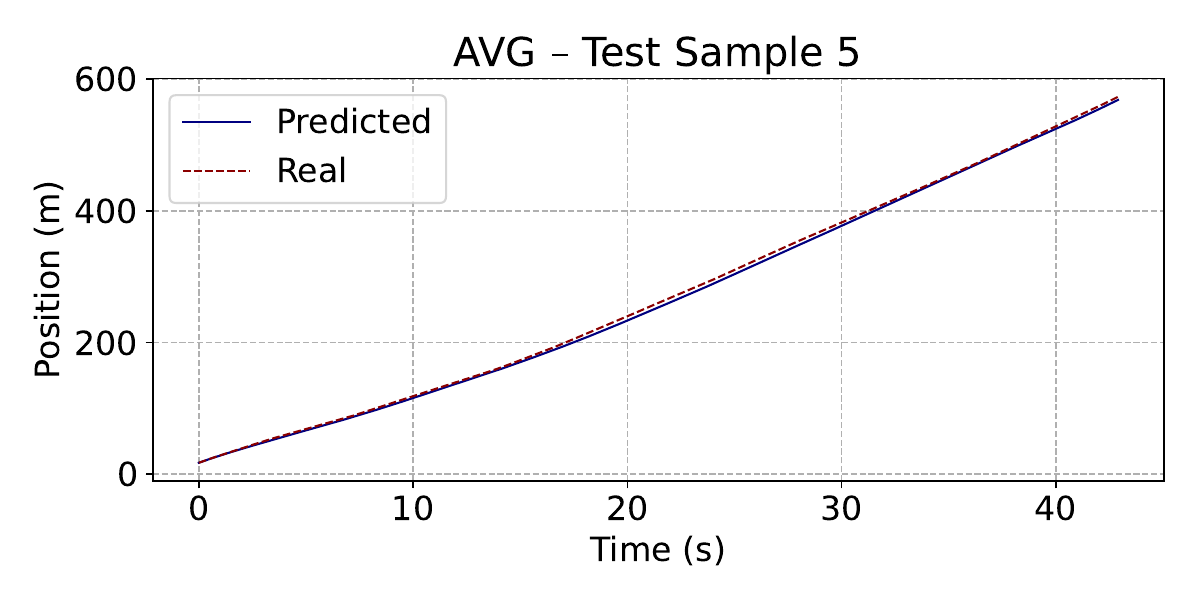}
  \end{subfigure}

  \begin{subfigure}[b]{0.48\textwidth}
    \centering
    \includegraphics[width=\textwidth]{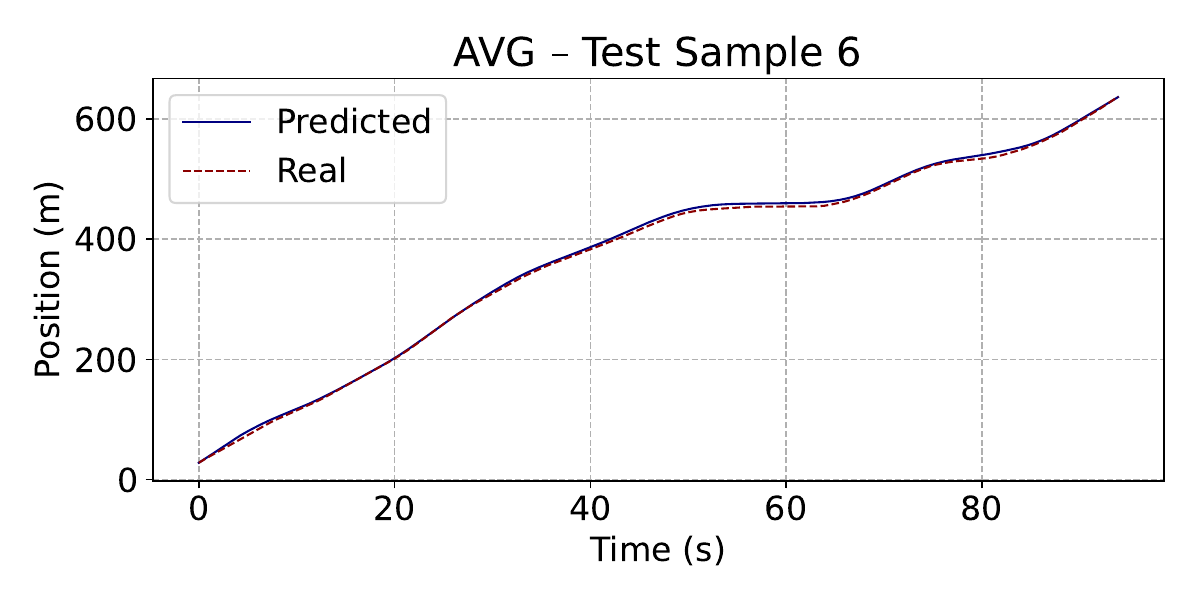}
  \end{subfigure}
  \begin{subfigure}[b]{0.48\textwidth}
    \centering
    \includegraphics[width=\textwidth]{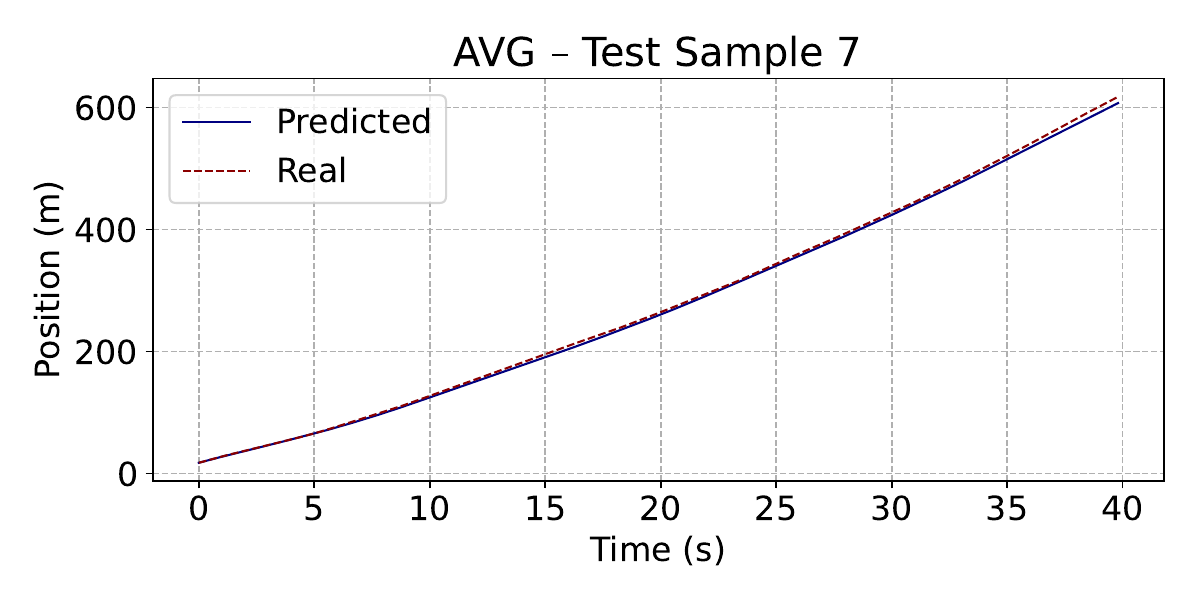}
  \end{subfigure}

  \begin{subfigure}[b]{0.48\textwidth}
    \centering
    \includegraphics[width=\textwidth]{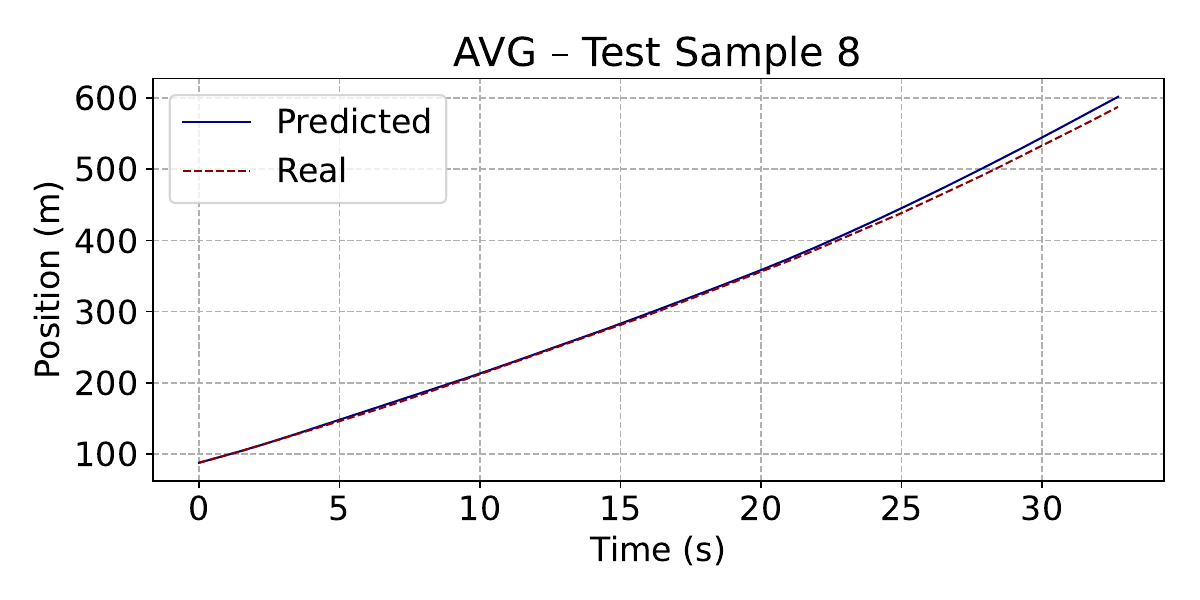}
  \end{subfigure}
  \begin{subfigure}[b]{0.48\textwidth}
    \centering
    \includegraphics[width=\textwidth]{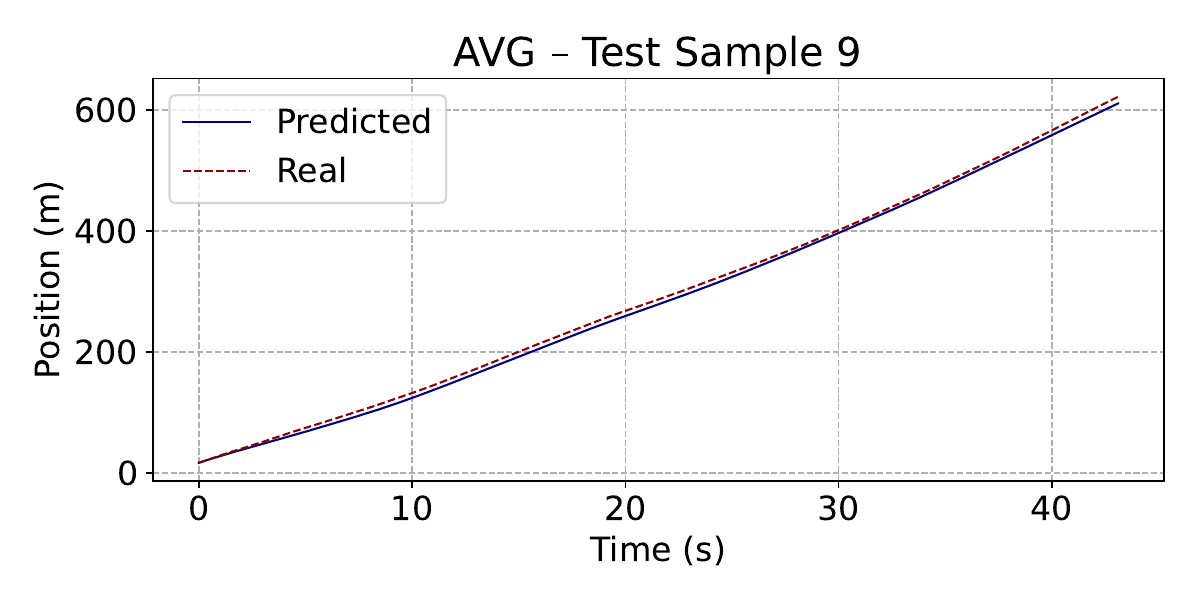}
  \end{subfigure}
  \caption{Comparison of predicted and real positions for ten Sample trajectories based on DCGD-AVG-PICF}
  \label{fig:12}
\end{figure}
\begin{figure}[htbp]
  \centering

  \begin{subfigure}[b]{0.48\textwidth}
    \centering
    \includegraphics[width=\textwidth]{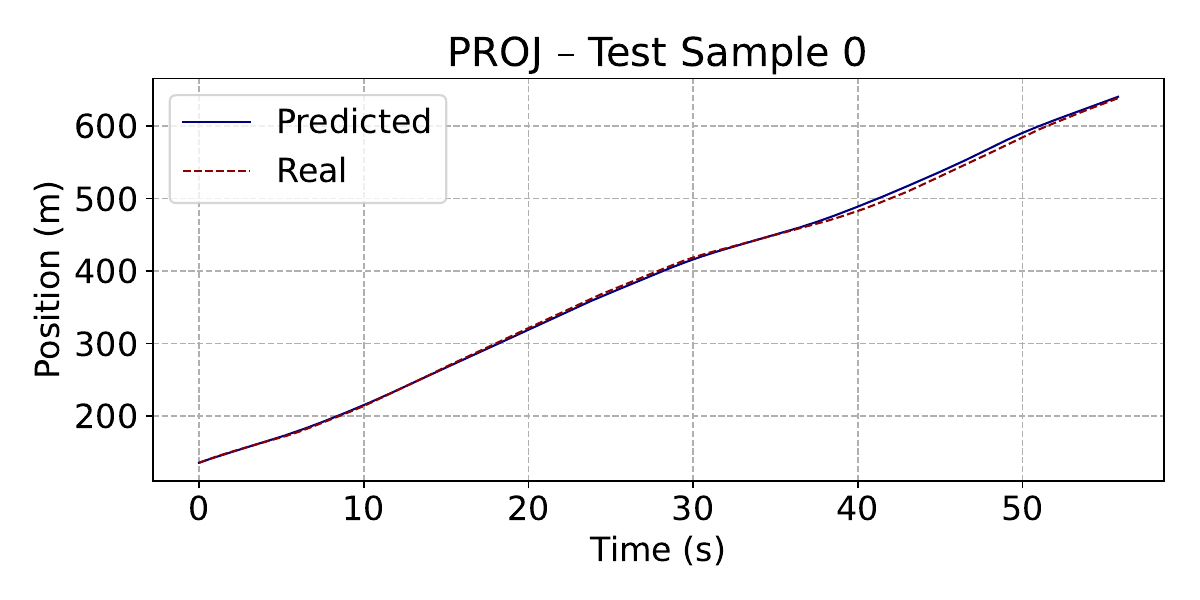}
  \end{subfigure}
  \begin{subfigure}[b]{0.48\textwidth}
    \centering
    \includegraphics[width=\textwidth]{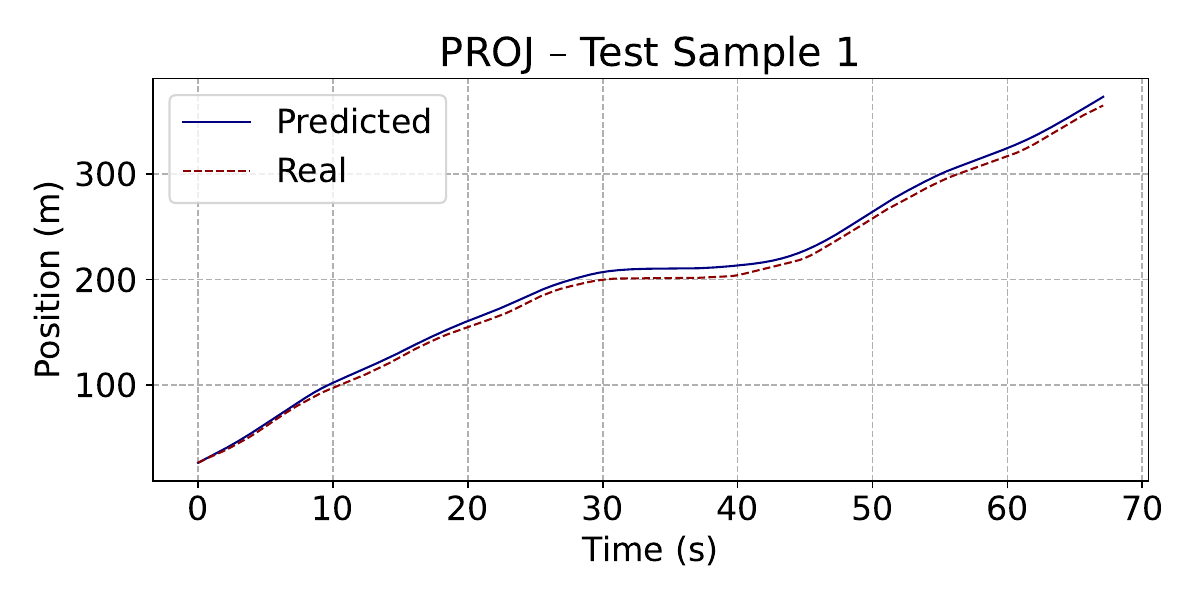}
  \end{subfigure}
  
  \begin{subfigure}[b]{0.48\textwidth}
    \centering
    \includegraphics[width=\textwidth]{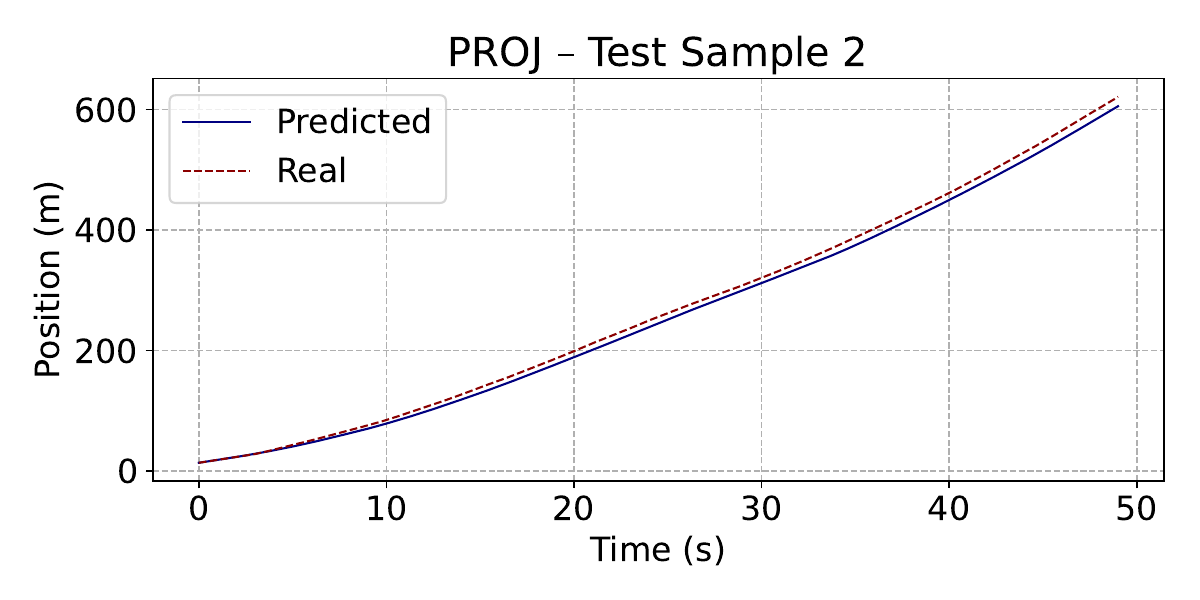}
  \end{subfigure}
  \begin{subfigure}[b]{0.48\textwidth}
    \centering
    \includegraphics[width=\textwidth]{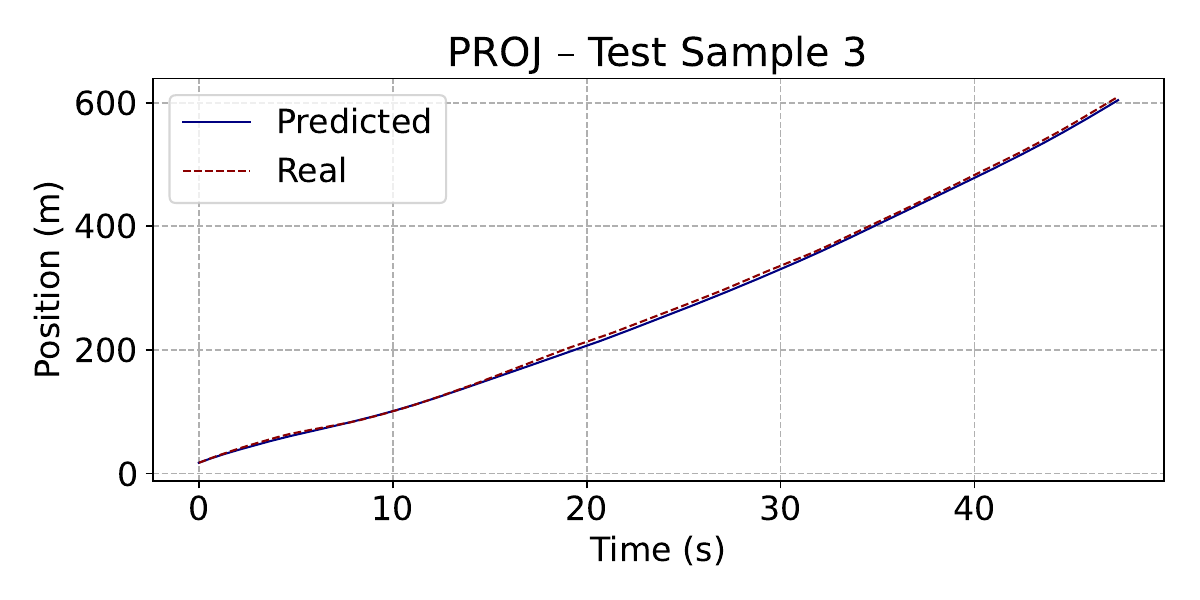}
  \end{subfigure}
  
  \begin{subfigure}[b]{0.48\textwidth}
    \centering
    \includegraphics[width=\textwidth]{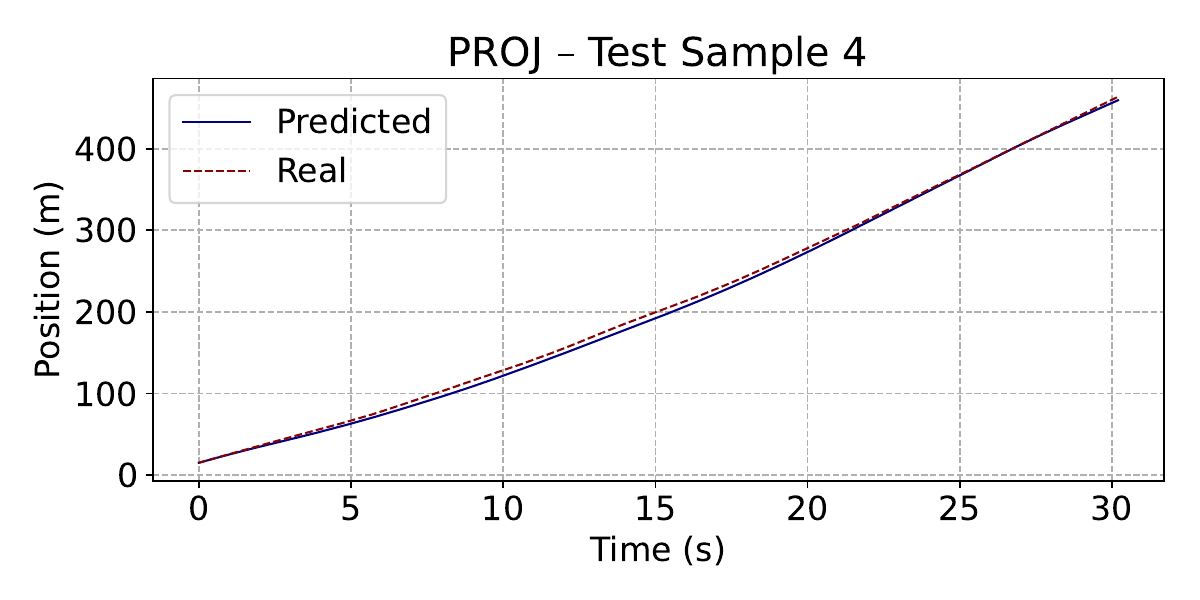}
  \end{subfigure}
  \begin{subfigure}[b]{0.48\textwidth}
    \centering
    \includegraphics[width=\textwidth]{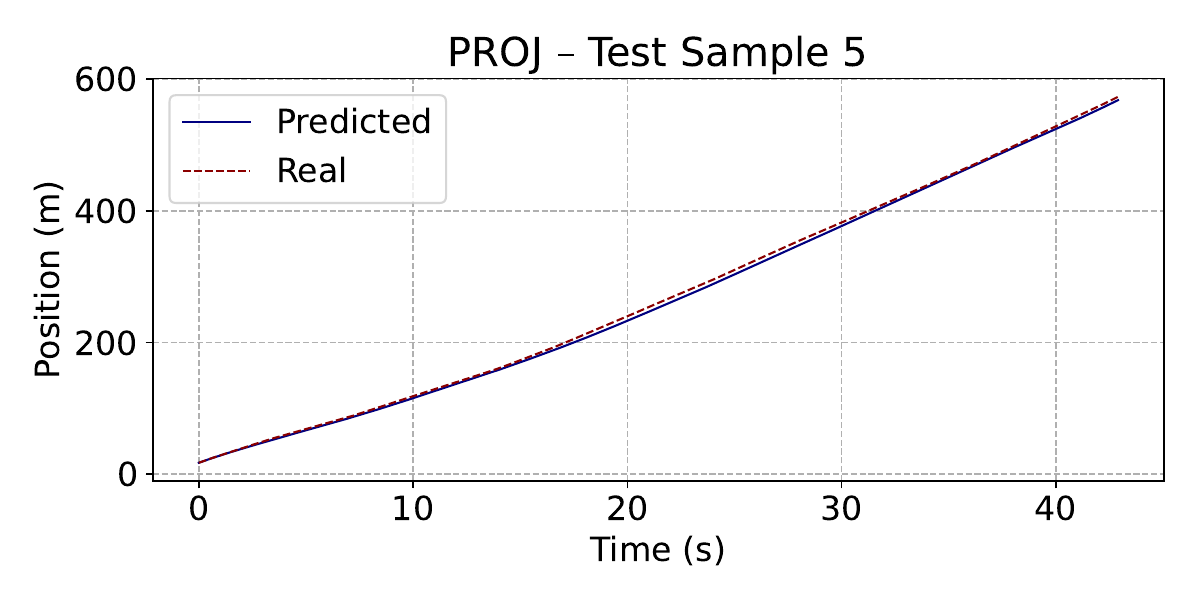}
  \end{subfigure}

  \begin{subfigure}[b]{0.48\textwidth}
    \centering
    \includegraphics[width=\textwidth]{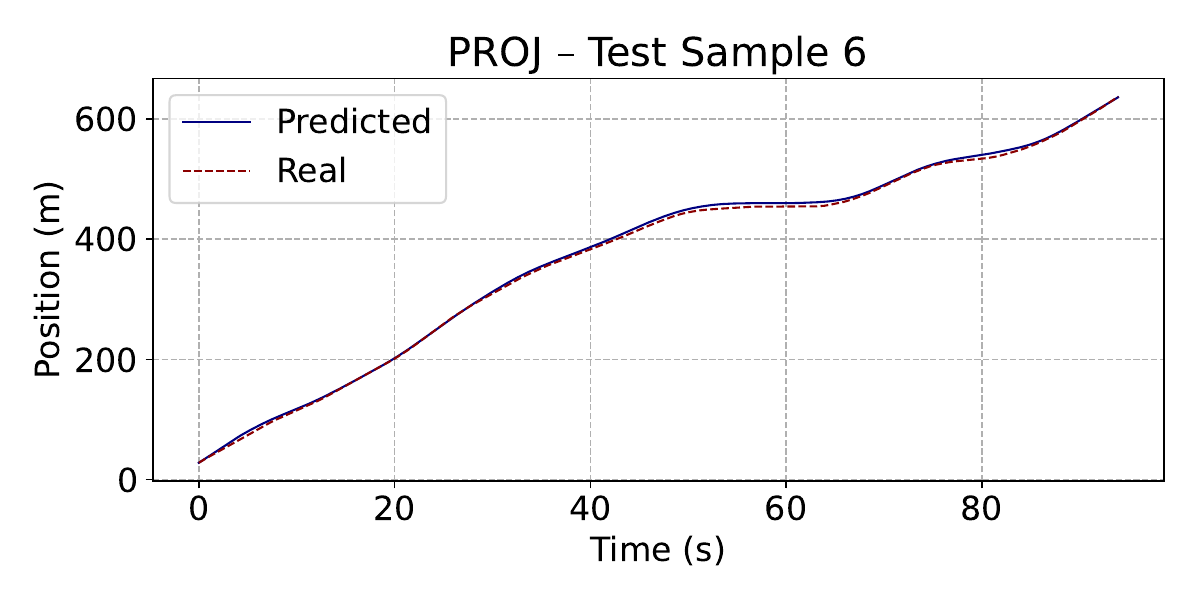}
  \end{subfigure}
  \begin{subfigure}[b]{0.48\textwidth}
    \centering
    \includegraphics[width=\textwidth]{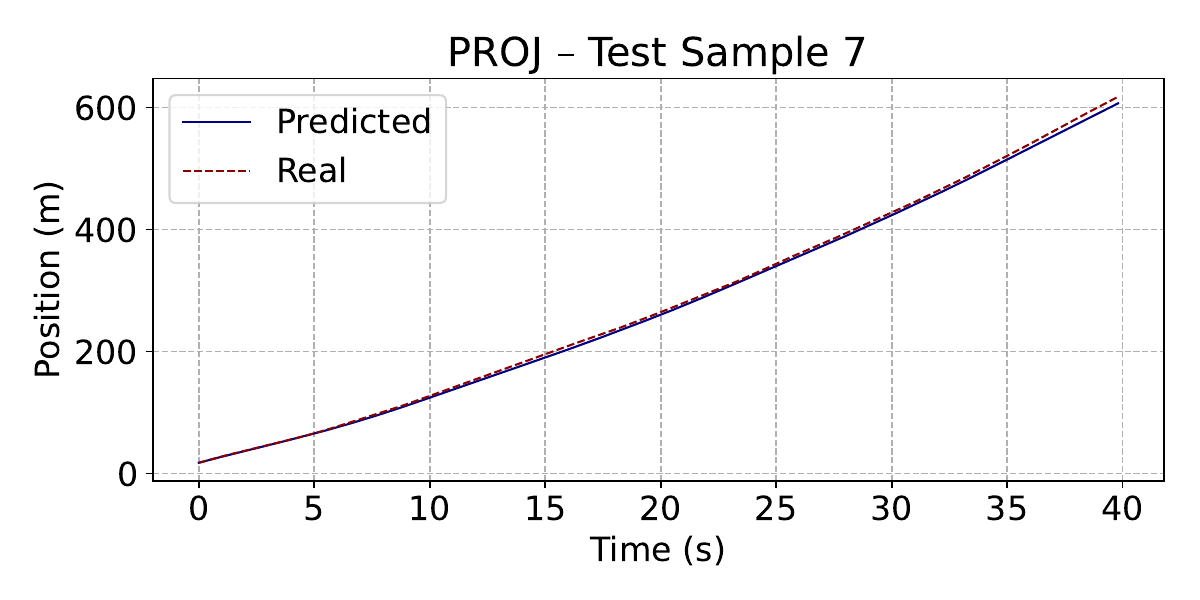}
  \end{subfigure}

  \begin{subfigure}[b]{0.48\textwidth}
    \centering
    \includegraphics[width=\textwidth]{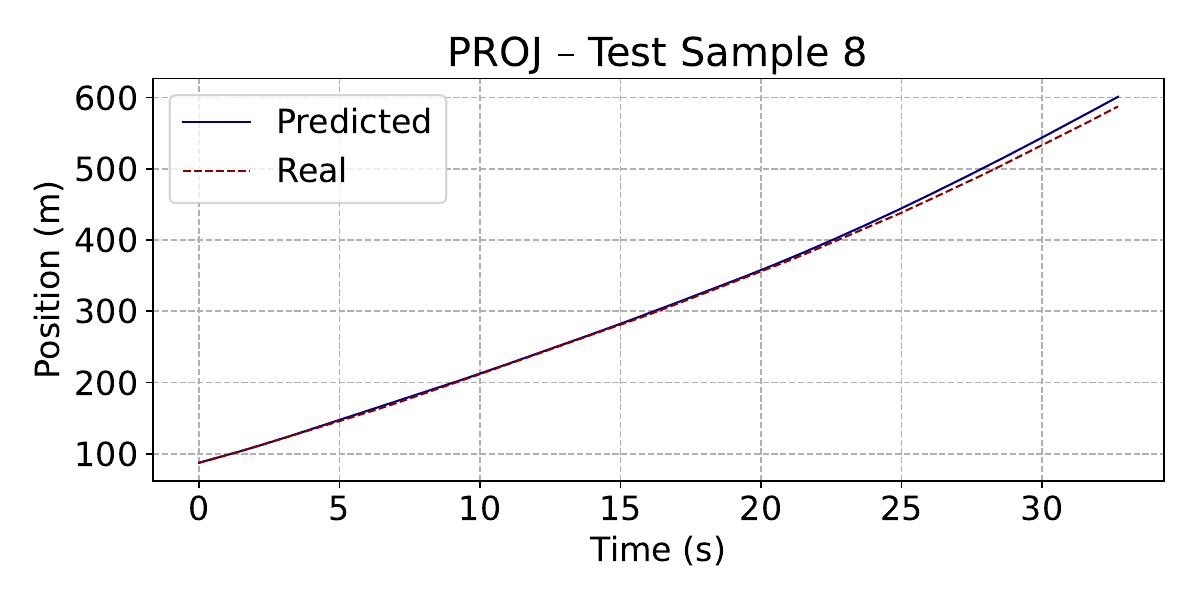}
  \end{subfigure}
  \begin{subfigure}[b]{0.48\textwidth}
    \centering
    \includegraphics[width=\textwidth]{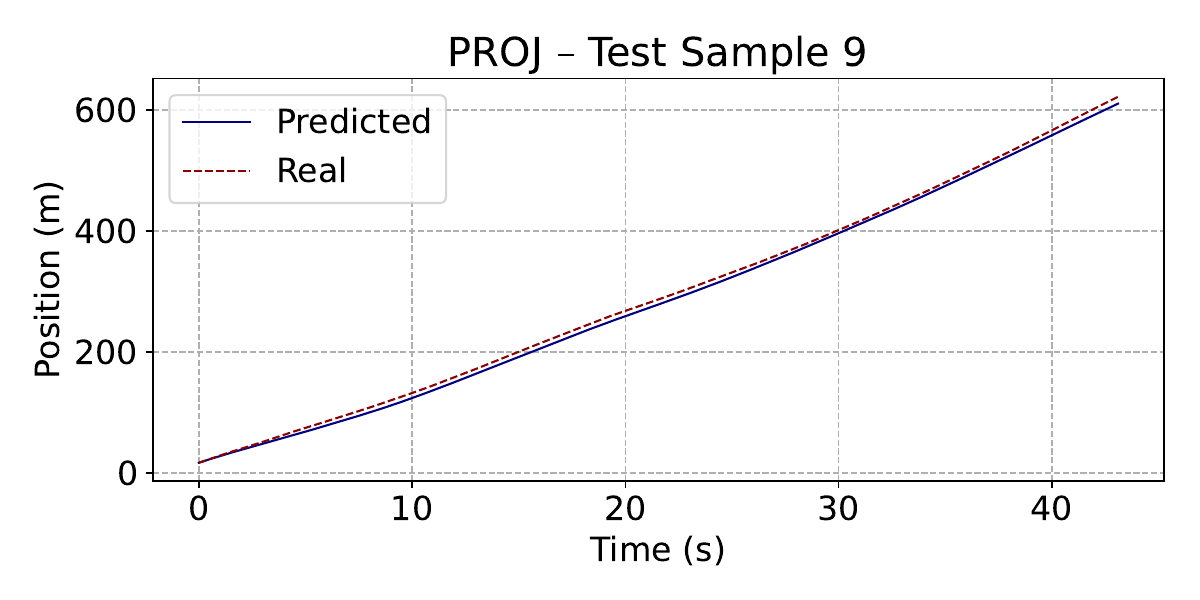}
  \end{subfigure}
  \caption{Comparison of predicted and real positions for ten Sample trajectories based on DCGD-PROJ-PICF}
  \label{fig:13}
\end{figure}
\section{Conclusion and discussion} \label{section:5}
\par In this paper, we analyze the potential limitations of the current applications of Predictive Interactive Machine Learning (PIML) in transportation, specifically regarding the total loss function. The reliance on linear scalarization introduces unrealistic assumptions about the trade-offs between different objectives, and the hyperparameter optimization process can be time-consuming. To address these issues, we reformulate the training process of PIML as a multi-objective optimization problem. We introduce several multi-gradient descent algorithms (MGDAs) to solve this problem effectively.
\par  After investigating several PIML models using real-world data, we discovered that almost all training methods based on Multi-Gradient Descent Algorithms (MGDAs) can achieve results comparable to or better than those obtained through training with a loss constructed using linear scalarization, and this is also theoretically guaranteed. 
\par In the first example, the MGDAs-based training approach can achieve a comparable performance compared to the baseline model in most cases. In the second example, all MGDAs-based training approaches bring a considerable performance improvement.  \label{sec:1} \label{line:1} From these case studies, we can conclude that using MGDAs to train the physics-informed machine learning model is preferable and beneficial over the traditional training method because it eliminates the need for hyperparameter searching that can be extremely tricky and time-consuming. Both traditional multi-gradient descent and dual cone gradient descent methods can theoretically ensure that the trained parameters reach a Pareto stationary point. This suggests that neither method has a theoretical advantage over the other, implying that each method may outperform the other in different situations. We suggest keeping both training methods based on multi-gradient descent to replace the traditional training logic.
\par Possible further directions could be to focus on three areas: (1) How to accelerate multi-gradient descent algorithms. Compared to directly using a stochastic gradient descent algorithm like Adam (\cite{kingma2014adam}) to minimize the linear scalarization loss. The training process based on the reformulated multi-objective optimization will lead to a longer training time per iteration. Especially for TMGD algorithms, in each iteration, we need to find the minimum norm point by solving a convex quadratic programming problem. (2) How to handle hard constraints in the training process. All MGDAs used in this paper are only compatible with unconstrained multi-objective optimization. What if we have to consider other constraints in PIML? Is there another way rather than modeling constraints in the loss function directly as a soft constraint term? A possible way is to extend the framework proposed in \cite{lei2025modified}, which presents a framework based the multi-gradient descent that can handle multi-objective optimization problems with convex constraints. (3) Research on multi-task learning suggests that TGMD (\cite{lin2019pareto}) may not consistently produce a well-balanced Pareto front. Instead, TGMD tends to generate a discrete and sparse representation of the Pareto front. Therefore, further studies could explore methods to create a more complete or continuous Pareto front by following the approaches outlined by \cite{lin2019pareto} and \cite{ma2020efficient}.
\section{Appendix} \label{section:6}
\subsection{Proof theorem \ref{theorem:5}} \label{ap:1}
\begin{proof}
\textbf{Proof of Statement 1:}
\par Suppose $\theta^*$ is a (local) minimizer of the scalarized total loss for some $\beta \in (0,1)$. Then
\begin{equation}\label{eq:34}
\nabla \mathcal{L}_\beta(\theta^*) \;=\; \beta \nabla \mathcal{L}_1(\theta^*) + (1 - \beta) \nabla \mathcal{L}_2(\theta^*) \;=\; 0
\end{equation}
Equivalently, $0\in\mathrm{co}\{\nabla \mathcal{L}_1(\theta^*),\nabla \mathcal{L}_2(\theta^*)\}$, hence $\theta^*$ is Pareto stationary. By \textcolor{blue}{\textbf{Theorem}}~\ref{theorem:3} and \ref{theorem:4}, MGDA identifies such points as fixed points without common strict descent. Therefore, every interior scalarized minimizer can be reached (identified) by MGDA.

\vspace{0.5em}
\textbf{Proof of Statement 2:}
\par Consider a Pareto stationary point $\hat{\theta}$ (for example, one found by MGDA) with
\begin{equation}
\nabla\mathcal{L}_1(\hat{\theta}) = 0 \qquad \nabla\mathcal{L}_2(\hat{\theta}) \neq 0
\end{equation}
For any $\beta \in (0,1)$ we have
\begin{equation}\label{eq:35}
\nabla \mathcal{L}_\beta(\hat{\theta}) = \beta \nabla \mathcal{L}_1(\hat{\theta}) + (1-\beta)\nabla \mathcal{L}_2(\hat{\theta})
= (1-\beta)\nabla \mathcal{L}_2(\hat{\theta}) \neq 0
\end{equation}
Thus $\hat{\theta}$ is not a stationary point, and therefore it is not a minimizer of any interior linear scalarization. Consequently, certain Pareto stationary points found by MGDA cannot be obtained by optimizing any interior linear scalarization.
\end{proof}

\section{Acknowledgement} 
This research is supported by the award "CAREER: Physics Regularized Machine Learning Theory: Modeling Stochastic Traffic Flow Patterns for Smart Mobility Systems (\# 2234289)" which is funded by the National Science Foundation. We would like to thank Professor Dong-Young Lim from Ulsan National Institute of Science and Technology for generously sharing the original vector versions of the figures for reproduction.
\section{CRediT}
\textbf{Yuan-Zheng Lei}: Conceptualization, Methodology, Writing - original draft. \textbf{Yaobang Gong}: Conceptualization, Writing - review \& editing. \textbf{Dianwei Chen}: Experiment design. \textbf{Yao Cheng}: Writing - review \& editing. \textbf{Xianfeng Terry Yang}: Conceptualization, Methodology and Supervision, Writing - review \& editing.
\section{Data and Code Availability} 
The source code and data are available for readers to download via the following link: \url{https://github.com/EdisonYLei/Reconstructing-PIML-through-MGDA}

\bibliography{ref}

\end{document}